\documentclass[twoside,11pt]{article}

\usepackage{amsmath}
\usepackage{mathrsfs}
\usepackage{mathtools}
\usepackage{algorithm}
\usepackage{algorithmic}
\usepackage{enumitem}
\usepackage{bbm}
\usepackage{changepage}
\usepackage{booktabs, verbatim} 
\usepackage{tikz}
\usepackage{subcaption}
\usetikzlibrary{shapes,arrows,positioning}
\mathtoolsset{showonlyrefs}
\usepackage{kbordermatrix}

\usepackage[preprint]{jmlr2e}

\allowdisplaybreaks

\usepackage[open,openlevel=1]{bookmark}
\usepackage[titletoc]{appendix}

\usepackage{xcolor}

\newcommand{\Cov}{\mathrm{Cov}}
\newcommand{\Var}{\mathrm{Var}}
\newcommand{\tr}{\mathrm{tr}}
\newcommand{\frobnorm}[1]{\left\lVert#1\right\rVert_{\text{F}}}
\DeclareMathOperator*{\argmax}{arg\,max}
\DeclareMathOperator*{\argmin}{arg\,min}
\DeclareMathOperator{\vect}{vec}

\newcommand{\independent}{\perp\!\!\!\perp}

\usepackage{centernot}
\usepackage{bm}

\DeclareFontFamily{U}{mathx}{\hyphenchar\font45}
\DeclareFontShape{U}{mathx}{m}{n}{
      <5> <6> <7> <8> <9> <10>
      <10.95> <12> <14.4> <17.28> <20.74> <24.88>
      mathx10
      }{}
\DeclareSymbolFont{mathx}{U}{mathx}{m}{n}
\DeclareFontSubstitution{U}{mathx}{m}{n}
\DeclareMathAccent{\widecheck}{0}{mathx}{"71}
\DeclareMathAccent{\wideparen}{0}{mathx}{"75}

\newtheorem{assump}{Assumption}

\usepackage{lastpage}

\jmlrheading{23}{2022}{1-\pageref{LastPage}}{3/20}{3/22}{20-231}{Boxin Zhao, Y. Samuel Wang, and Mladen Kolar}

\ShortHeadings{Functional Differential Graph Estimation}{Zhao, Wang, and Kolar}

\firstpageno{1}  

\begin{document}
	
\title{FuDGE: A Method to Estimate a Functional Differential Graph in a High-Dimensional Setting}
	
\author{\name Boxin Zhao \email boxinz@uchicago.edu \\
		\addr Booth School of Business\\
		The University of Chicago\\
		Chicago, IL 60637, USA\\
		\name Y. Samuel Wang \email ysw7@cornell.edu \\
      \addr Department of Statistics and Data Science\\
Cornell University\\
       Ithaca, NY 14853, USA\\
		\name Mladen Kolar \email mkolar@chicagobooth.edu \\
		\addr Booth School of Business\\
		The University of Chicago\\
		Chicago, IL 60637, USA}
	
\editor{Daniela Witten}

\maketitle

\begin{abstract}
We consider the problem of estimating the difference between two undirected functional graphical models with shared structures. In many applications, data are naturally regarded as a vector of random functions rather than as a vector of scalars. For example, electroencephalography (EEG) data are treated more appropriately as functions of time. In such a problem, not only can the number of functions measured per sample be large, but each function is itself an infinite-dimensional object, making estimation of model parameters challenging. This is further complicated by the fact that curves are usually observed only at discrete time points. We first define a functional differential graph that captures the differences between two functional graphical models and formally characterize when the functional differential graph is well defined. We then propose a method, FuDGE, that directly estimates the functional differential graph without first estimating each individual graph. This is particularly beneficial in settings where the individual graphs are dense but the differential graph is sparse. We show that FuDGE consistently estimates the functional differential graph even in a high-dimensional setting for both fully observed and discretely observed function paths. We illustrate the finite sample properties of our method through simulation studies. We also propose a competing method, the Joint Functional Graphical Lasso, which generalizes the Joint Graphical Lasso to the functional setting. Finally, we apply our method to EEG data to uncover differences in functional brain connectivity between a group of individuals with alcohol use disorder and a control group.
\end{abstract}


\section{Introduction}\label{sec:intro}
We consider a setting where we observe two samples of multivariate functional data, $X_i(t)$ for $i = 1, \ldots, n_X$ and $Y_i(t)$ for $i = 1, \ldots, n_Y$. The primary goal is to determine if and how the underlying populations---specifically their conditional dependency structures---differ. As a motivating example, consider electroencephalography (EEG) data, where the electrical activity of
multiple regions of the brain can be measured simultaneously over a period of time. Given samples from the general population, fitting a graphical model to the observed measurements would allow a researcher to determine which regions of the brain are dependent after conditioning on all other regions. The EEG data analyzed in Section~\ref{subsec:EEG} consists of two samples: one from a control group and the other from a group of individuals with alcohol use disorder (AUD). Using these data, researchers may be interested in explicitly comparing the two groups and investigating the complex question of how brain functional connectivity patterns in the AUD group differ from those in the
control group.

The conditional independence structure within multivariate data is commonly represented by a graphical model \citep{lauritzen1996graphical}. Let $G = \{V, E\}$ denote an undirected graph where $V$ is the set of vertices with $|V| = p$ and $E \subset V^2$ is the set of edges. At times, we also denote $V$ as $[p]=\{1,2,\dots,p\}$. When the data consist of random vectors $X=(X_1,\dots,X_p)^{\top}$, we say that $X$ satisfies the pairwise Markov property with respect to $G$ if $X_{v} \centernot \independent X_{w} \mid \{X_u\}_{u \in V\setminus \{v,w\}}$ holds if and only if $\{v,w\} \in E$.  When $X$ follows a multivariate Gaussian distribution with covariance $\Sigma = \Theta^{-1}$, then $\Theta_{vw} \neq 0$ if and only if $\{v,w\} \in E$. Thus, recovering the structure of an undirected graph from multivariate Gaussian data is equivalent to estimating the support of the precision matrix, $\Theta$.

When the primary interest is in characterizing the difference between the conditional independence structure of two populations, the object of interest may be the \emph{differential graph}, $G_\Delta = \{V, E_\Delta\}$. When $X$ and $Y$ follow multivariate normal distributions with covariance matrices $\Sigma^{X}$ and $\Sigma^{Y}$, let $\Delta = \Theta^X - \Theta^Y$, where $\Theta^X = (\Sigma^X)^{-1}$ and $\Theta^Y = (\Sigma^Y)^{-1}$ are the precision matrices of $X$ and $Y$, respectively. The differential graph is then defined by letting $E_\Delta = \left\{\{v, w\} \, : \, \Delta_{v,w} \neq 0\right\}$. This type of differential model for vector-valued data has been adopted in \cite{zhao2014direct}, \cite{xu2016semiparametric}, and \cite{cai2017global}.

In the motivating example of EEG data, electrical activity is observed over a period of time. When the measurements smoothly vary over time, it may be more natural to consider the observations as arising from an underlying function. This is particularly true when data from different subjects are observed at different time points. Furthermore, when characterizing conditional independence, it is likely that the activity of each region depends not only on what is occurring simultaneously in the other regions but also on what has previously occurred in other regions; this suggests that a functional graphical model might be appropriate. 

In this paper, we define a differential graph for functional data that we refer to as a functional differential graphical model. Similar to differential graphs for vector-valued data, functional differential graphical models characterize the differences in the conditional dependence structures of two distributions of multivariate curves. We build on the functional graphical model developed in \cite{Qiao2015Functional}. However, while \cite{Qiao2015Functional} required that the observed functions lie in a finite-dimensional space in order for the functional graphical model to be well defined, the functional differential graphical models may be well defined even in certain cases where the observed functions live in an infinite-dimensional space. 

We propose an algorithm called FuDGE to estimate the differential graph and show that this procedure enjoys many benefits, similar to differential graph estimation in the vector-valued setting. Most notably, we show that under suitable conditions, the proposed method can consistently recover the differential graph even in the high-dimensional setting where $p$, the number of observed variables, may be larger than $n$, the number of observed samples.

A conference version of this paper was presented at the Conference on Neural Information Processing Systems \citep{zhao2019direct}. Compared to the conference version, this paper includes the following new results.
\begin{itemize}
\item We give a new definition for a differential graph for functional data, which allows us to circumvent the unnatural assumption made in the previous version and take a truly functional approach. Specifically, instead of defining the differential graph based on the difference between conditional covariance functions, we use the limit of the norm of the difference between finite-dimensional precision matrices. 

\item We include new theoretical guarantees for discretely observed curves. In practice, we can only observe the functions at discrete time points, so this extends the theoretical guarantees to a practical estimation procedure. Discrete observations bring an additional source of error when the estimated curves are used in the functional PCA. In Theorem~\ref{Thm:ErrSamCovDis}, we give an error bound for estimating the covariance matrix of the PCA score vectors under mild conditions.

\item We introduce the Joint Functional Graphical Lasso, which is a generalization of the Joint Graphical Lasso \citep{Danaher2011Joint} to the functional data setting. Empirically, we show that the procedure performs competitively in some settings but is generally outperformed by the FuDGE procedure. 
\end{itemize}

The software implementation can be found at \url{https://github.com/boxinz17/FuDGE}. The repository also contains the code to reproduce the simulation results.

\subsection{Related Work}
The work we develop lies at the intersection of two different lines of literature: graphical models for functional data and direct estimation of differential graphs.

Many previous works have studied the structure estimation of a static undirected graphical model \citep{chow68approximating,Yuan2007Model,Cai2011Constrained,Meinshausen2006High,Kolar2012Consistent,Wang2016Inference,Vogel2011Elliptical,Sun2015Learning,Suggala2017Expxorcist}. Previous methods have also been proposed to characterize conditional independence for multivariate observations recorded over time. For example, \cite{talih2005structural}, \cite{Xuan2007modeling}, \cite{amr09tesla}, \citet{le09keller}, \citet{song09time}, \citet{Kolar2010Estimating}, \citet{kolar09nips_tv_paper}, \cite{kolar09sparsistent}, \citet{Zhou08time}, \cite{yin10nonparametric}, \citet{kolar10nonparametric}, \citet{kolar2011time}, \citet{kolar10estimating}, \cite{Wang2014Inference}, \citet{Lu2015Posta}, \citet{Geng2018Joint}, \citet{Geng2019Partially}, \citet{Tsai2020Nonconvex} studied methods for dynamic graphical models that assume that data are sampled independently at different time points but generated by related distributions.  In these works, the authors proposed procedures to estimate a series of graphs that represent the conditional independence structure at each time point; however, they assumed that the observed data do not encode ``longitudinal'' dependence. In contrast, \citet{Wang2020Statistical} focused on graphical models for time series data, while \citet{Qiao2015Functional}, \citet{zhu2016bayesian}, \citet{Li2018nonparametric}, \citet{Zhang2018Dynamic}, \citet{zhao2021high} considered the setting where the data are multivariate random functions. Most similar to the setting we consider, \citet{Qiao2015Functional} assumed that the data are distributed as a multivariate Gaussian process (MGP) and use a graphical lasso type procedure on the estimated functional principal component scores. \citet{zhu2016bayesian} also assumed an MGP but proposed a Bayesian procedure. Crucially, however, both required that the covariance kernel can essentially be represented by a finite-dimensional object. \citet{zapata2021partial} showed that under various notions of separability---roughly when the covariance kernel can be decomposed into covariance across time and covariance across nodes---the conditional independence of the MGP is well defined even when the functional data are truly infinite-dimensional and that the conditional independence graph can be recovered by the union of a (potentially infinitely) countable number of graphs over finite-dimensional objects. \citet{zhao2021high} adopted a neighborhood selection approach to learn the conditional independence structure of an MGP, which does not need to assume that functional data are finite-dimensional or that the MGP is separable to ensure consistency. In a different approach, \citet{Li2018nonparametric} did not assume that random functions are Gaussian and instead used the notion of additive conditional independence to define a graphical model for random functions. \citet{qiao2020doubly} also assumed that the data are random functions, but allowed the dependency structure to change smoothly over time---similar to a dynamic graphical model.

We also draw on recent literature that has shown that when the object of interest is the difference between two distributions, directly estimating the difference can provide improvements over first estimating each distribution and then taking the difference. Most notably, when estimating the difference in graphs in a high-dimensional setting, even if each individual graph does not satisfy the appropriate sparsity conditions, the differential graph may still be recovered consistently. \citet{zhao2014direct} considered data drawn from two Gaussian graphical models and showed that even if both underlying graphs are dense, if the difference between the precision matrices of each distribution is sparse, the differential graph can still be recovered in the high-dimensional setting.  \citet{Liu2014Direct} proposed procedure based on KLIEP \citep{Sugiyama2008Direct} that estimates the differential graph by directly modeling the ratio of two densities. They did not assume Gaussianity but required that both distributions lie in some exponential family. \citet{Fazayeli2016Generalized} extended this idea to estimate the differences in Ising models. \citet{wang2018direct} also proposed direct difference estimators for directed graphs when data are generated by linear structural equation models that share a common topological ordering.

\subsection{Notation}

Let $\vert \cdot \vert_p $ denote the vector $p$-norm and $\Vert \cdot\Vert_p$ denote the matrix/operator $p$-norm.  For example, for a $p\times 1$ vector $a=(a_1,a_2,\dots,a_p)^{\top}$, we have $\vert a\vert_1=\sum_{j}\vert a_j \vert$, $\vert a\vert_2=(\sum_{j}\vert a^{2}_j \vert)^{1/2}$ and $\vert a \vert_{\infty}=\max_{j}\vert a_j \vert$. For a $p\times{q}$ matrix $A$ with entries $a_{jk}$, $|A|_{1}=\sum_{j,k}|a_{jk}|$, $\|A\|_{1}=\max_{k}\sum_{j}|a_{jk}|$, $|A|_{\infty}=\max_{j,k}|a_{jk}|$, and $\|A\|_{\infty}=\max_{j}\sum_{k}|a_{jk}|$. Let $\frobnorm{A}=(\sum_{j,k}a^{2}_{jk})^{1/2}$ be the Frobenius norm of $A$. When $A$ is symmetric, let $\tr(A)=\sum_{j}a_{jj}$ denote the trace of A. Let $\lambda_{\min}(A)$ and $\lambda_{\max}(A)$ denote the minimum and maximum eigenvalues, respectively. Let $a_{n}\asymp{b_{n}}$ denote $0<C_1\leq{\inf_{n}|a_{n}/b_{n}|}\leq{\sup_{n}|a_{n}/b_{n}|}\leq C_2<\infty$ for some positive constants $C_1$ and $C_2$.

We assume that all random functions belong to a separable Hilbert space $\mathbb{H}$. For any two functions $f_{1},f_{2} \in \mathbb{H}$, we define their inner product as $\langle f_1,f_2 \rangle=\int f_1(t)f_2(t)dt$. The induced norm is $\Vert f_1 \Vert=\Vert f_1 \Vert_{\mathcal{L}^2}=\{ \int f_{1}^{2}(t)
dt \}^{1/2}$.  

For a function vector $f(t)=(f_1(t),f_2(t),\dots,f_p(t))^{\top}$, we let $\Vert f \Vert_{\mathcal{L}^2,2}=(\sum^{p}_{j=1}\Vert f_j \Vert^2)^{1/2}$ denote its $\mathcal{L}^2,2$-norm.  For a bivariate function $g(s,t)$, we define the Hilbert-Schmidt norm of $g(s,t)$ as $\|g\|_{\text{HS}}=\int\int \{g(s,t)\}^2dsdt$. Typically, we will use $f(\cdot)$ (and similarly $g(\cdot ,*)$) to denote the entire function $f$, while we use $f(t)$ (and similarly $g(s,t)$) to mean the value of $f$ evaluated at $t$. 

For a vector space $\mathbb{V}$, we use $\mathbb{V}^{\bot}$ to denote its orthogonal complement. For $v_1,\ldots,v_K \in \mathbb{V}$ and $v=(v_1,\ldots,v_K)^{\top}$, we use ${\rm Span}\left\{ v_1,v_2,\dots,v_K \right\} ={\rm Span}\left(v\right)$ to denote the vector subspace spanned by $v_1,\ldots,v_K$.

\section{Functional Differential Graphical Models}

In this section, we review functional graphical models and introduce the notion of a functional differential graphical model.

\subsection{Functional Graphical Model}\label{sec:FGM}

Suppose $X_i(\cdot)=\left(X_{i1}(\cdot),X_{i2}(\cdot),\dots,X_{ip}(\cdot)\right)^{\top}$ is a p-dimensional \emph{multivariate Gaussian process (MGP)} with mean zero and common domain $\mathcal{T}$, where $\mathcal{T}$ is a closed interval of the real line with length $\lvert \mathcal{T} \rvert$.\footnote{We assume mean zero and a common domain $\mathcal{T}$ to simplify the notation, but the methodology and theory generalize to non-zero means and different time domains.} Each observation, for $i = 1, 2, \ldots, n$, is i.i.d. In addition, assume that for $j \in V$, $X_{ij}(\cdot)$ is a random element of a separable Hilbert space $\mathbb{H}$. \citet{Qiao2015Functional}, define the conditional cross-covariance function for $X_i(\cdot)$ as 
\begin{equation}{\label{con:CX_jl}}
C^{X}_{jl}(s,t) \; = \; \Cov\left(X_{ij}(s), X_{il}(t) \, \mid \, \{X_{ik}(\cdot)\}_{k \neq j,l}\right).
\end{equation}
If $C^{X}_{jl}(s,t) = 0$ for all $s,t\in\mathcal{T}$, then the random functions $X_{j}(\cdot)$ and $X_{l}(\cdot)$ are conditionally independent given the other random functions, and the graph $G_X = \{V, E_X\}$ represents the pairwise Markov property of $X_i(\cdot)$ if 
\begin{equation}\label{eq:singleFGdef}
E_{X} = \left\{ ( j,l) \,: \, j < l   \text{ and } \Vert C^{X}_{jl}\Vert _{\text{HS}} \neq 0 \right\}.
\end{equation}

In general, we cannot directly estimate \eqref{eq:singleFGdef}, since $X_{i}(\cdot)$ may be an infinite-dimensional object. Thus, before applying a statistical estimation procedure, dimension reduction is typically required. \citet{Qiao2015Functional} used \emph{functional principal component analysis} (FPCA) to project each observed function onto an orthonormal function basis defined by a finite number of eigenfunctions. Their procedure then estimates the conditional independence structure from the ``projection scores'' of this basis. We outline their approach in the following. However, in contrast to \citet{Qiao2015Functional}, we do not restrict ourselves to dimension reduction by projecting onto the FPCA basis, and in our discussion we instead consider a general function subspace. 

Let $\mathbb{V}^{M_j}_j \subseteq \mathbb{H}$ be a subspace of a separable Hilbert space $\mathbb{H}$ with dimension $M_j \in \mathbb{N}^{+}$ for all $j=1,2,\dots,p$. Our theory easily generalizes to the setting where $M_j$ may differ, but to simplify the notation, we assume $M_j = M$ for all $j$ and simply write $\mathbb{V}^{M}_j$ instead of $\mathbb{V}^{M_j}_j$. Let $\mathbb{V}^M_{[p]} \coloneqq \mathbb{V}^M_1 \otimes \mathbb{V}^M_2 \otimes \dots \otimes \mathbb{V}^M_p$.

For any function $g(\cdot) \in \mathbb{H}$ and a subspace $\mathbb{F} \subseteq \mathbb{H}$, let $\pi(g(\cdot);\mathbb{F}) \in \mathbb{F}$ denote the projection of the function $g(\cdot)$ onto the subspace $\mathbb{F}$, and let
\begin{equation*}
\pi(X_{i}(\cdot);\mathbb{V}^M_{[p]})=\left(  \pi (X_{i1}(\cdot);\mathbb{V}^M_1), \pi (X_{i2}(\cdot);\mathbb{V}^M_2), \dots, \pi (X_{ip}(\cdot);\mathbb{V}^M_p)  \right)^{\top}.    
\end{equation*}
When the choice of subspace is clear from the context, we will use the following shorthand notation:
$X^{\pi}_{ij}(\cdot)=\pi (X_{ij}(\cdot);\mathbb{V}^M_j)$, 
$j=1,2,\dots,p$, and $X^{\pi}_{i}(\cdot)=\pi(X_{i}(\cdot);\mathbb{V}^M_{[p]})$.

Similarly to the definitions in \eqref{con:CX_jl} and \eqref{eq:singleFGdef}, we define the conditional independence graph of $X^\pi(\cdot)$ as
\begin{equation}\label{eq:singleFGdef-pi}
E^{\pi}_{X} = \left\{ \{ j,l \} \,: \, j < l   \text{ and } \Vert C^{X,\pi}_{jl}\Vert _{\text{HS}} \neq 0 \right\},
\end{equation}
where
\begin{equation}{\label{con:CX_jl-pi}}
C^{X,\pi}_{jl}(s,t) \; = \; \Cov\left(X^{\pi}_{ij}(s), X^{\pi}_{il}(t) \, \mid \, \{X^{\pi}_{ik}(\cdot)\}_{k \neq j,l}\right).
\end{equation}
Note that $E^{\pi}_{X}$ depends on the choice of $\mathbb{V}^M_{[p]}$ through the projection operator $\pi$, and, as we discuss below, $E_X^\pi$ may be recovered from the observed samples.

When data arise from an MGP, we can estimate the projected graphical structure by studying the precision matrix of projection score vectors (defined below) with \emph{any} orthonormal function basis of the subspace $\mathbb{V}^M_{[p]}$. Let $e^M_{j}=( e_{j1}(\cdot),e_{j2}(\cdot),\dots,e_{jM}(\cdot) )^{\top}$ be any orthonormal function basis of $\mathbb{V}^M_j$ and let $e^M(\cdot)=\{ e^M_{j} \}^p_{j=1}$ be an orthonormal function basis of $\mathbb{V}^M_{[p]}$.
Let
\begin{equation*}
a^{X}_{ijk}=\int_{\mathcal{T}} X_{ij}(t) e_{jk}(t)dt
\end{equation*}
denote the projection score of $X_{ij}(\cdot)$ onto $e_{jk}(\cdot)$ and let 
\begin{align*}
    a^{X,M}_{ij}=(a^{X}_{ij1}, a^{X}_{ij2}, \dots, a^{X}_{ijM})^{\top} \; \text{ and }  \; a^{X,M}_{i}=((a^{X,M}_{i1})^{\top},\ldots,(a^{X,M}_{ip})^{\top})^{\top}\in{\mathbb{R}^{pM}}. 
\end{align*}
Since $X_{i}(\cdot)$ is a $p$-dimensional MGP, $a^{X,M}_{i}$ follows a multivariate Gaussian distribution and we denote the covariance matrix of that distribution as $\Sigma^{X,M}=(\Theta^{X,M})^{-1} \in \mathbb{R}^{pM \times pM}$. Each function $X_{ij}(\cdot)$ is associated with $M$ rows and columns of $\Sigma^{X,M}$ corresponding to $a_{ij}^{X,M}$. We use $\Theta_{jl}^{X,M}$ to refer to the $M\times M$ submatrix of $\Theta^{X,M}$ that corresponds to the functions $X_{ij}(\cdot)$ and $X_{il}(\cdot)$. Lemma~\ref{lemma:score-precision-pi}, from \citet{Qiao2015Functional}, shows that the conditional independence structure of the projected functional data can be obtained from the block sparsity of $\Theta^{X,M}$. 

\begin{lemma}{[\citet{Qiao2015Functional}]}\label{lemma:score-precision-pi}
Let $\Theta^{X,M}$ be the inverse covariance of the projection scores. Then, $X^{\pi}_{ij}(s) \independent X^{\pi}_{il}(t) \mid \{X^{\pi}_{ik}(\cdot)\}_{k \neq j, l}$ for all\footnote{More precisely, we only need the conditional independence to hold for all $s, t \in {\cal T}$ except for a subset of $\mathcal{T}^2$ with zero measure.}  $s, t \in {\cal T}$ if and only if $\Theta_{jl}^{X,M} \equiv 0$. 
This implies that $E^{\pi}_X$---as defined in
\eqref{eq:singleFGdef-pi}---can be equivalently defined as
\begin{equation}\label{eq:finiteDimRep-pi}
E^{\pi}_{X}\;=\; \left\{ \{j,l\} \,: \, j < l \text{ and } \Vert \Theta^{X,M}_{jl} \Vert_F \neq 0\right\}.
\end{equation}
\end{lemma}

Although \citet{Qiao2015Functional} only considered projections onto the span of the FPCA basis (that is, the eigenfunctions of $X_{ij}(\cdot)$ corresponding to $M$ largest eigenvalues), the result trivially extends to the more general case of \emph{any subspace} and \emph{any orthonormal function basis} of that subspace. 

Although $\Theta^{X,M}$ depends on the specific basis onto which $X_i(\cdot)$ is projected, the edge set $E^\pi_X$ only depends on the subspace $\mathbb{V}^M_{[p]}$, that is, the span of the basis onto which $X_i(\cdot)$ is projected. Thus, Lemma~\ref{lemma:score-precision-pi} implies that although the entries of $\Theta^{X,M}$ can change when using different orthonormal function bases to represent $\mathbb{V}^M_{[p]}$, the block sparsity pattern of $\Theta^{X,M}$ only depends on the span of the selected basis.

When $X_i(\cdot) \neq X^\pi_i(\cdot)$, $E^{\pi}_{X}$ may not be the same as $E_{X}$; furthermore, it may not be the case that $E^{\pi}_{X} \subseteq E_{X}$ or $E_{X} \subseteq E^{\pi}_{X}$. Thus, Condition 2 of \citet{Qiao2015Functional} requires a finite $M^\star<\infty$ such that $X_{ij}$ lies in $\mathbb{V}^{M^\star}_{[p]}$ almost surely. When $M = M^\star$, then $X_i(\cdot) = X_i^\pi(\cdot)$ and $E^{\pi}_X=E_X$. 
Under this assumption, to estimate $E^{\pi}_X = E_X$, \cite{Qiao2015Functional} proposed the functional graphical lasso estimator (fglasso), which solves the following objective:
\begin{equation}
\label{eq:FGL_objective}
\hat{\Theta}^{X,M}=\argmax_{\Theta^{X,M}}{\left\{\log{\text{det} \left(\Theta^{X,M}\right)} - \tr\left(S^{X,M}\Theta^{X,M}\right) - \gamma_{n}\sum_{j \neq l}{\frobnorm{\Theta^{X,M}_{jl}}} \right\}}.
\end{equation}
In \eqref{eq:FGL_objective}, $\Theta^{X,M}$ is a symmetric positive definite matrix, $\Theta^{X,M}_{jl} \in \mathbb{R}^{M \times M}$ corresponds to the $(j,l)$ submatrix of $\Theta^{X,M}$, $\gamma_n$ is a non-negative tuning parameter, and $S^{X,M}$ is an estimator of $\Sigma^{X,M}$. The matrix $S^{X,M}$ is obtained by using FPCA on the empirical covariance functions (see Section~\ref{sec:choose-sub-fpca} for details). The resulting estimated edge set for the functional graph is
\begin{equation}
\hat{E}_{X}^\pi=\left\{\{j,l\} \, :\, j < l  \text{ and } \frobnorm{\hat{\Theta}^{X,M}_{jl}} > 0 \right\}.
\end{equation}
We also note that the objective in \eqref{eq:FGL_objective} was previously used in \citet{kolar13multiatticml} and \citet{Kolar2014Graph} for the estimation of graphical models from multi-attribute data.

However, the requirement that $X_i(\cdot)$ lies in a subspace with finite-dimension may be violated in many practical applications and negates one of the primary benefits of considering the observations as functions. Unfortunately, the extension to infinite-dimensional data is nontrivial, and indeed Condition 2 in \citet{Qiao2015Functional} requires that the observed functional data lie within a finite-dimensional span. To see why, we first note that $\Sigma^{X,M^\star}$ is always a compact operator on $\mathbb{R}^{pM^\star}$. Thus, as $M^\star \to \infty$, the smallest eigenvalue of $\Sigma^{X,M^\star}$ will go to zero. As a consequence, $\Sigma^{X,M^\star}$ becomes increasingly ill-conditioned, and $\Theta^{X,M^\star}$, the inverse of $\Sigma^{X,M^\star}$ will become ill-defined when $M^\star=\infty$. This behavior makes the estimation of a functional graphical model---at least through the basis expansion approach proposed by \citet{Qiao2015Functional}---generally infeasible for truly infinite-dimensional functional data. When the data are truly infinite-dimensional, the best we can do is to estimate a finite-dimensional approximation and hope that it captures the relevant information.

\subsection{Functional Differential Graphical Models: Finite-Dimensional Setting}
\label{sec:finiteDim}

In this paper, rather than estimating the conditional independence structure of a single MGP, we are interested in characterizing the difference between two MGPs, $X$ and $Y$.  For brevity, we will typically only explicitly define the notation for $X$; however, the reader should infer that all the notation for $Y$ is defined analogously. As described in the introduction, \citet{li2007finding} and \citet{zhao2014direct} consider the setting where $X$ and $Y$ are multivariate Gaussian vectors, and define the differential graph $G_\Delta = \{V, E_\Delta\}$ by letting
\begin{equation}\label{eq:vecValDiff}
    E_\Delta = \left\{ (v,w) \,: \, v < w \text{ and } \Delta_{vw} \neq 0\right\}
\end{equation}
where $\Delta = (\Sigma^X)^{-1} - (\Sigma^Y)^{-1}$ and $\Sigma^X,\Sigma^Y$ are the covariance matrices of $X$ and $Y$.

We extend this definition to the functional data setting and define functional differential graphical models. To develop intuition, we first start by defining the differential graph with respect to the finite-dimensional projections of functional data, that is, with respect to $X^{\pi}_{i}(t)$ and $Y^{\pi}_{i}(t)$ for some choice of $\mathbb{V}^M_{[p]}$. As implied by Lemma~\ref{lemma:score-precision-pi}, in the functional graphical model setting, the $M\times M$ blocks of the precision matrix of the projection scores play a similar role to the individual entries of a precision matrix in the vector-valued Gaussian graphical model setting. Thus, we also define a functional differential graphical model by the difference of the precision matrices of the projection scores. Note that for each $j \in V$, we require that both $a^X_{ij}$ and $a^Y_{ij}$ be calculated using the same function basis of $\mathbb{V}^M_j$. Let $\Theta^{X,M}=\left( \Sigma^{X,M} \right)^{-1}$ and $\Theta^{Y,M}=\left( \Sigma^{Y,M} \right)^{-1}$ be the precision matrices for the projection scores for $X$ and $Y$, respectively, where the inverse should be understood as the pseudo-inverse when $\Sigma^{X,M}$ or $\Sigma^{Y,M}$ are not invertible. 

We now define the functional differential graphical model.
Let $\Delta^M = \Theta^{X,M} - \Theta^{Y,M}$ and $\Delta^{M}_{jl}$ be the $(j,l)$-th $M \times M$ block of $\Delta^{M}$. We define the edges of the functional differential graph of the projected data as:
\begin{equation}\label{eq:diffGraphDef}
E^{\pi}_\Delta \, = \, \left\{ (j,l) \, :  \, j < l \text{ and } \, \Vert \Delta^{M}_{jl}\Vert_F > 0 \right\}.
\end{equation}

While the entries of $\Delta^M$ depend on the choice of orthonormal function basis, the definition of $E^{\pi}_\Delta$ is invariant to the particular basis and only depends on the span. The following lemma formally states this result.

\begin{lemma}\label{lemma:Delta-indp-funB}
Suppose that ${\rm span}(e^M(\cdot)) = {\rm span}(\tilde{e}^M(\cdot))$
for two orthonormal bases $e^M(\cdot)$ and $\tilde{e}^M(\cdot)$. Let $E_\Delta^\pi$ and $E_\Delta^{\tilde{\pi}}$ be defined by \eqref{eq:diffGraphDef} when projecting $X$ and $Y$ onto $e^M(\cdot)$ and $\tilde{e}^M(\cdot)$, respectively. Then, $E_\Delta^\pi = E_\Delta^{\tilde{\pi}}$.
\end{lemma}
\begin{proof}
See Appendix~\ref{sec:proof-lemma-delta-indp-funB}.
\end{proof}

We have several comments about $E^{\pi}_\Delta$ defined in \eqref{eq:diffGraphDef}.

\textbf{Projecting $X$ and $Y$ onto different subspaces:} 
While we project both $X$ and $Y$ onto the same subspace  $\mathbb{V}^M_{[p]}$, our definition can be easily generalized to  a setting where we project $X$ onto $\mathbb{V}^{X,M}_{[p]}$ and $Y$ onto $\mathbb{V}^{Y,M}_{[p]}$, with $\mathbb{V}^{X,M}_{[p]} \neq \mathbb{V}^{Y,M}_{[p]}$. For example, naively following the procedure of \citet{Qiao2015Functional}, we could perform FPCA on $X$ and $Y$ separately, and subsequently we could use the difference between the precision matrices of the projection scores to define the functional differential graph. Although defining the functional differential graph using this alternative approach may be suitable for some applications, it may result in the undesirable case where $(j,l) \in E_\Delta^\pi$ even though $C_{jl}^{X,\pi}(\cdot ,* ) = C_{jl}^{Y,\pi}(\cdot ,* )$, $C_{jj}^{X,\pi}(\cdot ,* ) = C_{ll}^{Y,\pi}(\cdot ,* )$, and $C_{ll}^{\setminus j, X,\pi}(\cdot ,* ) = C_{ll}^{\setminus j, Y,\pi}(\cdot ,* )$. Therefore, we restrict our discussion to the setting where $X$ and $Y$ are projected onto the same subspace.

\textbf{Connection to Multi-Attribute Graphical Models:} 
The selection of a specific functional subspace is connected to multi-attribute graphical models \citep{Kolar2014Graph}. If we treat the random function $X_{ij}(\cdot)$ as representing an infinite number of attributes, then $X^{\pi}_{ij}(\cdot)$ will be an approximation using $M$ attributes. The chosen attributes are given by the subspace $\mathbb{V}^M_j$. While we allow different nodes to choose different attributes by allowing $\mathbb{V}^M_j$ to vary across $j$, we require that the same attributes are used to represent both $X$ and $Y$ by restricting $\mathbb{V}^M_{[p]}$ to be the same for $X$ and $Y$. The specific choice of $\mathbb{V}^M_{[p]}$, can extract different attributes from the data. For instance, using the subspace spanned by the Fourier basis can be viewed as extracting frequency information, while using the subspace spanned by the eigenfunctions---as introduced in the next section---can be viewed as extracting the dominant modes of variation.

Given the definition~\eqref{eq:diffGraphDef} and the Lemma~\ref{lemma:Delta-indp-funB}, there are two main questions to answer: First, how do we choose $\mathbb{V}^M_{[p]}$? Second, what happens when $X$ and $Y$ are infinite-dimensional? We answer the first question in Section~\ref{sec:choose-sub-fpca} and the second question in Section~\ref{sec:infiniteDim}.

\subsection{Choosing Functional Subspace via FPCA}
\label{sec:choose-sub-fpca}

As discussed in Section~\ref{sec:finiteDim}, the choice of $\mathbb{V}^M_{[p]}$ in Definition~\ref{eq:diffGraphDef} decides---roughly speaking---the attributes or dimensions in which we compare the conditional independence structures of $X$ and $Y$. In some applications, we may have very good prior knowledge about this choice. However, in many cases, we may not have a strong prior knowledge. In this section, we describe our recommended ``default choice'' that uses FPCA on the combined $X$ and $Y$ observations. In particular, suppose that there exist subspaces $\{\mathbb{V}^{M^\star}_j\}_{j \in V}$ such that $\mathbb{V}^{M^\star}_j$ has dimension $M^\star < \infty$ and $X_{ij}(t),Y_{ij}(t)\in \mathbb{V}^{M^\star}_j$ for all $j \in V$. Then, FPCA---when given population values---recovers this subspace.

Similarly to the way principal component analysis provides the $L_2$ optimal lower dimensional representation of vector-valued data,  FPCA provides the $L_2$ optimal finite-dimensional representation of functional data. Let $K^{X}_{jj}(t,s)=\Cov(X_{ij}(t),X_{ij}(s))$ denote the covariance function for $X_{ij}$ for $j \in V$. Then, there exist orthonormal eigenfunctions and eigenvalues $\{\phi^X_{jk}(t), \lambda^{X}_{jk} \}_{k \in \mathbb{N}}$ such that $\int_{\mathcal{T}}K^{X}_{jj}(s,t)\phi_{jk}^{X}(t)dt=\lambda_{jk}^{X}\phi_{jk}^{X}(s)$ for all $k \in \mathbb{N}$ \citep{Hsing2015Theoretical}. Since $K^{X}_{jj}(s,t)$ is symmetric and non-negative definite, we assume, without loss of generality, that $\{\lambda^{X}_{js}\}_{s \in \mathbb{N}^+}$ is non-negative and non-increasing. By the Karhunen-Lo\`{e}ve expansion \citep[Theorem7.3.5]{Hsing2015Theoretical}, $X_{ij}(t)$ can be expressed as $X_{ij}(t)=\sum_{k=1}^{\infty}a^X_{ijk}\phi^{X}_{jk}(t)$, where the principal component scores satisfy $a^X_{ijk} =\int_{\mathcal{T}}X_{ij}(t)\phi^{X}_{jk}(t)dt$ and $a^X_{ijk} \sim N(0, \lambda_{jk}^X)$ with $E(a^X_{ijk} a^X_{ijl}) = 0$ if $k \neq l$. Because the eigenfunctions are orthonormal, the $L_2$ projection of $X_{ij}$ onto the span of the first $M$ eigenfunctions is $X^{M}_{ij}(t)=\sum_{k=1}^{M}a^{X}_{ijk}\phi^{X}_{jk}(t)$. Similarly, we can define $K^Y_{jj}(t,s)$, $\{\phi^Y_{jk}(t), \lambda^Y_{jk} \}_{k \in \mathbb{N}}$ and $Y^M_{ij}(t)$. Let $K_{jj}(s,t)=K^X_{jj}(s,t)+K^Y_{jj}(s,t)$ and let $\{\phi_{jk}(t), \lambda_{jk} \}_{k \in \mathbb{N}}$ be the eigenfunction-eigenvalue pairs of $K_{jj}(s,t)$. 

Lemma~\ref{lemma:M-dim-subspace-find} implies that $X_{ij}(\cdot)$ and $Y_{ij}(\cdot)$ lie within the span of the eigenfunctions corresponding to the non-zero eigenvalues of $K_{jj}$. Furthermore, this subspace is minimal in the sense that no subspace of a smaller dimension contains $X_{ij}(\cdot)$ and $Y_{ij}(\cdot)$ almost surely. Thus, the FPCA basis of $K_{jj}$ provides a good default choice for dimension reduction.

\begin{lemma}\label{lemma:M-dim-subspace-find}
Let $\vert  \mathbb{V} \vert$ denote the dimension of a subspace $\mathbb{V}$ and suppose that 
\[
M^{\prime}_j=\inf\{ \vert  \mathbb{V} \vert : \mathbb{V} \subseteq \mathbb{H}, X_{ij}(\cdot),Y_{ij}(\cdot) \in \mathbb{V} \, \text{almost surely} \}. 
\]
Let $\{\phi_{jk}(t), \lambda_{jk} \}_{k \in \mathbb{N}}$ be 
the eigenfunction-eigenvalue pairs of $K_{jj}(s,t)$ and 
\[M^\star_j = \sup\{ M \in \mathbb{N}^{+}: \lambda_{jM } > 0 \}.\]
Then $M^{\prime}_j=M^\star_j$ and $X_{ij},Y_{ij} \in {\rm Span}\{ \phi_{j1}(\cdot), \phi_{j2}(\cdot), \dots, \phi_{j,M^\star_j}(\cdot) \}$ almost surely.
\end{lemma}
\begin{proof}
See Appendix~\ref{sec:proof-lemma-M-dim-subspace-find}.
\end{proof}

\subsection{Infinite-Dimensional Functional Data}
\label{sec:infiniteDim}

In Section~\ref{sec:finiteDim}, we defined a functional differential graph for functional data that have finite-dimensional representation. In this section, we present a more general definition that also allows for infinite-dimensional functional data.

As discussed in Section~\ref{sec:FGM}, when the data are infinite-dimensional, estimating a functional graphical model is not straightforward because the precision matrix of the scores does not have a well-defined limit as $M$, the dimension of the projected data, increases to $\infty$. When estimating the differential graph, however, although $\Vert \Theta^{X,M} \Vert_{\text{F}} \to \infty$ and $\Vert \Theta^{Y,M} \Vert_{\text{F}} \to \infty$ as $M \to \infty$, it is still possible for $\Vert \Theta^{X,M} - \Theta^{Y,M} \Vert_{\text{F}}$ to be bounded as $M \to \infty$. For instance, $x_n,y_n \in \mathbb{R}$ may both tend to infinity, but $\lim_n x_n-y_n$ may still exist and be bounded. Furthermore, even when $\Vert \Theta^{X,M} - \Theta^{Y,M} \Vert_{\text{F}} \rightarrow \infty$, it is still possible for the difference $\Theta^{X,M} - \Theta^{Y,M}$ to be informative. This observation leads to Definition~\ref{def:DGO} below. To simplify notation, in the rest of the paper, we assume that $X_{ij}(\cdot)$ and $Y_{ij}(\cdot)$ live in an $M^\star$ dimensional space where $ M^\star \leq \infty$. Recall that $\{\phi^X_{jk}(\cdot), \lambda^X_{jk} \}_{k \in \mathbb{N}}$  and $\{\phi^Y_{jk}(\cdot), \lambda^Y_{jk} \}_{k \in \mathbb{N}}$ denote the eigenpairs of $K_{jj}^X$ and $K_{jj}^Y$ respectively. 

\begin{definition}[Differential Graph Matrix and Comparability]\label{def:DGO}
The MGPs $X$ and $Y$ are \textbf{comparable} if the following two conditions hold:
\begin{enumerate}
    \item For all $j \in [p]$, $K_{jj}^X$ and $K_{jj}^Y$ have $M^\star$ non-zero eigenvalues and \[\mathrm{ span}\left(\{\phi_{jk}^X\}_{k=1}^{M^\star} \right) = \mathrm{span}\left(\{\phi_{jk}^Y\}_{k=1}^{M^\star} \right).\]
    \item For every $(j,l) \in V^2$ where $j\neq l$ and a projection subspace sequence $\left\{\mathbb{V}^M_{[p]}\right\}_{M \geq 1}$ satisfying $\lim_{M \to M^{\star}}\mathbb{V}^M_j=\mathrm{ span}\left(\{\phi_{jk}^X\}_{k=1}^{M^\star} \right)$, we have either:
\begin{equation*}
\lim_{M \to M^\star} \Vert \Delta^M_{jl} \Vert_{\text{F}} = 0 
\qquad\text{or}\qquad
\lim\inf_{M \to M^\star} \Vert \Delta^M_{jl} \Vert_{\text{F}} > 0.
\end{equation*}
\end{enumerate} 

We say that $X$ and $Y$ are \textbf{incomparable}, if for some $j$, $K_{jj}^X$ and $K_{jj}^Y$ have a different number of non-zero eigenvalues, or if $\mathrm{ span}\left(\{\phi_{jk}^X\}_{k=1}^{M^\star} \right) \neq \mathrm{span}\left(\{\phi_{jk}^Y\}_{k=1}^{M^\star} \right)$, or if there exists some $(j,l)$ such that given $\left\{\mathbb{V}^M_{[p]}\right\}_{M \geq 1}$ satisfying $\lim_{M \to M^{\star}}\mathbb{V}^M_j=\mathrm{ span}\left(\{\phi_{jk}^X\}_{k=1}^{M^\star} \right)$, we have
\begin{equation*}
\lim\inf_{M \to M^\star} \Vert \Delta^M_{jl} \Vert_{\text{F}} = 0, \qquad \text{but} \qquad
\lim\sup_{M \to M^\star} \Vert \Delta^M_{jl} \Vert_{\text{F}} > 0.
\end{equation*}

When $X$ and $Y$ are comparable, we define the \textbf{differential graph matrix} (DGM) $D=(D_{jl})_{(j,l) \in V^2}\in \mathbb{R}^{p\times p}$, where
\begin{equation}\label{eq:DGOdef}
D_{jl} = \lim\inf_{M \to M^\star} \Vert \Delta^M_{jl} \Vert_{\text{F}}.
\end{equation}

\end{definition}
In Definition~\ref{def:DGO} we say $\lim_{M \to M^{\star}}\mathbb{V}^M_j=\mathrm{ span}\left(\{\phi_{jk}^X\}_{k=1}^{M^\star} \right)$, to mean the following: For any $\epsilon>0$ and all $g \in \mathrm{ span}\left(\{\phi_{jk}^X\}_{k=1}^{M^\star} \right)$, there exists $M^{\prime} = M^{\prime}(\epsilon) < \infty$ such that $\Vert g -g^M_{P} \Vert < \epsilon$ for all $M \geq M^{\prime}$, where $g^M_{P}$ denotes the projection of $g$ onto the subspace of $\mathbb{V}^M_j$.

When $M^\star<\infty$, the conditional independence structure in $X_i$ and $Y_i$ can be fully captured by a finite-dimensional representation. When $M^\star=\infty$, as $M \to \infty$, $\Delta^{M}_{jl}$ approaches the difference of two matrices with unbounded eigenvalues. However, when $X$ and $Y$ are comparable, the limits are still informative. This would suggest that by using a sufficiently large subspace, we can capture such a difference arbitrarily well. However, if the MGPs are not comparable, then using a larger subspace may not improve the approximation regardless of the sample size. For this reason, in the
remainder of the article, we focus only on the setting where $X$ and $Y$ are comparable.

To our knowledge, there is no existing procedure to estimate a graphical model for functional data when the functions are infinite-dimensional. Thus, it is not straightforward to determine whether the comparability condition is stronger or weaker than what might be required for estimating the graphs separately and then comparing post hoc. However, we hope to provide some intuition to the reader.   

Suppose that $X$ and $Y$ are of the same dimension, $M^\star$. If $M^\star < \infty$ and the functional graphical model for each sample could be estimated separately (that is, $\Vert \Theta^{X,M} \Vert_{F} < \infty$ and $ \Vert \Theta^{Y,M} \Vert_{F} < \infty$), then $X$ and $Y$ are comparable when the minimal basis that spans $X$ and $Y$ is the same. Thus, the functional differential graph is also well defined. On the other hand, the conditions required by \citet[Condition 2]{Qiao2015Functional} for consistent estimation are not satisfied when $M^\star = \infty$, since $\lim_{M \rightarrow \infty} \Vert \Theta^{X,M} \Vert_{F} = \infty$ due to the compactness of the covariance operator. However, $X$ and $Y$ may still be comparable depending on the limiting behavior of $\Theta^{X,M}$ and $\Theta^{Y,M}$. Thus, there are settings where the differential graph may exist and can be consistently recovered even when each individual graph cannot be recovered (even when $p$ is fixed).  

However, when one MGP is finite-dimensional and the other is infinite-dimensional, then the MGPs are incomparable. To see this, without loss of generality, we assume that MGP $X$ has infinite-dimension $M^X_{j}=M^\star_X=\infty$ for all $j\in V$ and MGP $Y$ has finite-dimension $M^Y_{j}=M^\star_Y<\infty$ for all $j\in V$. Then $\Theta^{Y,M}$ is ill-defined when $M > M^\star_Y$ and recovering the differential graph is not straightforward.

We now define the notion of a functional differential graph.

\begin{definition}\label{def:funDGM}
  When two MGPs $X$ and $Y$ are comparable, we define their \textbf{functional differential graph} as an undirected graph $G_{\Delta}=\{V,E_{\Delta}\}$, where $E_{\Delta}$ is defined as
	\begin{equation}\label{eq:funcDGM}
	E_{\Delta}=\left\{ \{j,l\} \,:\, j < l \text{ and } D_{jl} > 0 \right\}.
	\end{equation}
\end{definition}

\begin{remark}
The functional graphical model defined by \citet{Qiao2015Functional} uses the conditional covariance function $C_{jl}^{X}(\cdot ,* )$ given in~\eqref{con:CX_jl}. Thus, it would be quite natural to use the conditional covariance functions directly to define a differential graph, where 
\begin{equation}\label{eq:directCovDiff}
    E_{\Delta} = \left\{\{j,l\}  \; : \; j <l \text{ and } C_{jl}^{X}(\cdot ,* ) \neq C_{jl}^{Y}(\cdot ,* )\right\}. 
\end{equation}
Unfortunately, this definition does not always coincide with the one we propose in Definition~\ref{def:funDGM}. However, the functional differential graph given in Definition~\ref{def:funDGM} has many nice statistical properties and retains important features of the graph defined in~\eqref{eq:directCovDiff}.

The primary statistical benefit of the graph defined in Definition~\ref{def:funDGM} is that it can be directly estimated without estimating each conditional independence function: $C^X_{jl}(\cdot, \cdot)$ and $C^Y_{jl}(\cdot, \cdot)$. Similarly to the vector-valued case considered by \citep{zhao2014direct}, this allows for a much lower sample complexity when each individual graph is dense but the difference is sparse. In some settings, there may not be enough samples to accurately estimate each individual graph, but the difference may still be recovered. This result is demonstrated in Theorem~\ref{Thm:smallboundThm}. 

The statistical advantages of our estimand unfortunately come at the cost of a slightly less precise characterization of the difference in the conditional covariance functions. However, many of the key characteristics are still preserved. Suppose $X_i$ and $Y_i$ are both $M^\star$-dimensional with $M^\star < \infty$ and further suppose that $\{\phi_{jm}(\cdot)\phi_{lm'}(*)\}_{m, m' \in [M^\star] \times [M^\star]}$ is a linearly independent set of functions. Suppose that the conditional covariance functions for $j, l \in V$ are unchanged so that $C_{jj}^{X}(\cdot ,* ) = C_{jj}^{Y}(\cdot , * )$ and $C_{ll}^{\backslash j,X}(\cdot ,* ) = C_{ll}^{\backslash j,Y}(\cdot ,* )$, where 
\[C_{ll}^{\backslash j,X}(\cdot ,* ) \coloneqq {\rm Cov}(X_l(\cdot),X_l(*) \,\vert\, X_k(\cdot), k \neq j,l)\]
and $C_{ll}^{\backslash j,Y}(\cdot ,* )$ is defined similarly; then, $\Delta_{jl} = 0$ if and only if $C_{jl}^{X}(\cdot ,* ) = C_{jl}^{Y,\pi}(\cdot ,* )$. When this holds for all pairs $j,l \in V$, then the definitions of a differential graph in Definition~\ref{def:funDGM} and \eqref{eq:directCovDiff} are equivalent. When the conditional covariance functions change so that $C_{jj}^{X}(\cdot ,* ) \neq C_{jj}^{Y}(\cdot ,* )$, then we still have $\Delta_{jl} \neq 0$ if $C_{jl}^{X,\pi}(\cdot ,* ) = 0$ and $C_{jl}^{Y,\pi}(\cdot ,* ) \neq 0$ (or vice versa). Thus, even in this more general setting, the functional differential graph given in Definition~\ref{def:funDGM} captures all qualitative differences between conditional covariance functions $C_{jl}^{X}(\cdot ,* )$ and $C_{jl}^{Y}(\cdot ,* )$.  
\end{remark}

Our objective is to directly estimate $E_{\Delta}$ without first estimating $E_X$ or $E_Y$. Since the functions we consider may be infinite-dimensional objects, in practice, what we can directly estimate is actually $E^{\pi}_{\Delta}$ defined in \eqref{eq:diffGraphDef}. We will use a sieve estimator to estimate $\Delta^{M}$, where $M$ increases with the sample size $n$. When $M^\star=M$, then $E^{\pi}_{\Delta}=E_{\Delta}$. When $M<M^\star\leq \infty$, then this is generally not true; however, we would expect the graphs to be similar when $M$ is large enough compared to $M^\star$. Thus, by constructing a suitable estimator of $\Delta^{M}$, we can still recover $E_{\Delta}$.

\subsection{Illustration of Comparability}

We provide a few examples that illustrate the notion of comparability. In the first two examples, the graphs are comparable, whereas in the third example, the graphs are incomparable. First, we state a lemma that will be helpful in the following discussion. The lemma follows directly from the properties of the multivariate normal and the inverse of block matrices. 

\begin{lemma}\label{lemma:lemma1}
Let $H^{X,M}_{jl}=\Cov (a^{X,M}_{ij}, a^{X,M}_{il} \mid a^{X,M}_{ik}, k \neq j,l )$ and $H^{\backslash l,X,M}_{jj}=\Var (a^{X,M}_{ij} \mid a^{X,M}_{ik}, k \neq j,l )$.
For any $j\in V$, we have $\Theta^{X,M}_{jj}=(H^{X,M}_{jj})^{-1}$. For any $(j,l) \in V^2$ and $j \neq l$, we have $
\Theta^{X,M}_{jl}=-(H^{X,M}_{jj})^{-1}H^{X,M}_{jl}(H^{\backslash j, X,M}_{ll})^{-1}$.
\end{lemma}
\begin{proof}
See Appendix~\ref{sec:proof-lemma-lemma1}.
\end{proof}

The following proposition follows directly from Lemma~\ref{lemma:lemma1}.

\begin{proposition}\label{prop:OpRepForm}
Assume that for any $(j,l) \in V^2$ and $j\neq l$, we have
\[
a^{X}_{ijm} \independent a^{X}_{ijm^{\prime}} \mid a^{X,M}_{ik},k\neq j \qquad\text{and}\qquad
a^{X}_{ijm}\independent a^{X}_{ijm^{\prime}} \mid a^{X,M}_{ik},k\neq j,l,
\]
for any $M$ and $1 \leq m \neq m^{\prime} \leq M$. We then have
\begin{equation*}
\Theta^{X,M}_{jj}={\rm diag} \left( \frac{1}{ {\rm Var}\left(a^{X}_{ij1}\mid a^{X,M}_{ik},k\neq j \right) }, \dots, \frac{1}{ {\rm Var}\left(a^{X}_{ijM}\mid a^{X,M}_{ik},k\neq j \right) }  \right)
\end{equation*}
and
\begin{equation*}
\Theta^{X,M}_{jl, m m^{\prime}} = \frac{ {\rm Cov}\left(a^{X}_{ijm},a^{X}_{ilm^{\prime}} \mid a^{X,M}_{ik},k\neq j,l \right) }{ {\rm Var}\left(a^{X}_{ijm} \mid a^{X,M}_{ik},k\neq j \right) {\rm Var}\left(a^{X}_{ilm^{\prime}}\mid a^{X,M}_{ik},k\neq j \right) }\overset{\Delta}{=} \bar{v}^{X,jl,M}_{mm^{\prime}},
\end{equation*}
for any $M$ and $1 \leq m \neq m^{\prime} \leq M$. 
In addition, if 
\[
a^{Y}_{ijm} \independent a^{Y}_{ijm^{\prime}} \mid a^{Y,M}_{ik},k\neq j \quad\text{and}\quad
a^{Y}_{ijm}\independent a^{Y}_{ijm^{\prime}} \mid a^{Y,M}_{ik},k\neq j,l,
\]
for any $M$ and $1 \leq m \neq m^{\prime} \leq M$, then 
\begin{align*}
 \Theta^{X,M}_{jj} - \Theta^{Y,M}_{jj} 
&= {\rm diag} \left( \left\{ \frac{ {\rm Var}\left(a^{Y}_{ijm}\mid a^{Y,M}_{ik},k\neq j \right) -  {\rm Var}\left(a^{X}_{ijm}\mid a^{X,M}_{ik},k\neq j \right) }{ {\rm Var}\left(a^{X}_{ijm}\mid a^{X,M}_{ik},k\neq j \right) {\rm Var}\left(a^{Y}_{ijm}\mid a^{Y,M}_{ik},k\neq j \right) } \right\}^M_{m=1} \right) \\
&\overset{\Delta}{=} {\rm diag} \left( \bar{w}^{j,M}_{1}, \bar{w}^{j,M}_{2}, \dots, \bar{w}^{j,M}_{M} \right)
\end{align*}
and
\begin{align*}
\Theta^{X,M}_{jl, m m^{\prime}} - \Theta^{Y,M}_{jl, m m^{\prime}} 
& = \frac{ {\rm Cov}\left(a^{X}_{ijm},a^{X}_{ilm^{\prime}} \mid a^{X,M}_{ik},k\neq j,l \right) }{ {\rm Var}\left(a^{X}_{ijm} \mid a^{X,M}_{ik},k\neq j \right) {\rm Var}\left(a^{X}_{ilm^{\prime}}\mid a^{X,M}_{ik},k\neq j \right) } \\
& \qquad \qquad\qquad - \frac{ {\rm Cov}\left(a^{Y}_{ijm},a^{Y}_{ilm^{\prime}} \mid a^{Y,M}_{ik},k\neq j,l \right) }{ {\rm Var}\left(a^{Y}_{ijm} \mid a^{Y,M}_{ik},k\neq j \right) {\rm Var}\left(a^{Y}_{ilm^{\prime}}\mid a^{Y,M}_{ik},k\neq j \right) } \\
& = \bar{v}^{Y,jl,M}_{mm^{\prime}} - \bar{v}^{X,jl,M}_{mm^{\prime}} \overset{\Delta}{=} \bar{z}^{jl,M}_{mm^{\prime}},
\end{align*}
for any $M$ and $1 \leq m \neq m^{\prime} \leq M$.
\end{proposition}

With the notation defined in Proposition~\ref{prop:OpRepForm},
we have that 
\begin{equation}
\Vert \Delta^M_{jj} \Vert^{2}_{\text{HS}}=\sum^{M}_{m=1}\left( \bar{w}^{j,M}_{m} \right)^2
\qquad\text{and}\qquad
\Vert \Delta^M_{jl} \Vert^{2}_{\text{HS}}=\sum^{M}_{m^{\prime}=1}\sum^{M}_{m=1}\left( \bar{z}^{jl,M}_{mm^{\prime}} \right)^2.
\end{equation}
As a result, we have the following condition for comparability.
\begin{proposition}\label{prop:CondforComp}
Under the assumptions in Proposition~\ref{prop:OpRepForm}, assume that MGPs $X$ and $Y$ are $M^{\star}$-dimensional, with $1\leq M^{\star} \leq \infty$, and lie in the same space. Then they are comparable if and only if for every $(j,l) \in V \times V$, we have either
\begin{equation}
\lim\inf_{M\to M^{\star}}\sum^{M}_{m^{\prime}=1}\sum^{M}_{m=1} \left( \bar{z}^{jl,M}_{mm^{\prime}} \right)^2>0
\qquad\text{or}\qquad
\lim_{M\to M^{\star}}\sum^{M}_{m^{\prime}=1}\sum^{M}_{m=1}\left( \bar{z}^{jl,M}_{mm^{\prime}} \right)^2=0,
\end{equation}
where $\bar{z}^{jl,M}_{mm^{\prime}}$ are defined in Proposition~\ref{prop:OpRepForm}.
\end{proposition}

We now give an infinite-dimensional comparable example.

\begin{example}
\label{example:comp-diff-graph}
Assume that $\{\epsilon^{X}_{i1k}\}_{k\geq 1}$, 
$\{\epsilon^{X}_{i2k}\}_{k\geq 1}$,
and $\{\epsilon^{X}_{i3k}\}_{k\geq 1}$ are all independent mean 
zero Gaussian variables with
${\rm Var}(\epsilon^{X}_{ijk})=\sigma^{2}_{X,jk}$, $j=1,2,3$, $k\geq 1$ for all $i$.  
For any $k\geq 1$, let
\begin{equation*}
  a^{X}_{i1k}=a^{X}_{i2k}+\epsilon^{X}_{i1k}, \quad
  a^{X}_{i2k}=\epsilon^{X}_{i2k}, \quad
  a^{X}_{i3k}=a^{X}_{i2k}+\epsilon^{X}_{i3k}.
\end{equation*}
Let $a^{X,M}_{ij}=(a^X_{ij1}, \cdots, a^X_{ijM})^{\top}$, $j=1,2,3$. We then define $X_{ij}(t)=\sum^{\infty}_{k=1}a^{X}_{ijk} b_k(t)$, $j=1,2,3$, where $\{b_k(t)\}^{\infty}_{k=1}$ is some orthonormal function basis of $\mathbb{H}$. We define $\{\epsilon^{Y}_{ijk}\}_{k\geq 1}$, $\{a^{Y}_{ijk}\}_{k\geq 1}$, $a^{Y,M}_{ij}$, and 
$Y_{ij}(t)$, $j=1,2,3$, similarly.

The graph structure of $X$ and $Y$ is shown in Figure~\ref{fig:ComEx}. Since
$a^{X,M}_{ij}$ follows a multivariate Gaussian distribution, for any
$M \geq 2$, $1\leq m,m^{\prime}\leq M$ and $m\neq m^{\prime}$:
\begin{equation*}
\begin{aligned}
&{\rm Var}\left(a^{X}_{i1m}\mid a^{X,M}_{i2},a^{X,M}_{i3}\right)=\sigma^{2}_{X,1m},\\
&{\rm Var}\left(a^{X}_{i3m}\mid a^{X,M}_{i1}, a^{X,M}_{i2}\right)=\sigma^{2}_{X,3m},\\
&{\rm Var}\left(a^{X}_{i2m}\mid a^{X,M}_{i1}, a^{X,M}_{i3}\right)=\frac{\sigma^{2}_{X,1m}\sigma^{2}_{X,2m}\sigma^{2}_{X,3m}}{\sigma^{2}_{X,1m}\sigma^{2}_{X,2m}+\sigma^{2}_{X,1m}\sigma^{2}_{X,3m}+\sigma^{2}_{X,2m}\sigma^{2}_{X,3m}},\\
\end{aligned}
\end{equation*}
and
\begin{equation*}
\begin{aligned}
&{\rm Var}\left(a^{X}_{i1m}\mid a^{X,M}_{i2} \right)=\sigma^{2}_{X,1m},\\
&{\rm Var}\left(a^{X}_{i1m}\mid a^{X,M}_{i3} \right)=\frac{\sigma^{2}_{X,1m}\sigma^{2}_{X,2m}+\sigma^{2}_{X,1m}\sigma^{2}_{X,3m}+\sigma^{2}_{X,2m}\sigma^{2}_{X,3m}}{\sigma^{2}_{2m}+\sigma^{2}_{3m}},\\
&{\rm Var}\left(a^{X}_{i3m}\mid a^{X,M}_{i2} \right)=\sigma^{2}_{X,3m},\\
&{\rm Var}\left(a^{X}_{i3m}\mid a^{X,M}_{i1} \right)=\frac{\sigma^{2}_{X,1m}\sigma^{2}_{X,2m}+\sigma^{2}_{X,1m}\sigma^{2}_{X,3m}+\sigma^{2}_{X,2m}\sigma^{2}_{X,3m}}{\sigma^{2}_{2m}+\sigma^{2}_{1m}},\\
&{\rm Var}\left(a^{X}_{i2m}\mid a^{X,M}_{i1} \right)=\frac{\sigma^{2}_{X,1m}\sigma^{2}_{X,2m}}{\sigma^{2}_{X,1m}+\sigma^{2}_{X,2m}},\\
&{\rm Var}\left(a^{X}_{i2m}\mid a^{X,M}_{i3} \right)=\frac{\sigma^{2}_{X,3m}\sigma^{2}_{X,2m}}{\sigma^{2}_{X,3m}+\sigma^{2}_{X,2m}}.
\end{aligned}
\end{equation*}
In addition, we also have
\begin{equation*}
\begin{aligned}
&{\rm Cov}(a^{X}_{i1m},a^{X}_{3m^{\prime}}\mid a^{X,M}_{i2})=0,\\
&{\rm Cov}(a^{X}_{i1m},a^{X}_{i2m^{\prime}}\mid a^{X,M}_{i3})=\mathbbm{1}(m=m^{\prime})\cdot\frac{\sigma^{2}_{X,3m}\sigma^{2}_{X,2m}}{\sigma^{2}_{X,3m}+\sigma^{2}_{X,2m}},\\
&{\rm Cov}(a^{X}_{i2m},a^{X}_{i3m^{\prime}}\mid a^{X,M}_{i3})=\mathbbm{1}(m=m^{\prime})\cdot\frac{\sigma^{2}_{X,1m}\sigma^{2}_{X,2m}}{\sigma^{2}_{X,1m}+\sigma^{2}_{X,2m}}.
\end{aligned}
\end{equation*}
\begin{figure}[t]
	\centering
		\begin{tikzpicture}[shorten >=1pt,
		auto,
		main node/.style={ellipse,inner sep=0pt,fill=gray!20,draw,font=\sffamily,
			minimum width = .8cm, minimum height = .8cm}]
		
		\node[main node] (1) {1};
		\node[main node] (2) [right= 1cm of 1]  {2};
		\node[main node] (3) [right = 1cm of 2]  {3};

		\path[color=black!20!blue, line width = .25mm]
		(1) edge node {} (2)
		(2) edge node {} (3)
		;
		\end{tikzpicture}
	\caption{The conditional independence graph for both $X$ and $Y$ in Example~\ref{example:comp-diff-graph}. The differential graph between $X$ and $Y$ has the same structure.}
	\label{fig:ComEx}
\end{figure}
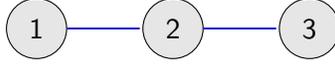
Similar results hold for $Y$. Suppose that
\[
  \sigma^{2}_{X,jk},\sigma^{2}_{Y,jk}\asymp k^{-\alpha}
  \quad\text{and}\quad
  \vert \sigma^{2}_{X,jk}-\sigma^{2}_{Y,jk} \vert \asymp k^{-\beta},\quad
  j=1,2,3,
\]
where $\alpha,\beta > 0$ and $\beta>\alpha$.
Then
\begin{equation*}
\begin{aligned}
&\bar{z}^{13,M}_{mm^{\prime}}=0,\\
&\bar{z}^{12,M}_{mm^{\prime}}=\mathbbm{1}(m=m^{\prime})\frac{\sigma^{2}_{X,1m}-\sigma^{2}_{Y,1m}}{\sigma^{2}_{X,1m}\cdot\sigma^{2}_{Y,1m}} \asymp \mathbbm{1}(m=m^{\prime})\cdot m^{-(\beta-\alpha)},\\
&\bar{z}^{23,M}_{mm^{\prime}}=\mathbbm{1}(m=m^{\prime})\frac{\sigma^{2}_{X,3m}-\sigma^{2}_{Y,3m}}{\sigma^{2}_{X,3m}\cdot\sigma^{2}_{Y,3m}}\asymp\mathbbm{1}(m=m^{\prime})\cdot m^{-(\beta-\alpha)}.
\end{aligned}
\end{equation*}
This implies that
\begin{equation}
\begin{aligned}
& \Vert \Delta^M_{13} \Vert^2_{\text{F}} = \sum^{M}_{m^{\prime}=1}\sum^{M}_{m=1} \left( \bar{z}^{13,M}_{mm^{\prime}} \right)^2=0, \\
& \Vert \Delta^M_{12} \Vert^2_{\text{F}} = \sum^{M}_{m^{\prime}=1}\sum^{M}_{m=1} \left( \bar{z}^{12,M}_{mm^{\prime}} \right)^2 \asymp \sum^M_{m=1} \frac{1}{m^{\beta-\alpha}}, \\
& \Vert \Delta^M_{23} \Vert^2_{\text{F}} =\sum^{M}_{m^{\prime}=1}\sum^{M}_{m=1} \left( \bar{z}^{23,M}_{mm^{\prime}} \right)^2 \asymp \sum^M_{m=1} \frac{1}{m^{\beta-\alpha}}.
\end{aligned}
\end{equation}
When $\beta>\alpha+1$, we have $0<\lim_{M\to\infty} \Vert \Delta^M_{12} \Vert_{\text{F}}=\lim_{M\to\infty} \Vert \Delta^M_{23} \Vert_{\text{F}} < \infty$. When $\beta \leq \alpha+1$, we have $\lim_{M\to\infty} \Vert \Delta^M_{12} \Vert_{\text{F}}=\lim_{M\to\infty} \Vert \Delta^M_{23} \Vert_{\text{F}} = \infty$. In both cases, the two graphs are comparable.
\end{example}

Comparability describes population quantities and does not immediately imply that the differential graph is easy to estimate. The following example describes a sequence of MGPs that are comparable; however, the differential graph can be arbitrarily hard to estimate. 

\begin{example}
\label{example:incomp-diff-graph}

We define $\{\epsilon^{X}_{ijk}\}_{k\geq 1}$, $\{a^{X}_{ijk}\}_{k\geq 1}$, $\{\epsilon^{Y}_{ijk}\}_{k\geq 1}$, and $\{a^{Y}_{ijk}\}_{k\geq 1}$ as in Example~\ref{example:comp-diff-graph}. Let $X_{ij}(t)=\sum^{M^{\star}}_{k=1}a^{X}_{ijk} b_k(t)$ and $Y_{ij}(t)=\sum^{M^{\star}}_{k=1}a^{Y}_{ijk} b_k(t)$, $j=1,2,3$, where $M^{\star}$ is a positive integer. Suppose that
\[
  \sigma^{2}_{X,jk},\sigma^{2}_{Y,jk}\asymp k^{-\alpha}
  \quad\text{and}\quad
  \vert \sigma^{2}_{X,jk}-\sigma^{2}_{Y,jk} \vert \asymp \mathbbm{1}(k=M^{\star}) k^{-\beta},\quad
  j=1,2,3,
\]
where $\alpha,\beta > 0$ and $\beta>\alpha$. Following the argument in Example~\ref{example:comp-diff-graph}, for any $1 \leq M \leq M^{\star}$, we have
\begin{equation*}
\begin{aligned}
&\bar{z}^{13,M}_{mm^{\prime}}=0,\\
&\bar{z}^{12,M}_{mm^{\prime}}=\mathbbm{1}(m=m^{\prime})\mathbbm{1}(m=M^{\star})\cdot\frac{\sigma^{2}_{X,1m}-\sigma^{2}_{Y,1m}}{\sigma^{2}_{X,1m}\cdot\sigma^{2}_{Y,1m}} \asymp \mathbbm{1}(m=m^{\prime})\mathbbm{1}(m=M^{\star})\cdot m^{-(\beta_1-2\alpha_1)},\\
&\bar{z}^{23,M}_{mm^{\prime}}=\mathbbm{1}(m=m^{\prime})\mathbbm{1}(m=M^{\star})\cdot\frac{\sigma^{2}_{X,3m}-\sigma^{2}_{Y,3m}}{\sigma^{2}_{X,3m}\cdot\sigma^{2}_{Y,3m}}\asymp\mathbbm{1}(m=m^{\prime})\mathbbm{1}(m=M^{\star})\cdot m^{-(\beta_3-2\alpha_3)}.
\end{aligned}
\end{equation*}
This implies that
\begin{equation}
\begin{aligned}
& \Vert \Delta^M_{13} \Vert^2_{\text{F}} = \sum^{M}_{m^{\prime}=1}\sum^{M}_{m=1} \left( \bar{z}^{13,M}_{mm^{\prime}} \right)^2=0, \\
& \Vert \Delta^M_{12} \Vert^2_{\text{F}} = \sum^{M}_{m^{\prime}=1}\sum^{M}_{m=1} \left( \bar{z}^{12,M}_{mm^{\prime}} \right)^2 \asymp M^{-2(\beta-2\alpha)} \mathbbm{1}(M=M^{\star}), \\
& \Vert \Delta^M_{23} \Vert^2_{\text{F}} = \sum^{M}_{m^{\prime}=1}\sum^{M}_{m=1} \left( \bar{z}^{23,M}_{mm^{\prime}} \right)^2 \asymp M^{-2(\beta-2\alpha)}\mathbbm{1}(M=M^{\star}).
\end{aligned}
\end{equation}

On the basis of the calculation above, we observe that estimation of the differential graph here is intrinsically hard. For any $M<M^{\star}$, we have $\Vert \Delta^M_{12} \Vert_{\text{F}}=\Vert \Delta^M_{23} \Vert_{\text{F}}=0$. Thus, when $M<M^{\star}$ is used for estimation, the resulting target graph $E^{\pi}_\Delta$ would be empty. However, by Definition~\ref{def:DGO} and Definition~\ref{def:funDGM}, we have $D_{12}=D_{23} \asymp (M^{\star})^{-2(\beta-2\alpha)}>0$ and $E_{\Delta}=\{(1,2), (2,3)\}$.

In practice, if $M^{\star}$ is very large and we do not have enough samples to accurately estimate $\Delta^M$ for a large $M$, then it is impossible for us to estimate the differential graph correctly. Furthermore, the situation is worse if $\beta>2\alpha$, since $D_{12}$ and $D_{23}$---the signal strength---vanishes as $M^{\star}$ increases. Figure~\ref{fig:sig-strg-exmp} shows how the signal strength (defined as $D_{12}$) changes as $M^{\star}$ increases for three cases: $\beta<2\alpha$, $\beta=2\alpha$, and $\beta>2\alpha$.

This problem is intrinsically hard because the difference between two graphs occurs only between components with the smallest positive eigenvalue. To capture this difference, we have to use a large number of bases $M$ to approximate the functional data, which is statistically expensive. As we increase $M$, no useful information is captured until $M=M^{\star}$. Furthermore, if the difference between the eigenvalues decreases quickly relative to the decrease of the eigenvalues, the signal strength will be very weak when the intrinsic dimension is large.

\begin{figure}[t]
	\centering
	\includegraphics[width=\linewidth]{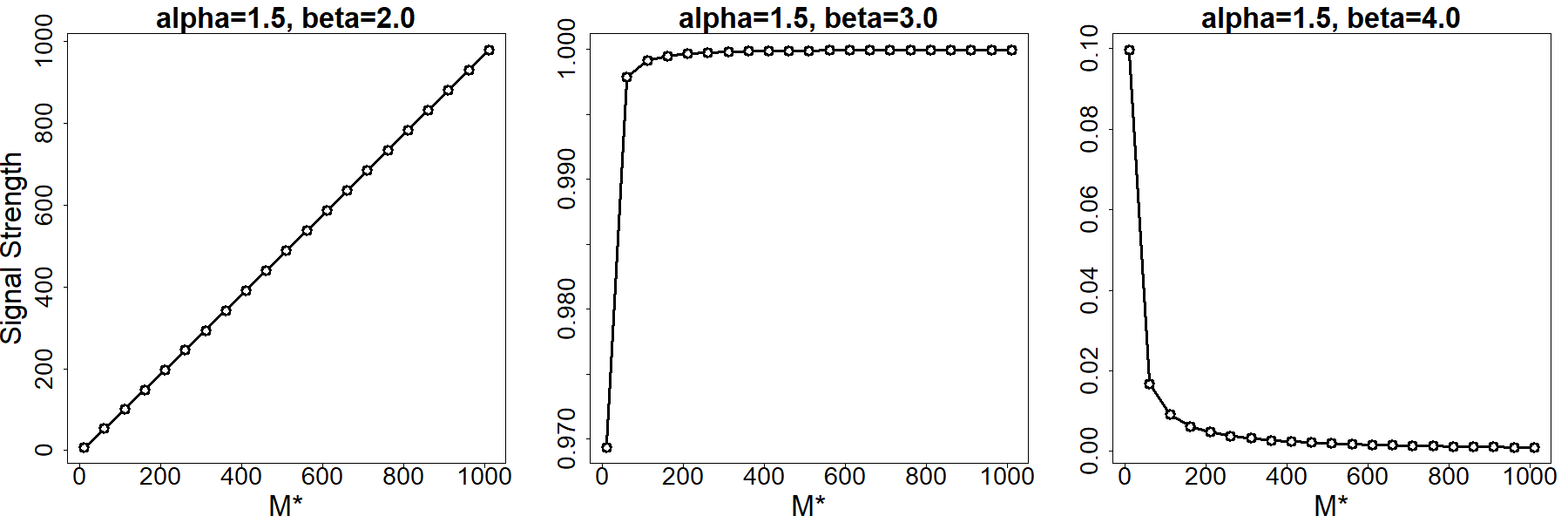}
	\caption{Signal Strength $D_{12}\asymp (M^{\star})^{-2(\beta-2\alpha)}$ in Example~\ref{example:incomp-diff-graph}. }
	\label{fig:sig-strg-exmp}
\end{figure}

\end{example}

In Example~\ref{example:comp-diff-graph}, we characterized a pair of infinite-dimensional MGPs that are comparable, and in Example~\ref{example:incomp-diff-graph} we discussed a sequence of models that are all comparable, but increasingly difficult to recover. The following example demonstrates that there are infinite-dimensional MGPs that may share the same eigenspace but are still not comparable.

\begin{example}
\label{example:incomp-graphs2}

We construct two MGPs that are both infinite-dimensional and have the same eigenspace but are incomparable. As in the previous two examples, let $V = \{1,2,3\}$. We assume that $X$ and $Y$ share a common set of eigenfunctions: $\{\phi_m\}_{m=1}^\infty$ for $j = 1,2, 3$.  

We construct the distribution of the scores of $X$ and $Y$ as follows. For any $m \in \mathbb{N}^{+}$, let $a^X_{i\,\cdot\,m}$ denote the vector of scores $(a^X_{i1m}, a^X_{i2m}, a^X_{i3m})$ and define $a^Y_{i\,\cdot\,m}$ analogously. For any natural number $z$, we first assume that
\begin{equation}
a^X_{i\,\cdot\,(3z-2)}, a^X_{i\,\cdot\,(3z-1)}, a^X_{i\,\cdot\,(3z)} \independent \{a^X_{i\,\cdot\,k}\}_{k \neq 3z, 3z-1, 3z-2}.
\end{equation}
Thus, the conditional independence graph for individual scores is a set of disconnected subgraphs corresponding to $\{a^X_{i\,\cdot\, (3z-2)}, a^X_{i\,\cdot, (3z-1)}, a^X_{i\,\cdot \, (3z)}\}$ for $z \in \mathbb{N}^{+}$. We make an analogous assumption for the scores of $Y$.
 
Within sets $\{a^X_{i \,\cdot\, (3z-2)}, a^X_{i\,\cdot\, (3z-1)}, a^X_{i\,\cdot\, (3z)}\}$ and $\{a^Y_{i\,\cdot\, (3z-2)}, a^Y_{i\,\cdot\, (3z-1)}, a^Y_{i\,\cdot\, (3z)}\}$, we assume that the conditional independence graph has the structure shown in Figure~\ref{fig:incompEx}. By construction, when projecting onto the span of the first $M$ functions, the edge set of individual functional graphical models for $X^\pi$ and $Y^\pi$ is not stable as $M \rightarrow \infty$. In particular, for both $X$ and $Y$, the edges $(1,2)$ and $(2,3)$ will persist; however, the edge $(1,3)$ will appear or disappear depending on $M$.

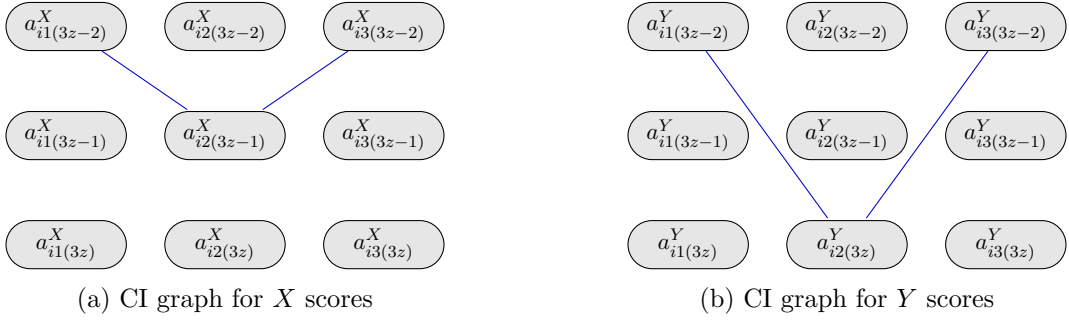
\begin{figure}[t]
\centering
\begin{minipage}[b]{.45\linewidth}
\centering 
\begin{tikzpicture}[-,shorten >=1pt,
		auto,
		main node/.style={rounded rectangle,inner sep=0pt,fill=gray!20,draw,font=\sffamily,
			minimum width = 2cm, minimum height = .8cm, scale=.8}]

		\node[main node] (1z1) {$a^X_{i1(3z-2)}$};
		\node[main node] (2z1) [right= .5cm of 1z1]  {$a^X_{i2(3z-2)}$};
		\node[main node] (3z1) [right = .5cm of 2z1]  {$a^X_{i3(3z-2)}$};
		\node[main node] (1z2) [below= .8cm of 1z1] {$a^X_{i1(3z-1)}$};
		\node[main node] (2z2) [right= .5cm of 1z2]  {$a^X_{i2(3z-1)}$};
		\node[main node] (3z2) [right = .5cm of 2z2]  {$a^X_{i3(3z-1)}$};
		\node[main node] (1z3) [below= .8cm of 1z2] {$a^X_{i1(3z)}$};
		\node[main node] (2z3) [right= .5cm of 1z3]  {$a^X_{i2(3z)}$};
		\node[main node] (3z3) [right = .5cm of 2z3]  {$a^X_{i3(3z)}$};
		
		\path[color=black!20!blue]
		(1z1) edge node {} (2z2)
		(3z1) edge  node {} (2z2);
		
	\end{tikzpicture}
\subcaption{CI graph for $X$ scores}\label{fig:1a}
\end{minipage}%
~$\qquad \quad$
\begin{minipage}[b]{.45\linewidth}
\centering
\begin{tikzpicture}[-,shorten >=1pt,
		auto,
		main node/.style={rounded rectangle,inner sep=0pt,fill=gray!20,draw,font=\sffamily,
			minimum width = 2cm, minimum height = .8cm, scale=.8}]

	\node[main node] (1z1) {$a^Y_{i1(3z-2)}$};
		\node[main node] (2z1) [right= .5cm of 1z1]  {$a^Y_{i2(3z-2)}$};
		\node[main node] (3z1) [right = .5cm of 2z1]  {$a^Y_{i3(3z-2)}$};
		\node[main node] (1z2) [below= .8cm of 1z1] {$a^Y_{i1(3z-1)}$};
		\node[main node] (2z2) [right= .5cm of 1z2]  {$a^Y_{i2(3z-1)}$};
		\node[main node] (3z2) [right = .5cm of 2z2]  {$a^Y_{i3(3z-1)}$};
		\node[main node] (1z3) [below= .8cm of 1z2] {$a^Y_{i1(3z)}$};
		\node[main node] (2z3) [right= .5cm of 1z3]  {$a^Y_{i2(3z)}$};
		\node[main node] (3z3) [right = .5cm of 2z3]  {$a^Y_{i3(3z)}$};
		
		\path[color=black!20!blue]
		(1z1) edge node {} (2z3)
		(3z1) edge  node {} (2z3);
		
	\end{tikzpicture}
\subcaption{CI graph for $Y$ scores}\label{fig:1b}
\end{minipage}%
\caption{\label{fig:incompEx}CI graph for the individual scores for two incomparable MGPs.}
\end{figure}

If $M = 3z-2$ for some $z \in \mathbb{N}^{+}$, which corresponds to the first row in Figure~\ref{fig:incompEx},  then 
\begin{equation} \label{eq:condAllBeore}
    \{a^{X}_{i1k}\}_{k < M}  \independent \{a^{X}_{i3k}\}_{k < M} \mid \{a^{X}_{i2k}\}_{k \leq M}\quad \text{ and } \quad
    \{a^{Y}_{i1k}\}_{k < M}  \independent \{a^{Y}_{i3k}\}_{k < M} \mid \{a^{Y}_{i2k}\}_{k \leq M}.
\end{equation}
However, $a^{X}_{i1M} \not \independent a^{X}_{i3M} \mid \{a^{X}_{i2k}\}_{k \leq M}$ since we do not condition on $a^X_{i2(M+1)}$. Similarly, $a^{Y}_{i1M} \not \independent a^{Y}_{i3M} \mid \{a^{Y}_{i2k}\}_{k \leq M} $ since we do not condition on $a^{Y}_{i2(M+2)}$. Thus, the edge $(1,3)$ is in the functional graphical model for both $X^\pi$ and $Y^\pi$; however, the specific values of $\Theta^{X,M}$ and $\Theta^{Y,M}$ may differ.

In contrast to the previous case, when $M = 3z-1$ for some $z \in \mathbb{N}^{+}$, which corresponds to the second row of Figure~\ref{fig:incompEx}, the functional graphical models for $X^\pi$ and $Y^\pi$ now differ. Note that $\{a^{X}_{i1k}\}_{k \leq M}  \independent \{a^{X}_{i3k}\}_{k \leq M} \mid \{a^{X}_{i2k}\}_{k \leq M}$. Thus, the edge $(1,3)$ is absent from the functional graphical model for $X^\pi$ and $\Theta^{X,M}_{1,3} = 0$. Considering $Y^\pi$, we have $\{a^{Y}_{i1k}\}_{k < M - 1}  \independent \{a^{Y}_{i3k}\}_{k < M - 1} \mid \{a^{Y}_{i2k}\}_{k \leq M}$. However, because we do not condition on $a^Y_{i2(M+1)}$ (the node in the third row of Figure~\ref{fig:incompEx}), the $(1,3)$ edge exists in the functional graphical model for $Y^\pi$ since $a^{Y}_{i1(M-1)} \not \independent a^{Y}_{i3(M-1)}  \mid \{a^{Y}_{i2k}\}_{k \leq M}$.

Suppose that the covariance of the scores is
\begin{equation*}
z^{-\beta} \times \kbordermatrix{ &a^Y_{i1(3z-2)}&a^Y_{i1(3z-1)}&a^Y_{i1(3z)} & a^Y_{i2(3z-2)}&a^Y_{i2(3z-1)}&a^Y_{i2(3z)} & a^Y_{i3(3z-2)}&a^Y_{i3(3z-1)}&a^Y_{i3(3z)}\\
a^Y_{i1(3z-2)}& 3/2 & 0 & 0  & 0 & 0 & -1 & 1/2 & 0 & 0  \\
a^Y_{i1(3z-1)}&0  &1 & 0 & 0 & 0 & 0 & 0 & 0 & 0  \\
a^Y_{i1(3z)} &0 & 0 & 1 & 0 & 0 & 0 & 0 & 0 & 0 
\\
a^Y_{i2(3z-2)}&0 & 0 & 0 & 8 & 0 & 0 & 0 & 0 & 0 \\
a^Y_{i2(3z-1)}&0 & 0 & 0 & 0 & 4 & 0 & 0 & 0 & 0 \\
a^Y_{i2(3z)} &-1 & 0 & 1 & 0 & 0 & 2 & -1 & 0 & 0
\\
a^Y_{i3(3z-2)}&1/2 & 0 & 0 & 0 & 0 & -1 & 3/2 & 0 & 0 \\
a^Y_{i3(3z-1)}&0 & 0 & 0 & 0 & 0 & 0 & 0 & 1 & 0 \\
a^Y_{i3(3z)}&0 & 0 & 0 & 0 & 0 & 0 & 0 & 0 & 1 \\
},
\end{equation*}
where $\beta > 0$ is a parameter that determines the decay rate of the eigenvalues (see Assumption~\ref{assump:EigenAssump}). We then set all other elements of the covariance as $0$. The support of the inverse of this matrix corresponds to the edges of the graph in Figure~\ref{fig:incompEx}. However, when we consider the marginal distribution of the first $M$ scores and invert the corresponding covariance, $\Theta^{Y,M}_{1,3}$ is $0$ everywhere except for the element corresponding to $a^Y_{i,1,M-1}$ and $a^Y_{i,3,M-1}$, that is, the nodes in the top row of Figure~\ref{fig:incompEx}, which is equal to $-1/4 \times ((M+1)/ 3)^\beta$. Thus, $\Vert \Delta^M_{1,3}\Vert_F =1/4 \times ((M+1)/ 3)^\beta$ and $\lim\sup_{M\rightarrow \infty} \Vert \Delta^M_{1,3}\Vert_F = \infty$.

Finally, when $M = 3z$ for some $z \in \mathbb{N}^{+}$, the $(1,3)$ edge is absent in both functional graphical models for $X^\pi$ and $Y^\pi$ because 
\[
\{a^{X}_{i1 k}\}_{k \leq M}  \independent \{a^{X}_{i3 k}\}_{k \leq M} \mid \{a^{X}_{i2 k}\}_{k \leq M}
\ \text{ and } \ 
\{a^{Y}_{i1 k}\}_{k \leq M}  \independent \{a^{Y}_{i3 k}\}_{k \leq M} \mid \{a^{Y}_{i2 k}\}_{k \leq M}.
\]
Thus, $\Theta^{X,M}_{1,3} = \Theta^{Y,M}_{1,3} = \Delta^M_{1,3} = 0$. This implies that $\lim\inf_{M \rightarrow \infty} \Vert \Delta^M_{1,3} \Vert_F = 0$.

Because $\lim\inf_{M \rightarrow \infty} \Vert \Delta^M_{1,3} \Vert_F = 0$, but $\lim\sup_{M\rightarrow \infty} \Vert \Delta^M_{1,3}\Vert_F = \infty$, $X$ and $Y$ are incomparable. 
\end{example}

The notion of comparability illustrates the intrinsic difficulty of dealing with functional data. However, it also illustrates when we can still hope to estimate the differential network consistently. We have formally stated when two infinite-dimensional functional graphical models will be comparable and have given conditions and examples of comparability. Unfortunately, these conditions cannot be checked using the observed data. For this reason, we mainly discuss the methodology and theoretical properties for estimation of $E^{\pi}_{\Delta}$. Prior knowledge about the problem at hand should be used to decide whether two infinite-dimensional functional graphs are comparable. This is similar to other assumptions common in the graphical modeling literature, such as ``faithfulness'' \citep{spirtes00causation}, that are critical to graph recovery, but can not be verified.

\section{\textbf{Fu}nctional \textbf{D}ifferential \textbf{G}raph \textbf{E}stimation: FuDGE}
\label{sec:FuDGE}

In this section, we detail our methodology for estimating a
functional differential graph. Unfortunately, in most situations, there may not be prior knowledge on which subspace to use to define the functional differential graph. In such situations, we suggest using the principle component scores of $K_{jj}(s,t)=K^X_{jj}(s,t)+K^Y_{jj}(s,t)$, $j \in V$ as a default choice. In addition, each observed function is only recorded (potentially with measurement error) at discrete time points. In Section~\ref{sec:estCovScores} we consider this practical setting. Of course, if an appropriate basis for dimension reduction is known in advance or if the functions are fully observed at all time points, then the estimated objects can always be replaced with their known/observed counterparts.

\subsection{Estimating the Covariance of the Scores}\label{sec:estCovScores}
For each $X_{ij}$, suppose we have measurements at time points $t_{ijk}$, $k=1,\ldots,T$,\footnote{For simplicity, we assume that all functions have the same number of observations, however, our method and theory can be trivially extended to allow a different number of observations for each function.} and the recorded data, $h_{ijk}$, are the function values with random noise. That is,
\begin{equation}\label{eq:datagenerate}
h_{ijk}=g_{ij}(t_{ijk})+\epsilon_{ijk},
\end{equation}
where $g_{ij}$ can denote either $X_{ij}$ or $Y_{ij}$ and the unobserved noise $\epsilon_{ijk}$ is i.i.d.~Gaussian with mean $0$ and variance $\sigma^{2}_{0}$. Without loss of generality, we assume that $t_{ij1}< \ldots< t_{ijT}$ for any $1\leq i \leq n$ and $1\leq j \leq p$. We do not assume that $t_{ijk} = t_{i'jk}$ for $i \neq i'$, so that each observation may be observed on a different grid.

We first use a basis expansion to estimate a least squares approximation of the whole curve $X_{ij}(t)$ (see Section 4.2 in \cite{Ramsay2005Functional}). Specifically, given an initial basis function vector $b(t)=(b_{1}(t),\dots,b_{L}(t))^{\top}$---for example, the B-spline or Fourier basis---our estimated approximation for $X_{ij}(t)$ is given by:
\begin{equation}
\begin{aligned}
\hat{X}_{ij}(t)&=\hat{\beta}_{ij}^{\top}b(t),\\
\hat{\beta}_{ij}&=\left(B^{\top}_{ij} B_{ij}\right)^{-1}B^{\top}_{ij}h_{ij},
\end{aligned}
\end{equation}
where $h_{ij}=(h_{ij1},h_{ij2},\dots,h_{ijT})^{\top}$ and $B_{ij}$ is the design matrix for $g_{ij}$:
\begin{equation}\label{eq:designmat}
B_{ij}=
\left[
\begin{matrix}
b_{1}(t_{ij1}) & \cdots & b_{L}(t_{ij1}) \\
\vdots & \ddots & \vdots \\
b_{1}(t_{ijT}) & \cdots & b_{L}(t_{ijT})
\end{matrix}
\right] \in \mathbb{R}^{T \times L}.
\end{equation}

The computational complexity of the basis expansion procedure 
is $O(n p T^3 L^3)$, and in practice, there are many efficient package implementations of this step; for example, \texttt{fda} \citep{fdaRpackage}.

We repeat this process for the observed $Y$ functions. After obtaining $\{\hat{X}_{ij}(t)\}_{j\in V, i=1,\ldots,n_X}$ and $\{\hat{Y}_{ij}(t)\}_{j\in V, i=1,\ldots,n_Y}$, we use them as inputs for the FPCA procedure. Specifically, we first estimate the sum of the covariance functions by
\begin{equation}\label{eq:CovEstDis}
\hat{K}_{jj}(s,t)=\hat{K}^X_{jj}(s,t)+\hat{K}^Y_{jj}(s,t) = \frac{1}{n_X}\sum^{n_X}_{i=1}\hat{X}_{ij}(s)\hat{X}_{ij}(t) + \frac{1}{n_Y}\sum^{n_Y}_{i=1}\hat{Y}_{ij}(s)\hat{Y}_{ij}(t).
\end{equation}
Using $\hat{K}_{jj}(s,t)$ as input to FPCA, we can estimate the corresponding eigenfunctions $\hat{\phi}_{jk}(t)$, $k=1,\ldots, M$, $j=1,\ldots,p$. Given the estimated eigenfunctions, we compute the estimated projection scores 
\begin{align*}
    \hat{a}^{X}_{ijk} &=\int_{\mathcal{T}}\hat X_{ij}(t)\hat{\phi}_{jk}(t)dt \qquad \text{and} \qquad
    \hat{a}^{Y}_{ijk}=\int_{\mathcal{T}}Y_{ij}(t)\hat{\phi}_{jk}(t)dt,
\end{align*}
and collect them into vectors
\begin{align*}
    \hat{a}^{X,M}_{ij}&=(\hat{a}^{X}_{ij1},\cdots, \hat{a}^{X}_{ijM})^{\top} \in{\mathbb{R}^{M}} \qquad \text{and} \qquad
    \hat{a}^{X,M}_{i}&=((\hat{a}^{X,M}_{i1})^{\top},\ldots,(\hat{a}^{X,M}_{ip})^{\top})^{\top}\in{\mathbb{R}^{pM}},\\
        \hat{a}^{Y,M}_{ij}&=(\hat{a}^{Y}_{ij1},\cdots, \hat{a}^{Y}_{ijM})^{\top} \in{\mathbb{R}^{M}} \qquad \text{and} \qquad
    \hat{a}^{Y,M}_{i}&=((\hat{a}^{Y,M}_{i1})^{\top},\ldots,(\hat{a}^{Y,M}_{ip})^{\top})^{\top}\in{\mathbb{R}^{pM}}.
\end{align*}
Finally, we estimate the covariance matrices of the score vectors, $\Sigma^{X,M}$ and $\Sigma^{Y,M}$, as 
\begin{align*}
    S^{X,M}=\frac{1}{n_X}\sum^{n_X}_{i=1}\hat a^{X,M}_{i} (\hat a^{X,M}_{i})^{\top}
    \qquad\text{and}\qquad
    S^{Y,M}=\frac{1}{n_Y}\sum^{n_Y}_{i=1} \hat a^{Y,M}_{i} (\hat a^{Y,M}_{i})^{\top}.
\end{align*}

\subsection{FuDGE: \textbf{Fu}nctional
\textbf{D}ifferential \textbf{G}raph \textbf{E}stimation}

Now we describe the FuDGE algorithm for \textbf{Fu}nctional \textbf{D}ifferential \textbf{G}raph \textbf{E}stimation. To estimate $\Delta^M$, we solve the following optimization program:
\begin{equation}\label{eq:objectivefunc}
\hat\Delta^{M} \in \argmin_{\Delta\in{\mathbb{R}^{pM\times pM}}} \mathcal{L}(\Delta) + \lambda_{n}\sum_{\{i,j\}\in V^2}\|\Delta_{ij}\|_{F},
\end{equation}
where
\begin{equation}
\label{eq:def-loss-fun-1}
\mathcal{L}(\Delta)=\mathrm{tr}\left[\frac{1}{2}
S^{Y,M}\Delta^{\top}{S^{X,M}}\Delta-\Delta^{\top}\left(S^{Y,M}-S^{X,M}\right)\right]
\end{equation}
and $S^{X,M}$ and $S^{Y,M}$ are obtained as described in Section~\ref{sec:estCovScores}.

We construct the loss function, $\mathcal{L}(\Delta)$, so that the true parameter value, that is, $\Delta^{M}=\left(\Sigma^{X,M}\right)^{-1}-\left(\Sigma^{Y,M}\right)^{-1}$, minimizes the population loss $\mathbb{E}\left[\mathcal{L}(\Delta)\right]$, which for a differentiable and convex loss function is equivalent to selecting $\mathcal{L}$ so that $\mathbb{E}\left[\nabla \mathcal{L}(\Delta^{M})\right]=0$. Since $\Delta^{M}$ satisfies
\[
  \Sigma^{X,M}\Delta^{M}\Sigma^{Y,M}-(\Sigma^{Y,M}-\Sigma^{X,M})=0,
\]
a choice for $\nabla \mathcal{L}(\Delta)$ is
\begin{equation}\label{con:GradLoss}
\nabla{\mathcal{L}(\Delta^M)}=S^{X,M}\Delta^M{S^{Y,M}}-\left(S^{Y,M}-S^{X,M}\right)
\end{equation}
so that
\[
  \mathbb{E}\left[\nabla
    \mathcal{L}(\Delta^{M})\right]=\Sigma^{X,M}\Delta^{M}\Sigma^{Y,M}-(\Sigma^{Y,M}-\Sigma^{X,M})=0.
\]
Given this choice of $\nabla \mathcal{L}(\Delta)$, $\mathcal{L}(\Delta)$ in \eqref{eq:objectivefunc} directly follows from properties of the differential of the trace function. The chosen loss is quadratic (see \eqref{con:lossdef} in appendix) and leads to an efficient algorithm. Similar loss functions are used in \cite{xu2016semiparametric}, \citet{yuan2017differential}, \citet{Na2019Estimating}, and \cite{zhao2014direct}.  

We also include the additional group lasso penalty~\citep{yuan2006model} to promote blockwise sparsity in $\hat \Delta^M$. The objective in~\eqref{eq:objectivefunc} can be solved by a proximal gradient method detailed in Algorithm~\ref{alg:prox}. Finally, we form $\hat E_{\Delta}$ by thresholding $\hat \Delta^M$ so that:
\begin{equation}{\label{con:changedgeest}}
\hat{E}_{\Delta} = \left\{ \{j,l\}  \,:\, \|\hat{\Delta}^{M}_{jl}\|_{F}>\epsilon_{n}\;\;\text{or}\;\;\|\hat{\Delta}^{M}_{lj}\|_{F}>\epsilon_{n}\right\}.
\end{equation}
The thresholding step in \eqref{con:changedgeest} is used for theoretical purposes. Specifically, it helps correct for bias induced by finite-dimensional truncation and relaxes commonly used assumptions for the recovery of the graph structure, such as the irrepresentability or incoherence condition \citep{geer09conditions}. In practice, one can simply set $\epsilon_n=0$, as we do in the simulations.

\subsection{Optimization Algorithm for FuDGE}

\begin{algorithm}[t]
	\caption{\label{alg:prox}Functional differential graph estimation} 
	\label{alg:A}
	\begin{algorithmic}
		\REQUIRE $S^{X,M},S^{Y,M},\lambda_{n}, \eta$.
		\ENSURE $\hat{\Delta}^{M}$.
		\STATE Initialize $\Delta^{(0)}=0_{pM}$.
		\REPEAT 
		\STATE $A=\Delta-\eta\nabla \mathcal{L}(\Delta)=\Delta-\eta\left(S^{X,M}\Delta S^{Y,M}-\left(S^{Y,M}-S^{X,M}\right)\right)$
		\FOR{$1\leq{i,j}\leq{p}$} 
		\STATE $\Delta_{jl}\leftarrow\left(\frac{\|A_{jl}\|_{F}-\lambda_{n}\eta}{\|A_{jl}\|_{F}}\right)_{+} \cdot A_{jl}$
		\ENDFOR
		\UNTIL{Converge} 
	\end{algorithmic}
\end{algorithm}

The optimization program~\eqref{eq:objectivefunc} can be solved using a proximal gradient method~\citep{parikh2014proximal} summarized in Algorithm~\ref{alg:prox}. Specifically, at each iteration, we update the current value of $\Delta$, denoted as $\Delta^{\text{old}}$, by solving the following problem:
\begin{equation}{\label{con:iterorign}}
\Delta^{\text{new}}=\argmin_{\Delta}\left(\frac{1}{2}\left\Vert\Delta-\left(\Delta^{\text{old}}-\eta\nabla \mathcal{L}\left(\Delta^{\text{old}}\right)\right)\right\Vert_{F}^{2}+\eta\cdot\lambda_{n}\sum^{p}_{j,l=1}\|\Delta_{jl}\|_{F}\right),
\end{equation}
where $\nabla \mathcal{L}(\Delta)$ is defined in \eqref{con:GradLoss} and $\eta$ is a step size specified by the user. Note that $\nabla \mathcal{L}(\Delta)$ is Lipschitz continuous with Lipschitz constant $\lambda^{S}_{\max} = \|S^{Y,M}\otimes S^{X,M}\|_{2}= \lambda_{\max}(S^{Y,M})\lambda_{\max}(S^{X,M})$. Thus, for any step size $\eta$ such that $0<\eta\leq1/\lambda^{S}_{\max}$, the proximal gradient method is guaranteed to converge \citep{Beck2009Fast}.

The update in~\eqref{con:iterorign} has a closed-form solution:
\begin{equation}{\label{con:itersimplified}}
\Delta^{\text{new}}_{jl}=\left[\left(\Vert A^{\text{old}}_{jl}\Vert_{F}-\lambda_{n}\eta\right)/\Vert A^{\text{old}}_{jl}\Vert_{F}\right]_{+}\cdot A^{\text{old}}_{jl}, \qquad 1\leq{j,l}\leq{p},
\end{equation}
where $A^{\text{old}}=\Delta^{\text{old}}-\eta\nabla \mathcal{L}(\Delta^{\text{old}})$ and $x_{+}=\max\{0,x\},x\in{\mathbb{R}}$, represents the positive part of $x$. Detailed derivations are given in the appendix. Note that although true $\Delta^M$ is symmetric, we do not explicitly enforce symmetry in $\hat \Delta^M$ in Algorithm~\ref{alg:prox}.   

After performing FPCA, the proximal gradient descent method converges in $O\left(\lambda^{S}_{\max}/\text{tol}\right)$ iterations, where $\text{tol}$ is a user specified optimization error tolerance, and each iteration takes $O((pM)^3)$ operations; see \citet{tibshirani2010proximal} for a convergence analysis of the general proximal gradient descent algorithm.

\subsection{Selection of Tuning Parameters}

There are four tuning parameters that must be chosen for the implementation of FuDGE: $L$ (basis dimension used to estimate the curves from the discretely observed data), $M$ (subspace dimension to estimate projection scores), $\lambda_n$ (regularization parameter to tune the block sparsity of $\Delta^{M}$), and $\epsilon_n$ (thresholding parameter for $\hat{E}_{\Delta}$). While we need the thresholding parameter $\epsilon_n$ in~\eqref{con:changedgeest} to establish theoretical results, in practice, we simply set $\epsilon_n = 0$. To select $M$, we follow the procedure in \cite{Qiao2015Functional}. More specifically, for each discretely-observed curve, we first estimate the underlying functions by fitting an $L$-dimensional B-spline basis. Both $M$ and $L$ are chosen by 5-fold cross-validation as discussed in~\cite{Qiao2015Functional}.

Finally, to choose $\lambda_n$, we recommend using selective cross-validation (SCV) \citep{she2012iterative}. Given a value of $\lambda_n$, we use the entire data set to estimate a sparsity pattern. Then, fixing the sparsity pattern, we use a typical cross-validation procedure to calculate the CV error, where the CV error is measured by an unpenalized version of the loss function in~\eqref{eq:def-loss-fun-1} on the held-out data set. Ultimately, we choose the value of $\lambda_n$ that results in the sparsity pattern that minimizes CV error. In addition to SCV, if we have prior knowledge about the number of edges in the differential graph, we can also choose $\lambda_n$ which results in the desired level of sparsity of the differential graph. 

\section{Theoretical Properties}\label{sec:ThmProp}

In this section, we provide theoretical guarantees for FuDGE. We first give a deterministic result for $\hat{E}_{\Delta}$ defined in \eqref{con:changedgeest} when the max-norm of the difference between the estimates $S^{X,M},S^{Y,M}$  and their corresponding parameters, $\Sigma^{X,M},\Sigma^{Y,M}$, is bounded by $\delta_n$. We then show that when projecting the data onto either a fixed basis or an estimated basis---under some mild conditions---$\delta_n$ can be controlled and the bias of the finite-dimensional projection decreases fast enough that $E_\Delta$ can be consistently recovered.

\subsection{Deterministic Guarantees for $\hat{E}_{\Delta}$}
\label{sec:thm-E-Delta}

In this section, we assume that $S^{X,M},S^{Y,M}$ are good estimates of $\Sigma^{X,M},\Sigma^{Y,M}$ and give a deterministic result in Theorem~\ref{Thm:smallboundThm}. Let $n=\min\{n_X,n_Y\}$. We assume that the following holds.
\begin{assump}
\label{assump:req-SXSY}
The matrices $S^{X,M},S^{Y,M}$ are estimates of $\Sigma^{X,M},\Sigma^{Y,M}$ that satisfy
\begin{equation}\label{eq:req-SXSY}
\max \left\{ \vert S^{X,M}-\Sigma^{X,M} \vert_{\infty}, \vert S^{Y,M}-\Sigma^{Y,M} \vert_{\infty} \right\} \leq \delta_n.
\end{equation}
\end{assump}

We also require that $E_{\Delta}$ be sparse. This does not preclude the case where $E_X$ and $E_Y$ are dense, as long as there are not too many differences in the precision matrices. This assumption is also required when estimating a differential graph from vector-valued data; for example, see Condition 1 in \cite{zhao2014direct}.   

\begin{assump}{\label{assump:SparseAssump}}
There are $s$ edges in the differential graph; that is, $\vert E_{\Delta} \vert = s$ and $s \ll p$.
\end{assump}

We introduce the following three quantities that characterize the problem instance and will be used in Theorem~\ref{Thm:smallboundThm} below:
\begin{equation*}
\nu_1=\nu_1(M)=\min_{(j,l)\in E_{\Delta}}  \Vert\Delta^{M}_{jl}\Vert_{F} ,
\quad
\nu_2=\nu_2(M)=\max_{(j,l) \in E^C_{\Delta}}  \Vert\Delta^{M}_{jl}\Vert_{F},
\end{equation*}
and 
\begin{equation}\label{eq:signal-strength-tau}
\tau = \tau(M) = \nu_1(M) - \nu_2(M).
\end{equation}
Roughly speaking, $\nu_1(M)$ indicates the ``signal strength'' present when we use the $M$-dimensional representation and $\nu_2(M)$ measures the bias. By Definition~\ref{def:DGO}, when $X$ and $Y$ are comparable, we have $\lim\inf_{M \to M^\star}\nu_1(M)>0$ and $\lim_{M \to M^\star}\nu_2(M)=0$. Therefore, for a sufficiently large $M$, we have $\tau > 0$. However, a smaller $\tau$ implies that the differential graph is harder to recover. 

Before we give the deterministic result in Theorem~\ref{Thm:smallboundThm}, we first define additional quantities that will be used in subsequent results. Let
\begin{equation}
\begin{aligned}
\sigma_{\max} &= \max\{\vert \Sigma^{X,M}\vert_\infty, \vert
\Sigma^{Y,M}\vert_\infty\},\\
\lambda^{*}_{\min} &=\lambda_{\min}\left(\Sigma^{X,M}\right)\times
\lambda_{\min}\left(\Sigma^{Y,M}\right), \text{ and }\\
  \Gamma^2_n &= \frac{9\lambda^{2}_{n}s}{\kappa^{2}_{\mathcal{L}}}+\frac{2\lambda_{n}}{\kappa_{\mathcal{L}}}(\omega^{2}_{\mathcal{L}}+2p^2\nu_2),
\end{aligned}
\end{equation}
where
\begin{equation}
  \begin{aligned}
    \lambda_{n}\;&=\;2M\left[\left(\delta_{n}^2 + 2\delta_{n}\sigma_{\max}\right)\left \vert \Delta^M \right\vert_{1} + 2\delta_{n}\right], \\
    \kappa_{\mathcal{L}} \; &= \; (1/2)\lambda^{*}_{\min}-8M^{2}s\left(\delta_{n}^2 + 2\delta_{n}\sigma_{\max}\right), \\
    \omega_{\mathcal{L}} \;&=\; 4Mp^2\nu_2\sqrt{\delta_{n}^2 + 2\delta_{n}\sigma_{\max}},
  \end{aligned}    
\end{equation}
and $\delta_n$ is defined in Assumption~\ref{assump:req-SXSY}. Note that $\Gamma_n$---which measures the estimation error of $\Vert \hat{\Delta}^M-\Delta^M \Vert_{\text{F}}$---implicitly depends on $\delta_n$ through $\lambda_n$, $\kappa_\mathcal{L}$, and $\omega_\mathcal{L}$. We observe that $\Gamma_n$ decreases to zero as $\delta_n$ goes to zero. The quantity $\kappa_{\mathcal{L}}$ is the maximum restricted eigenvalue from the analysis framework of \cite{negahban2010unified}. Finally, $\omega_{\mathcal{L}}$ is the tolerance parameter that comes from the fact that $\nu_2$ could be larger than zero, and will decrease to zero as $\nu_2$ goes to zero.

\begin{theorem}\label{Thm:smallboundThm}
Given Assumptions~\ref{assump:req-SXSY} and~\ref{assump:SparseAssump}, when $\nu_1(M), \nu_2(M), \delta_n, \lambda_n, \sigma_{\max}, M$ and $s$ satisfy 
\begin{equation}{\label{con:sizecond}}
\begin{aligned}
0 <  \Gamma_n < \tau/2
\qquad\text{and}\qquad
\delta_n  < (1/4)\sqrt{ \left(\lambda^{*}_{\min} + 16M^2s (\sigma_{\max})^2\right)/\left( M^2 s\right)} - \sigma_{\max},
\end{aligned}
\end{equation}
then setting $\epsilon_{n} \in \left[\nu_2+\Gamma_n , \nu_1-\Gamma_n \right)$ ensures that $\hat{E}_{\Delta}=E_{\Delta}$.
\end{theorem}

As shown in Section~\ref{sec:th-SX-SY-delta_n}, under some additional conditions, Assumption~\ref{assump:req-SXSY} holds for a sequence of $\delta_n$ that decreases to $0$ as $n$ goes to infinity. Thus, as $M$ and $n$ both increase to infinity, we have $\nu_2+\Gamma_n \approx 0$ and $\nu_1-\Gamma_n \approx \min_{(j,l)\in E_{\Delta}}D_{jl}$, and we only require $0 \leq \epsilon_n < \min_{(j,l)\in E_{\Delta}}D_{jl}$.

\subsection{Theoretical Guarantees for $S^{X,M}$ and $S^{Y,M}$}
\label{sec:th-SX-SY-delta_n}

In this section, we prove that under some mild conditions, \eqref{eq:req-SXSY} will hold with a high probability for specific values of $\delta_n$. We discuss the results in two cases: the case where the curves are fully observed and the case where the curves are only observed at discrete time points.

\subsubsection{Fully Observed Curves}

In this section, we discuss the case where each curve is fully observed. We first consider the case where the basis defining the differential graph is known in advance; that is, the exact form of $\{e_{jk}\}_{k\geq 1}$ for all $j \in V$ is known. In this case, the projection score vectors $a^{X,M}_i$ and $a^{Y,M}_i$ can be recovered exactly for all $i=1,2,\dots,n$. By the assumption that $X_{i}(t)$ and $Y_{i}(t)$ are $p$-dimensional multivariate Gaussian processes with mean zero, we then have $a^{X,M}_{i} \sim N(0,\Sigma^{X,M})$ and $a^{Y,M}_{i} \sim N(0,\Sigma^{Y,M})$. The following result follows directly from the standard results on the sample covariance of multivariate Gaussian variables.
\begin{theorem}\label{thm:proj-score-mat-case1}
Assume that $S^{X,M}$ and $S^{Y,M}$ are computed as in Section~\ref{sec:estCovScores}, except that the basis functions $\{e_{jk}\}_{k\geq 1}$, $j \in V$, are fixed and known in advance. Recall that 
\[
n = \min\{n_{X}, n_{Y}\}
\quad\text{and}\quad 
\sigma_{\max} = \max\{\vert \Sigma^{X,M}\vert_\infty, \vert
\Sigma^{Y,M}\vert_\infty\}.
\]
Fix $\iota \in (0, 1]$. Suppose that $n$ is large enough so that 
\[
\delta_n = \sigma_{\max} \sqrt{ \frac{C_1}{n} \log \left( \frac{8 p^2 M^2 }{\iota} \right) } \leq C_2, 
\]
for some universal constants $C_1,C_2>0$. Then \eqref{eq:req-SXSY} holds with probability at least $1-\iota$.
\end{theorem}
\begin{proof}
The proof follows directly from Lemma 1 of \cite{Ravikumar2011High} and a union bound.
\end{proof}

With fully observed curves and known basis functions, it follows from Theorem~\ref{thm:proj-score-mat-case1} that \eqref{assump:req-SXSY} is satisfied for $\delta_n \asymp \sqrt{\log(p^2 M^2)/n}$ with high probability. As assumed in Section~\ref{sec:finiteDim} (and also in \citet{Qiao2015Functional}), when $\lambda^X_{jm'} =\lambda^Y_{jm'} = 0$ for all $j$ and $m' > M$ (where $M$ is allowed to grow with $n$), then $\nu_2(M)=0$, $\tau(M) = \nu_1(M)=\min_{(j,l)\in E_{\Delta}}D_{jl} > 0$, and $E_{\Delta} = E^{\pi}_{\Delta}$. We can recover $E_\Delta$ with high probability even in the high-dimensional setting, as long as
\[
\max \left\{  \frac{ sM^2\log(p^2M^2)\vert \Delta^M \vert_1^2 / ((\lambda^{\star}_{\min})^2 \tau^2) }{n} , \frac{sM^2\log(p^2M^2)/\lambda^{\star}_{\min}}{n} \right\} \rightarrow 0.
\] 
Even with an infinite number of positive eigenvalues, high-dimensional consistency is still possible for quickly increasing $\nu_1$ and quickly decaying $\nu_2$.

We then consider the case where the curves are fully observed, but we do not have any prior knowledge on which orthonormal function basis should be used. In this case, as discussed in Section~\ref{sec:choose-sub-fpca}, we recommend using the eigenfunctions of $K_{jj}(\cdot, *)=K^X_{jj}(\cdot, *)+K^Y_{jj}(\cdot, *)$ as basis functions. We use FPCA to estimate the eigenfuctions of $K_{jj}(\cdot, *)$ and make the following assumption.

\begin{assump}\label{assump:EigenAssump}
Let $\{\lambda_{jk}, \phi_{jk}(\cdot)\}$ be the eigenpairs of $K_{jj}(\cdot, *)=K^X_{jj}(\cdot, *)+K^Y_{jj}(\cdot, *)$, $j\in V$, and suppose that $\lambda_{jk}$ are non-increasing in $k$.
\begin{enumerate}[label=(\roman*)]
\item Suppose $\max_{j\in{V}}\sum_{k=1}^{\infty}\lambda_{jk}<\infty$ 
and assume that there exists a constant $\beta>1$ such that, 
for each $k\in \mathbb{N}$, 
$\lambda_{jk}\asymp{k^{-\beta}}$ and 
$d_{jk}\lambda_{jk}=O(k)$ uniformly in $j\in{V}$, 
where 
$d_{jk}=2\sqrt{2}\max\{(\lambda_{j(k-1)}-\lambda_{jk})^{-1},(\lambda_{jk}-\lambda_{j(k+1)})^{-1}\}$, $k\geq 2$,
and $d_{j1}=2\sqrt{2}(\lambda_{j1}-\lambda_{j2})^{-1}$.
\item For all $k$, 
$\phi_{jk}(\cdot)$'s are continuous on the compact set $\mathcal{T}$ and 
satisfy \begin{equation*}
    \max_{j\in{V}}\sup_{s\in{\mathcal{T}}}\sup_{k\geq{1}}|\phi_{jk}(s)|_\infty=O(1).
\end{equation*}
\end{enumerate}
\end{assump}
This assumption was used in \citet[Condition 1]{Qiao2015Functional}. We have the following result.
\begin{theorem}\label{thm:proj-score-mat-case2}
Suppose Assumption~\ref{assump:EigenAssump} holds and the basis functions are estimated using the FPCA of $K_{jj}(\cdot, *)$ with fully observed curves. Fix $\iota \in (0, 1]$. Suppose that $n$ is large enough so that 
\begin{equation*}
\delta_n = M^{1+\beta}\sqrt{ \frac{ \log \left( 2 C_2 p^2 M^2 / \iota  \right) }{ n} } \leq C_1,
\end{equation*}
for some universal constants $C_1,C_2>0$. Then \eqref{eq:req-SXSY} holds with probability at least $1-\iota$. Note that $C_1,C_2$ may be different constants from those of Theorem~\ref{thm:proj-score-mat-case1}.
\end{theorem}
\begin{proof}
The proof follows directly from Theorem 1 of \cite{Qiao2015Functional} and the fact that $\Vert \hat{K}_{jj}(\cdot, *) - K_{jj}(\cdot, *) \Vert_{\text{HS}} \leq \Vert \hat{K}^X_{jj}(\cdot, *) - K^X_{jj}(\cdot, *) \Vert_{\text{HS}} + \Vert \hat{K}^Y_{jj}(\cdot, *) - K^Y_{jj}(\cdot, *) \Vert_{\text{HS}}$.
\end{proof}

It follows from Theorem~\ref{thm:proj-score-mat-case2} that \eqref{assump:req-SXSY} holds for $\delta_n \asymp M^{1+\beta}\sqrt{\log(p^2 M^2/)/n}$ with high probability. Compared to Theorem~\ref{thm:proj-score-mat-case1}, there is an additional $M^{1+\beta}$ term that arises from FPCA estimation error. Similarly, when $\lambda^X_{jm'} =\lambda^Y_{jm'} = 0$ for all $j$ and $m' > M $, we can recover $E_\Delta$ with high probability as long as
\[\max \left\{  \frac{ sM^{(4+2\beta)}\log(p^2M^2)\vert \Delta^M \vert_1^2 / ((\lambda^{\star}_{\min})^2 \tau^2) }{n} , \frac{sM^{(4+2\beta)}\log(p^2M^2)/\lambda^{\star}_{\min}}{n} \right\} \rightarrow 0.
\]

\subsubsection{Discretely-Observed Curves}

Finally, we discuss the case where the curves are only observed at discrete time points---possibly with measurement error. Following Chapter 1 of~\cite{kokoszka2017introduction}, we first estimate each curve from the available observations by basis expansion; then, we use the estimated curves to form empirical covariance functions from which we estimate the eigenfunctions using FPCA. The estimated eigenfunctions are then used to calculate the scores.

Recall the model for discretely observed functions given in~\eqref{eq:datagenerate}:
\[
h_{ijk}=g_{ij}(t_{ijk})+\epsilon_{ijk},
\]
where $g_{ij}$ denotes either $X_{ij}$ or $Y_{ij}$, $\epsilon_{ijk}$ are i.i.d.~Gaussian noise with mean $0$ and variance $\sigma^{2}_{0}$. Assume that $t_{ij1}< \dots< t_{ijT}$ for any $1\leq i \leq n$ and $1\leq j \leq p$. Note that we do not need $X$ and $Y$ to be observed at the same time points and we use $t_{ijk}$ to represent either $t^X_{ijk}$ or $t^Y_{ijk}$. Furthermore, recall that we first compute a least squares estimator of $X_{ij}(\cdot)$ and $Y_{ij}(\cdot)$ by projecting it onto the basis $b(\cdot) = \left(b_1(\cdot), \ldots, b_L(\cdot)\right)$.

First, we assume that as we increase the number of basis functions, we can approximate any function in $\mathbb{H}$ arbitrarily well.
\begin{assump}\label{assump:CONS}
 We assume that $\{b_l\}^{\infty}_{l=1}$ is a complete orthonormal system (CONS) \citep[See Definition 2.4.11 of ][]{Hsing2015Theoretical} of $\mathbb{H}$, that is, $\overline{{\rm Span}\left(\{b_l\}^{\infty}_{l=1}\right)}=\mathbb{H}$. 
\end{assump}
Assumption~\ref{assump:CONS} requires that the basis functions are orthonormal. When this assumption is violated---for example, when using the B-splines basis---we can always first use an orthonormalization process, such as Gram-Schmidt, to convert the basis to an orthonormal one. For B-splines, there are many algorithms that can efficiently provide orthonormalization \citep{liu2019splinets}.

To establish theoretical guarantees for the least squares estimator, we require smoothness in both the curves we are trying to estimate as well as the basis functions we use.
\begin{assump}\label{assump:smoothness}
We assume that the basis functions $\{b_{l}(\cdot)\}^{\infty}_{l=1}$ satisfy the following conditions:
\begin{equation}
D_{0,b} \coloneqq \sup_{l \geq 1} \sup_{t \in \mathcal{T}} \lvert b_{l}(t)  \rvert < \infty, \qquad D_{1,b}(l) \coloneqq \sup_{t \in \mathcal{T}} \lvert b^{\prime}_{l}(t)  \rvert < \infty, \qquad D_{1,b,L} \coloneqq \max_{1 \leq l \leq L} D_{1,b}(l).
\end{equation}
We also require that the curves $g_{ij}$ satisfy the following smoothness condition:
\begin{equation}\label{eq:asump4-cond1}
\max_{1\leq j \leq p} \sum^{\infty}_{m=1} \mathbb{E} \left[ \left( \langle g_{ij},b_m \rangle \right)^2  \right] D^2_{1,b}(m) < \infty.
\end{equation}
\end{assump}

To better understand Assumption~\ref{assump:smoothness}, we use the Fourier basis as an example. Let $\mathcal{T}=[0,1]$ and $b_m(t)=\sqrt{2}\cos(2\pi m t)$, $0\leq t \leq 1$ and $m \in \mathbb{N}$. Thus, $\{b_m(t)\}^{\infty}_{m=0}$ then constitutes an orthonormal basis of $\mathbb{H}=\mathcal{L}^2[0,1]$. We then have $b^{\prime}(t)=-2\sqrt{2} \pi m \sin(2\pi m t)$, $D_{0,b}=\sqrt{2}$, $D_{1,b}(m)=2\sqrt{2}\pi m$ and $D_{1,b,L}=2\sqrt{2}\pi L$. In this case, \eqref{eq:asump4-cond1} is equivalent to
\begin{equation*}
\max_{1\leq j \leq p}  \sum^{\infty}_{m=1} \mathbb{E} \left[ \left( \langle g_{ij},b_m \rangle \right)^2  \right] m^2 < \infty.
\end{equation*}
On the other hand, $g_{ij}(t)=\sum^{\infty}_{m=1} \langle g_{ij},b_m \rangle b_m(t)$ and $g^{\prime}_{ij}(t)=\sum^{\infty}_{m=1} \langle g_{ij},b_m \rangle b^{\prime}_m(t)$. Suppose that $\mathbb{E} \left[ \Vert g^{\prime}_{ij} \Vert^2 \right] < \infty$. Then
\begin{equation}\label{eq:smoothnessAssumptionEx}
\mathbb{E} \left[ \Vert g^{\prime}_{ij} \Vert^2 \right] = \sum^{\infty}_{m=1} \mathbb{E} \left[ \left( \langle g_{ij},b_m \rangle \right)^2  \right] \Vert b^{\prime}_m \Vert^2 \asymp \sum^{\infty}_{m=1} \mathbb{E} \left[ \left( \langle g_{ij},b_m \rangle \right)^2  \right] m^2.
\end{equation}
Therefore, $\max_{1\leq j \leq p}  \mathbb{E} \left[ \Vert g^{\prime}_{ij} \Vert^2 \right] < \infty$, which is a commonly used assumption in nonparameteric statistics (e.g., Section 7.2 of~\cite{wasserman2006all}), implies \eqref{eq:asump4-cond1}.

Finally, we require that each function be observed at time points that are ``evenly spaced.'' Formally, we require the following assumption.
\begin{assump}\label{assump:observdensity}
  The observation time points
  $\{t_{ijk}:1\leq i \leq n,1 \leq j \leq p,1 \leq k \leq T\}$ satisfy
  \begin{equation}
    \max_{1\leq i \leq n}\max_{1 \leq j \leq p}\max_{1 \leq k \leq T+1}\left\vert \frac{t_{ijk}-t_{ij(k-1)}}{\lvert \mathcal{T} \rvert} - \frac{1}{T}\right\vert \leq \frac{\zeta_0}{T^2},
  \end{equation}
  where $t_{ij0}$ and $t_{ij(T+1)}$ are endpoints of $\mathcal{T}$ for any $1\leq i \leq n$, $1 \leq j \leq p$, and $\zeta_0$ is a positive constant that does not depend on $i$ or $j$.
\end{assump}

Any $g_{ij}$ can be decomposed into $g_{ij}=g^{\shortparallel}_{ij}+g^{\bot}_{ij}$, where $g_{ij}^{\shortparallel} \in {\rm Span}(b)$ and $g_{ij}^{\bot} \in {\rm Span}(b)^{\bot}$. We denote the eigenvalues of the covariance operator of $g_{ij}$ as $\{\lambda_{jk}\}_{k \geq 1}$ and $\lambda_{j0}=\sum^{\infty}_{k=1}\lambda_{jk}$; and we denote the eigenvalues of the covariance operator of $g^{\bot}_{ij}$ as $\{\lambda^{\bot}_{jk}\}_{k \geq 1}$ and $\lambda^{\bot}_{j0}=\sum^{\infty}_{k=1}\lambda^{\bot}_{jk}$. Note that under Assumption \ref{assump:EigenAssump}, we have $\max_{1\leq j \leq p}\lambda_{j0}<\infty$. Let $1<\lambda_{0,\max}<\infty$ be a constant such that $\max_{1\leq j \leq p}\lambda_{j0}\leq \lambda_{0,\max}$.  Let $B_{ij}$ be the design matrix of $g_{ij}$ as defined in \eqref{eq:designmat} and let $\lambda^{B}_{\min}=\min_{1\leq i \leq n, 1 \leq j \leq p}\left\{ \lambda_{\min}(B^{\top}_{ij}B_{ij}) \right\}$. We define
\begin{gather}
\tilde{\psi}_{1}(T,L)=\frac{\sigma_0 L }{\sqrt{\lambda^{B}_{\min}}}, \quad 
\tilde{\psi}_2(T,L) =  \frac{ L^2 }{ (\lambda^B_{\min})^2 } \left( \lambda_0  \left( \tilde{c}_1 D^2_{1,b,L}  + \tilde{c_2} \right) \tilde{\psi}_3 (L) +  \tilde{c}_1 \tilde{\psi}_4(L)  \right), \\
\tilde{\psi}_{3}(L)\;=\;\max_{1\leq j \leq p}\left( \lambda^{\bot}_{j0}/\lambda_{j0} \right), 
\quad
\tilde{\psi}_4 (L) = \max_{1\leq j \leq p}  \sum_{m>L} \mathbb{E} \left[ \left( \langle g_{ij},b_m \rangle \right)^2  \right] D^2_{1,b}(m), \\
\Phi(T,L)=\min\left\{ 1/\tilde{\psi}_{1}(T,L), 1/\sqrt{\tilde{\psi}_{3}(L)} \right\},
\end{gather}
where $\tilde{c}_1=18 D^2_{0,b} (\zeta_{0}+1)^4 \vert \mathcal{T} \vert^2$ and $\tilde{c_2}=36 D^4_{0,b} (2\zeta_{0}+1)^2$.

We now use superscripts or subscripts to indicate the specific quantities for $X$ and $Y$. In this way, we define $L_X$, $L_Y$, $T_X$, $T_Y$, $\tilde{\psi}^{X}_1$-$\tilde{\psi}^{X}_4$, $\tilde{\psi}^{Y}_1$-$\tilde{\psi}^{Y}_4$, and $\Phi^{X},\Phi^{Y}$.
Furthermore, let $T=\min\{T_X,T_Y\}$, $L=\min\{L_X,L_Y\}$, $\bar{\psi}_k=\max \{ \tilde{\psi}^X_k,\tilde{\psi}^Y_k \}$, $k=1,\cdots,4$, $\bar{\Phi}=\min \{ \Phi^X, \Phi^Y \}$,
and let $n$, $\beta$ be defined as in Section~\ref{sec:thm-E-Delta}.

\begin{theorem}\label{Thm:ErrSamCovDis}
  Assume the observation model given in~\eqref{eq:datagenerate}. Suppose Assumption~\ref{assump:EigenAssump} holds and Assumptions~\ref{assump:CONS}-\ref{assump:observdensity} hold for both $X$ and $Y$. Suppose $T$ and $L$ are large enough so that 
  \begin{equation}\label{eq:curverecovercond}
    \begin{aligned}
      \bar{\psi}_{1}(T,L) \leq \gamma_1 \frac{\delta_n}{M^{1+\beta}},\quad
      \bar{\psi}_{3}(L) \leq \gamma_3 \frac{\delta_n^2}{M^{2+2\beta}}
    \end{aligned}
  \end{equation}
  where
  \begin{multline}\label{eq:ErrCovDis}
  \delta_n = \max \left\{ \frac{ M^{1+\beta} \log \left( 4 \bar{C}_1 n p/ \iota \right) }{ \bar{C}_2 \bar{\Phi} (T,L) }, M^{1+\beta} \sqrt{ \frac{1}{C_6} \bar{\psi}_2 (T, L) \log \left( \frac{ C_5 n p L }{ \iota } \right) },  \right.\\ 
  \left. M^{1+\beta} \sqrt{  \frac{ \log \left( 4 \bar{C}_3 p^2 M^2 / \iota \right) }{ \bar{C}_4 n } } \right\},
  \end{multline}
  $\bar{C}_1=\max\{C^{X}_1,C^{Y}_1\}$, $\bar{C}_2=\min\{C^{X}_2,C^{Y}_2\}$,  $\bar{C}_3=\max\{C^{X}_3,C^{Y}_3\}$,  $\bar{C}_4=\min\{C^{X}_4,C^{Y}_4\}$,
  $\bar{C}_5=\max\{C^{X}_5,C^{Y}_6\}$,
  $\bar{C}_6=\min\{C^{X}_6,C^{Y}_6\}$.
  $\gamma^{X}_k$, $\gamma^{Y}_k$, $k=1,3$, and $C^{X}_k$, $C^{Y}_k$, $k=1,\cdots,6$ are constants that do not depend on $n$, $p$, and $M$. Then
 \begin{equation}
 \max \left\{ \vert S^{X,M}-\Sigma^{X,M} \vert_{\infty}, \vert S^{Y,M}-\Sigma^{Y,M} \vert_{\infty} \right\} \leq \delta_n
 \end{equation}
holds with probability at least $1-\iota$.
\end{theorem}
\begin{proof}
See Appendix~\ref{sec:proof-thm-ErrSamCovDis}.
\end{proof}

The rate $\delta_n$ in Theorem~\ref{Thm:ErrSamCovDis} is composed of three terms. The first two terms correspond to the error incurred by measuring the curves at discrete locations and are approximation errors. The third term, which also appears in Theorem~\ref{thm:proj-score-mat-case2}, is the sampling error.

We provide some insight into how $\tilde{\psi}_{1}$, $\tilde{\psi}_{2}$, $\tilde{\psi}_{3}$, and $\tilde{\psi}_{4}$ depend on $T$ and $L$. Note that we choose an orthonormal basis. Then, as $T \to \infty$, we have
\begin{equation*}
\begin{aligned}
  \frac{1}{T}B^{\top}_{ij}B_{ij}&=
  \frac{1}{T}\sum^{T}_{k=1}
  \left[
\begin{matrix}
b^{2}_{1}(t_{ijk}) & b_{1}(t_{ijk})b_{2}(t_{ijk}) & \cdots & b_{1}(t_{ijk})b_{L}(t_{ijk})\\
\vdots & \vdots & \ddots & \vdots \\
b_{L}(t_{ijk})b_{1}(t_{ijk}) & 
b_{L}(t_{ijk})b_{2}(t_{ijk}) & 
\cdots & b^{2}_{L}(t_{ijk})
\end{matrix}
\right]\\
&\approx \left[
\begin{matrix}
\Vert b_1 \Vert^2 & \langle b_1, b_2 \rangle & \cdots & \langle b_1, b_L \rangle \\
\vdots & \vdots &  & \vdots \\
\langle b_L, b_1 \rangle & \langle b_L, b_2 \rangle & \cdots & \Vert b_L \Vert^2
\end{matrix}
\right]\\
&=\left[
\begin{matrix}
1 & 0 & \cdots & 0 \\
\vdots & \vdots &  & \vdots \\
0 & 0 & \cdots & 1
\end{matrix}
\right].
\end{aligned}
\end{equation*}
Thus, as $T$ grows, we expect $\lambda_{\min}(B^{\top}_{ij}B_{ij}) \approx T$ for any $1\leq j \leq p$ and $1 \leq i \leq n$. This implies that $\tilde{\psi}_1(T,L) \approx L/\sqrt{T}$ and $\tilde{\psi}_2(T,L) \approx \left( D^2_{1,b,L}  \tilde{\psi}_3(L) + \tilde{\psi}_4(L) \right) L^2 /T^2$. Furthermore, $D^2_{1,b,L} \asymp L^2$ when we use the Fourier basis.

To understand $\tilde{\psi}_3(L)$ and $\tilde{\psi}_4(L)$, note that $\lambda^{\bot}_{j0}=\mathbb{E}[ \Vert g^{\bot}_{ij} \Vert^2 ]=\mathbb{E}_{g_{ij}}[ \mathbb{E}_{\epsilon}[ \Vert g^{\bot}_{ij} \Vert^2 \mid g_{ij} ] ]$. Under Assumption~\ref{assump:CONS}, $\lambda^{\bot}_{j0}\to 0$ as $L\to \infty$; however, the speed at which $\lambda^{\bot}_{j0}$ goes to zero will depend on $\mathbb{H}$ and the choice of the basis functions. For example, for a fixed $g_{ij}$, by well-known approximation results (see, for example, \citet{Barron1991Approximation}), if $g_{ij}$ has $r$-th continuous and square integrable derivatives, $\Vert g^{\bot}_{ij} \Vert^2 \approx 1/L^r$ for frequently used bases such as the Legendre polynomials, B-splines, and Fourier basis. Thus, roughly speaking, we should have $\tilde{\psi}_3(L) \approx 1/L^r$ when $\mathbb{H}$ is a Sobolev space of order $r$.  When $g_{ij}$ is an infinitely differentiable function and all derivatives can be uniformly bounded, then $\Vert g^{\bot}_{ij} \Vert^2 \approx \exp(-L)$ and thus $\tilde{\psi}_3(L) \approx \exp(-L)$. Similarly, we have $\tilde{\psi}_4(L) \approx 1 / L^{r-1}$ if $g_{ij}$ has $r$-th continuous and square integrable derivatives; and $\tilde{\psi}_4(L) \approx \exp (-L)$ if $g_{ij}$ is an infinitely differentiable function and all derivatives can be uniformly bounded.

To roughly show how $M$, $T$, $L$ and $n$ may co-vary, we assume that $p$ and $s$ are fixed and all elements of $\mathbb{H}$ have $r$-th continuous and square integrable derivatives. Then FuDGE will recover the differential graph with high probability if $M \ll n^{1/(2+2\beta)}$, $\sqrt{T}/L \gg M^{1+\beta}$, $T \gg L^{2-r/2}$, and $L \gg M^{(1+\beta)/r}$.

As a reviewer pointed out, the noise term in~\eqref{eq:datagenerate} will create a nugget effect in the covariance, meaning that $\text{Var}(h_{ijk})=\text{Var}(g_{ij}(t_{ijk}))+\sigma^2_0$. This nugget effect leads to bias in the estimated eigenvalues (variances of the scores). In our theorem, the nugget effect is reflected by $\sigma_0$ in $\tilde{\psi}_1$. When $\sigma_0$ is large, adding a regularization term when estimating eigenvalues can improve the estimation of FPCA scores and their covariance matrices (see Chapter 6 of~\cite{Hsing2015Theoretical}). However, adding a regularization term increases the number of tuning parameters that need to be chosen. An alternative approach to estimating the covariance matrix is through local polynomial regression \citep{Zhang2016sparse}. Since the focus of the paper is on the estimation of differential functional graphical models, we do not explore ways to improve the estimation of FPCA scores. However, we recognize that there are alternative approaches that can perform better in some cases.

\section{Joint Functional Graphical Lasso}\label{sec:JFGL}

In this section, we introduce two variants of a \textit{Joint Functional Graphical Lasso (JFGL)} estimator, which we compare empirically with our proposed FuDGE procedure in Section~\ref{subsec:Simulation}. \citet{Danaher2011Joint} proposed the \textit{Joint Graphical Lasso (JGL)} to estimate multiple related Gaussian graphical models from different classes simultaneously. Given $Q \geq 2$ data sets, where the $q$-th data set consists of $n_q$ independent random vectors drawn from $N(\mu_{q},\Sigma_{q})$, JGL simultaneously estimates $\{\Theta\}=\{\Theta^{(1)},\Theta^{(2)},\dots,\Theta^{(Q)}\}$, where $\Theta^{(q)}=\Sigma^{-1}_{q}$ is the precision matrix of the $q$-th data set. Specifically, JGL constructs an estimator $\{\hat{\Theta}\}=\{\hat{\Theta}^{(1)},\hat{\Theta}^{(2)},\dots,\hat{\Theta}^{(Q)}\}$ by solving the penalized log-likelihood:
\begin{equation}\label{eq:JGLobj}
\{\hat{\Theta}\}=\argmin_{\{\Theta\}}\left\{ -\sum^{Q}_{q=1}n_q\left( \log\text{det}\Theta^{(q)}-\text{trace}\left(S^{(q)}\Theta^{(q)}\right) \right)+P(\{\Theta\}) \right\},
\end{equation}
where $S^{(q)}$ is the sample covariance of the $q$-th data set and $P(\{\Theta\})$ is a penalty function. The \textit{fused graphical lasso (FGL)} is obtained by setting
\begin{equation}\label{eq:FGLPen}
P(\{\Theta\})=\lambda_1\sum^{Q}_{q=1}\sum_{i\neq j}\vert \Theta^{(q)}_{ij} \vert+\lambda_2\sum_{q<q^{\prime}}\sum_{i \neq j}\vert \Theta^{(q)}_{ij}-\Theta^{(q^{\prime})}_{ij}  \vert,
\end{equation}
while the \textit{group graphical lasso (GGL)} is obtained by setting
\begin{equation}\label{eq:GGLPen}
P(\{\Theta\})=\lambda_1\sum^{Q}_{q=1}\sum_{i\neq j}\vert \Theta^{(q)}_{ij} \vert+\lambda_2\sum_{i\neq j}\sqrt{\sum^{Q}_{q=1}\left(\Theta^{(q)}_{ij}\right)^2}.
\end{equation}
The terms $\lambda_1$ and $\lambda_2$ are non-negative tuning parameters, while $\Theta^{(q)}_{ij}$ denotes the $(i,j)$-th entry of $\Theta^{(q)}$. For both penalties, the first term is the lasso penalty, which encourages sparsity for the off-diagonal entries of all precision matrices; however, FGL and GGL differ in the second term. For FGL, the second term encourages the off-diagonal entries of precision matrices among all classes to be similar, which means that it encourages not only a similar network structure but also similar edge values. For GGL, the second term is a group lasso penalty, which encourages the support of the precision matrices to be similar but allows the specific values to differ. See \citet{Tsai2021Joint} for a recent survey of joint estimation procedures for joint Gaussian graphical models.

A similar approach can be used to estimate the precision matrix of the score vectors. Unlike the direct estimation procedure proposed in Section~\ref{sec:FuDGE}, we could first estimate $\hat{\Theta}^{X,M}$ and $\hat{\Theta}^{Y,M}$ using a joint graphical lasso objective, and then take the difference to estimate $\Delta$.

In the functional graphical model setting, we are interested in the block sparsity, so we modify the entry-wise penalties to a block-wise penalty. Specifically, we propose to solve the objective function in \eqref{eq:JGLobj}, where $S^{(q)}$ and $\Theta^{(q)}$ denote the sample covariance and the estimated precision of the projection scores for the $q$-th group. Note that now $S^{(q)}$, $\Theta^{(q)}$ and $\hat{\Theta}^{(q)}$, $q=1,\ldots,Q$ are all $pM \times pM$ matrices. Similarly to the GGL and FGL procedures, we define the \textit{Grouped Functional Graphical Lasso (GFGL)} and \textit{Fused Functional Graphical Lasso (FFGL)} penalties for functional graphs. Specifically, letting $\Theta^{(q)}_{jl}$ denote the $(j,l)$-th $M \times M$ block matrix, the GFGL penalty is
\begin{equation}\label{eq:GFGLPen}
P(\{\Theta\})=\lambda_1\sum^{Q}_{q=1}\sum_{j\neq l}\Vert \Theta^{(q)}_{jl} \Vert_{\text{F}}+\lambda_2\sum_{j\neq l}\sqrt{\sum^{Q}_{q=1}\Vert\Theta^{(q)}_{jl}\Vert^{2}_{\text{F}}},
\end{equation}
where $\lambda_1$ and $\lambda_2$ are non-negative tuning parameters. The FFGL penalty can be defined in two ways. The first way is to use the Frobenius norm for the second term:
\begin{equation}\label{eq:FFGLPen}
P(\{\Theta\})=\lambda_1\sum^{Q}_{q=1}\sum_{j\neq l}\Vert \Theta^{(q)}_{jl} \Vert_{\text{F}}+\lambda_2\sum_{q<q^{\prime}}\sum_{j,l}\Vert \Theta^{(q)}_{jl}-\Theta^{(q^{\prime})}_{jl} \Vert_{\text{F}}.
\end{equation}
The second way is to keep the element-wise $L_1$ norm as in FGL:
\begin{equation}\label{eq:FFGL2Pen}
P(\{\Theta\})=\lambda_1\sum^{Q}_{q=1}\sum_{j\neq l}\Vert \Theta^{(q)}_{jl} \Vert_{\text{F}}+\lambda_2\sum_{q<q^{\prime}}\sum_{j,l}\vert \Theta^{(q)}_{jl}-\Theta^{(q^{\prime})}_{jl} \vert_1,
\end{equation}
where $\lambda_1$ and $\lambda_2$ are non-negative tuning parameters. 

The Joint Functional Graphical Lasso accommodates an arbitrary $Q$. However, when estimating the functional differential graph, we set $Q=2$. We will refer to \eqref{eq:FFGLPen} as FFGL and to \eqref{eq:FFGL2Pen} as FFGL2. The algorithms to solve GFGL, FFGL, and FFGL2 are given in the Appendix~\ref{apex:opt_alg}.

\section{Experiments}\label{sec:experiments}

We examine the performance of FuDGE using both simulations and a real data set.\footnote{Code to replicate the simulations is available at \url{https://github.com/boxinz17/FuDGE}.}

\subsection{Simulations}\label{subsec:Simulation}
Given a graph $G_X$, we generate samples of $X$ such that $X_{ij}(t)=b^{\prime}(t)^{\top}\delta^{X}_{ij}$. The coefficients $\delta^{X}_{i}=((\delta^{X}_{i1})^{\top},\ldots,(\delta^{X}_{ip})^{\top})^{\top}\in{\mathbb{R}^{mp}}$ are drawn from $N\left(0, (\Omega^X)^{-1}\right)$ where $\Omega_X$ is described below. In all cases, $b^{\prime}(t)$ is an $m$-dimensional basis with disjoint support over $[0,1]$ such that for $k=1, \ldots m$:
\begin{equation}\label{eq:FunGeneration}
b^{\prime}_{k}(t)=
\begin{cases}
\cos\left(10\pi\left(x-(2k-1)/10\right)\right)+1 & \text{if }  (k-1)/m\leq{x}<k/m; \\
0 & \text{otherwise}.
\end{cases}
\end{equation}
To generate noisy observations at discrete time points, we sample data
\begin{equation*}
h^{X}_{ijk}=X_{ij}(t_{k})+e_{ijk},\quad e_{ijk}\sim N(0,0.5^2),
\end{equation*}
for $200$ evenly spaced time points $0 = t_{1} \leq \ldots \leq t_{200}= 1$. $Y_{ij}(t)$ and $h^{Y}_{ijk}$ are sampled in an analogous procedure. We use $m=5$ for the experiments below, except for the simulation, where we explore the effect of $m$ on empirical performance.

We consider three different simulation settings for the construction of $G_X$ and $G_Y$. In each setting, we let $n_X=n_Y=100$ and $p=30,60,90, 120$, and replicate the procedure 30 times for each $p$ and the model setting.

\textbf{Model 1:} This model is similar to the setting considered in \cite{zhao2014direct}, but modified for the functional case. We generate the support of $\Omega^{X}$ according to a graph with $p(p-1)/10$ edges and a power law degree distribution with an expected power parameter of 2. Although the graph is sparse with only 20\% of all possible edges present, the power-law structure mimics certain real-world graphs by creating hub nodes with a large degree \citep{newman2003structure}. For each non-zero block, we set $\Omega^{X}_{jl}=\delta^{\prime}I_{5}$, where $\delta^{\prime}$ is sampled uniformly from $\pm [0.2, 0.5]$. To ensure positive definiteness, we further scale each off-diagonal block by $1/2,1/3,1/4,1/5$ for $p=30, 60, 90, 120$ respectively. Each diagonal element of $\Omega^X$ is set to $1$ and the matrix is symmetrized by averaging it with its transpose.  To get $\Omega^{Y}$, we first select the top 2 hub nodes in $G_X$ (i.e., the nodes with top 2 largest degree), and for each hub node we select the top (by magnitude) 20\% of edges. For each selected edge, we set $\Omega^{Y}_ {jl}=\Omega^{X}_{jl}+W$ where $W_{k k^{\prime}}=0$ for $|k-k^{\prime}|\leq{2}$, and $W_{k k^{\prime}}=c$ otherwise, where $c$ is generated the same way as $\delta^{\prime}$.  For all other blocks, $\Omega^{Y}_{jl} = \Omega^{X}_{jl}$.

\textbf{Model 2:} We first generate a tridiagonal block matrix $\Omega^{*}_{X}$  with $\Omega^{*}_{X, jj}=I_5$, $\Omega^{*}_{X,j,j+1}=\Omega^{*}_{X,j+1,j}=0.6I_5$,  and $\Omega^{*}_{X,j,j+2}=\Omega^{*}_{X,j+2,j}=0.4I_5$ for $j=1,\ldots,p$.  All other blocks are set to 0. We form $G_Y$ by adding four edges to $G_X$. Specifically, we first let $\Omega^{*}_{Y,jl}=\Omega^{*}_{X,jl}$ for all blocks, then for $j = 1, 2, 3, 4$, we set $\Omega^{*}_{Y,j,j+3}=\Omega^{*}_{Y, j+3,j}=W$, where $W_{k k^{\prime}}=0.1$ for all $1 \leq k,k^{\prime} \leq M$. Finally, we set $\Omega^{X}=\Omega^{*}_{X}+\delta I$, $\Omega^{Y}=\Omega^{*}_{Y}+\delta I$, where $\delta=\max\left\{|\min(\lambda_{\min}(\Omega^{*}_{X}),0)|, |\min(\lambda_{\min}(\Omega^{*}_{Y}),0)|\right\} + 0.05$.

\textbf{Model 3:} We generate $\Omega^{*}_{X}$ according to an Erd\"{o}s-R\'{e}nyi graph. We first set $\Omega^{*}_{X,jj}=I_5$.  With probability $.8$, we set $\Omega^{*}_{X,jl}=\Omega^{*}_{X,lj}=0.1 I_5$, and set it to $0$ otherwise.  Thus, we expect 80\% of all possible edges to be present. Then we form $G_Y$ by randomly adding $s$ new edges to $G_X$, where $s=3$ for $p=30$, $s=4$ for $p=60$, $s=5$ for $p=90$, and $s=6$ for $p=120$. We set each corresponding block as $\Omega^{*}_{Y,jl} = W$, where $W_{k k^{\prime}}=0$ when $|k-k^{\prime}|\leq{1}$ and $W_{k k^{\prime}}=c$ otherwise.  We let $c=2/5$ for $p=30$, $c=4/15$ for $p=60$, $c=1/5$ for $p=90$, and $c=4/25$ for $p=120$.  Finally, we set $\Omega^{X}=\Omega^{*}_{X}+\delta I$, $\Omega^{Y}=\Omega^{*}_{Y}+\delta I$, where $\delta=\max\left\{|\min(\lambda_{\min}(\Omega^{*}_{X}),0)|,  |\min(\lambda_{\min}(\Omega^{*}_{Y}),0)|\right\} + 0.05$.

\begin{figure}[t]
	\centering
	\includegraphics[width=0.9\linewidth]{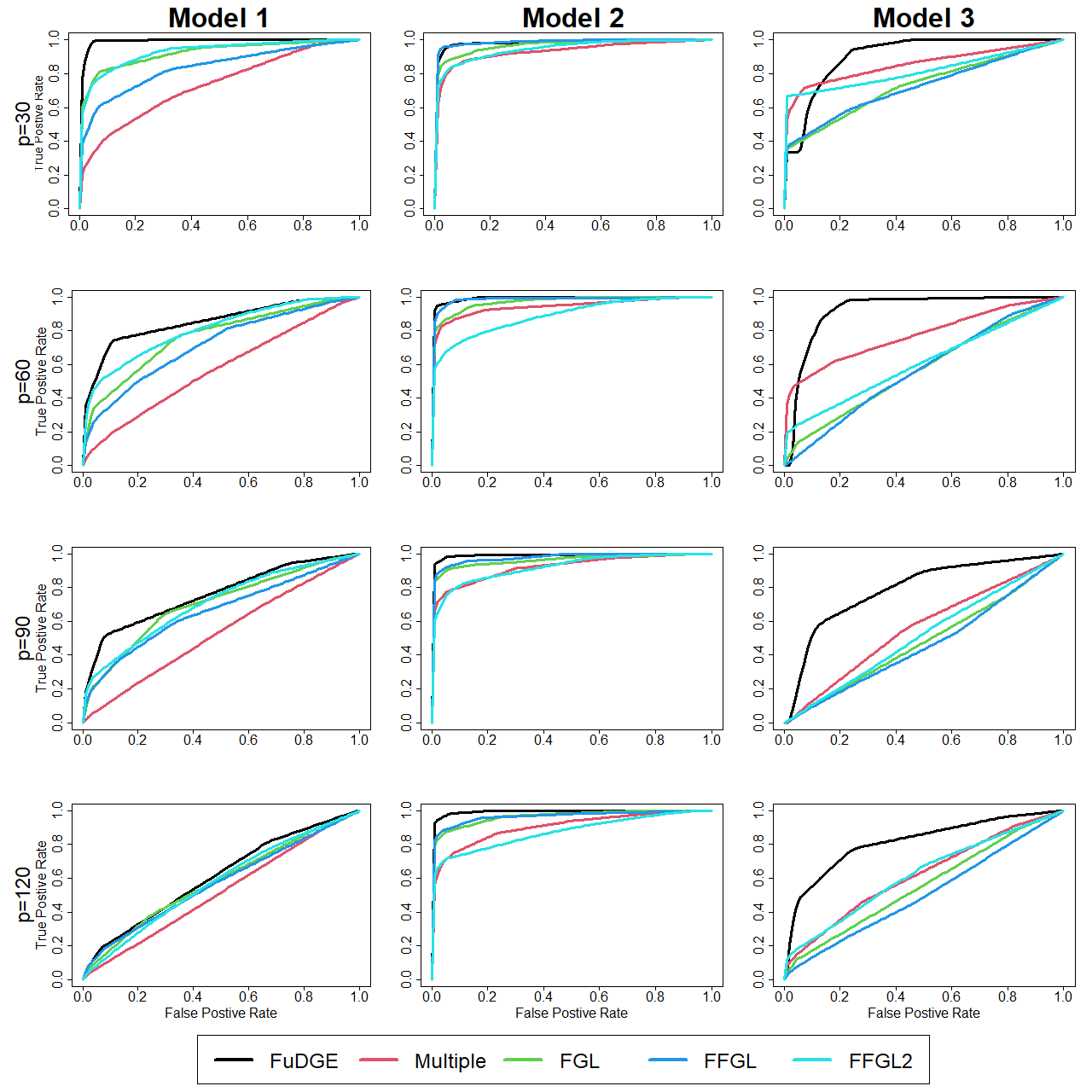}
	\caption{Average ROC curves across 30 simulations. Different columns correspond to different models, different rows correspond to different dimensions.}
	\label{fig:ROC}
\end{figure}

We compare FuDGE with four competing methods. The first competing method (denoted \emph{multiple} in Figure~\ref{fig:ROC}) ignores the functional nature of the data. We select 15 equally spaced time points and at each time point implement a direct difference estimation procedure \citep{zhao2014direct} to estimate the graph at that time point. Specifically, for each $t$, $X_{i}(t)$ and $Y_{i}(t)$ are simply $p$-dimensional random vectors, and we use their sample covariances in \eqref{eq:objectivefunc} to obtain a $p\times p$ matrix $\hat{\Delta}$. This produces 15 differential graphs, and we use a majority vote to form a single differential graph. The ROC curve is obtained by changing the $L_1$ penalty, $\lambda_{n}$, used for all time points.

The other three competing methods all estimate two functional graphical models using either the Joint Graphical Lasso or the Functional Joint Graphical Lasso introduced in Section~\ref{sec:JFGL}. For each method, we first estimate the sample covariances of the FPCA scores for $X$ and $Y$. The second competing method (denoted \emph{FGL}) ignores the block structure in precision matrices and applies the fused graphical lasso method directly. The third and fourth competing methods take into account the block structure and apply FFGL and FFGL2 defined in Section~\ref{sec:JFGL}. To draw an ROC curve, we follow the same approach as in \cite{zhao2014direct}. We first fix $\lambda_1=0.1$, which controls the overall sparsity in each graph; then we form an ROC curve by varying $\lambda_2$, which controls the similarity between two graphs.

For each setting and method, the ROC curve averaged across the $30$
replications is shown in Figure~\ref{fig:ROC}. We see that FuDGE clearly has the best overall performance in recovering the support of the differential graph for all cases. We also note that explicit consideration of block structure in the joint graphical methods does not seem to make a substantial difference as the performance of FGL is comparable to FFGL and FFGL2.

\textbf{The effect of the number of basis functions:} To examine how the accuracy of the estimation is associated with the dimension of the functional data, we repeat the experiment under Model 1 with $p=30$ and vary the number of basis functions used to generate the data in \eqref{eq:FunGeneration}. In each case, the number of principal components selected by cross-validation is $M =4$. In Figure~\ref{fig:ROC_basis_change}, we see that as the gap between the true dimension $m$ and the number of dimensions used $M$ increases, the performance of FuDGE degrades slightly, but remains relatively robust. This is because the FPCA procedure is data adaptive and produces an eigenfunction basis that approximates the true functions well with a relatively small number of basis functions.

\begin{figure}[t]
	\centering
	\includegraphics[width=0.5\linewidth]{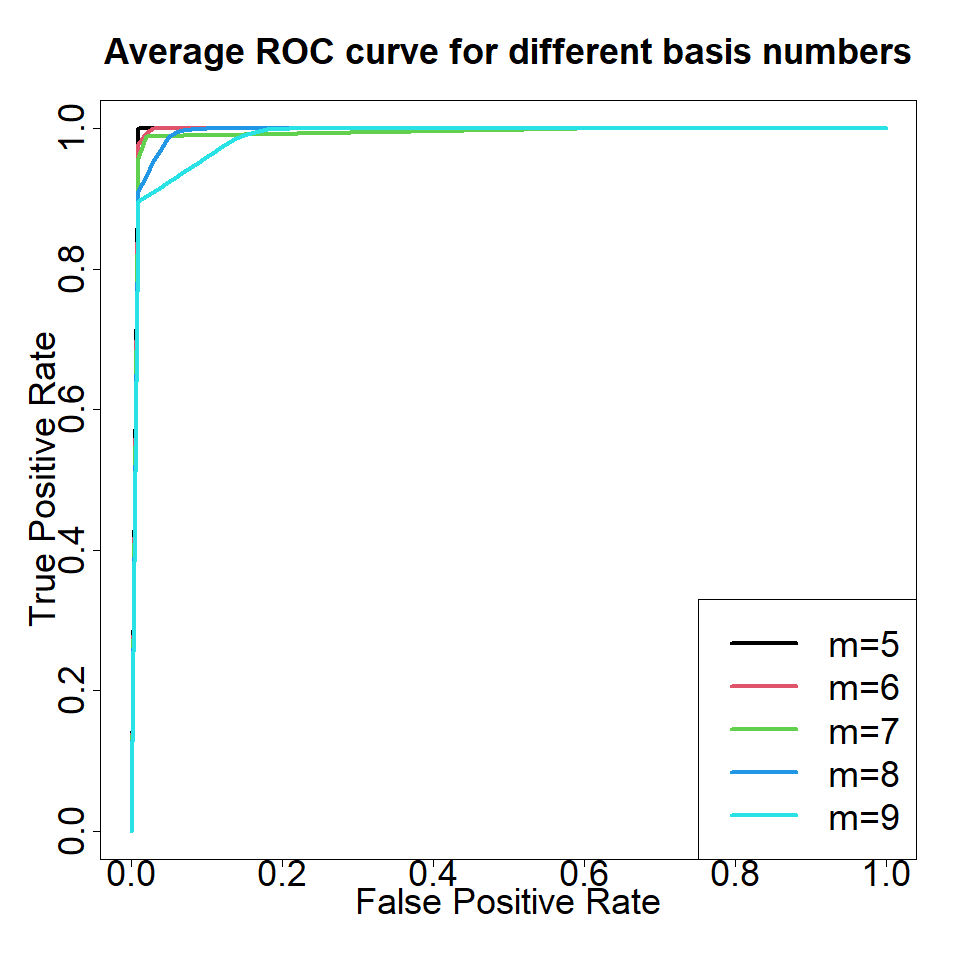}
	\caption{ROC curves for Model 1 with $p=30$ and changing number of basis functions~$m$. Each curve is drawn by averaging across 30 simulations. The number of eigenfunctions, $M$, selected by cross-validation is 4 in each replication.}
	\label{fig:ROC_basis_change}
\end{figure}

\subsection{Neuroscience Application}\label{subsec:EEG}
We apply our method to electroencephalogram (EEG) data obtained from a study \citep{Zhang1995Event, Ingber1997Statistical}, which included 122 total subjects; 77 individuals with alcohol use disorder (AUD) and 45 in the control group. Specifically, the EEG data was measured by placing $p=64$ electrodes at various locations on the subject's scalp and measuring voltage values over time. We follow the preprocessing procedure in \cite{knyazev2007motivation} and \cite{zhu2016bayesian}, which filters the EEG signals at $\alpha$ frequency bands between 8 and 12.5 Hz.

\cite{Qiao2015Functional} estimate separate functional graphs for each group, but we directly estimate the differential graph using FuDGE. We choose $\lambda_{n}$ so that the estimated differential graph has approximately 1\% of possible edges.  The estimated edges of the differential graph are shown in Figure~\ref{fig:EEG}.

In this setting, an edge in the differential graph suggests that the communication pattern between two different regions of the brain may be affected by alcohol use disorder. However, the differential graph does not exactly indicate how the communication pattern has changed. For example, the edge between P4 and P6 suggests that AUD affects the communication pattern between those two regions; however, it could be that these two regions are (conditionally) associated with the control group, but not with the AUD group or vice versa. It could also be that the two regions are (conditionally) associated in both groups, but the conditional covariance is different. However, many interesting observations can be gleaned from the results and may generate interesting hypotheses that could be investigated more thoroughly in an experimental setting.

\begin{figure}[t]
	\centering
	\includegraphics[width=0.7\linewidth]{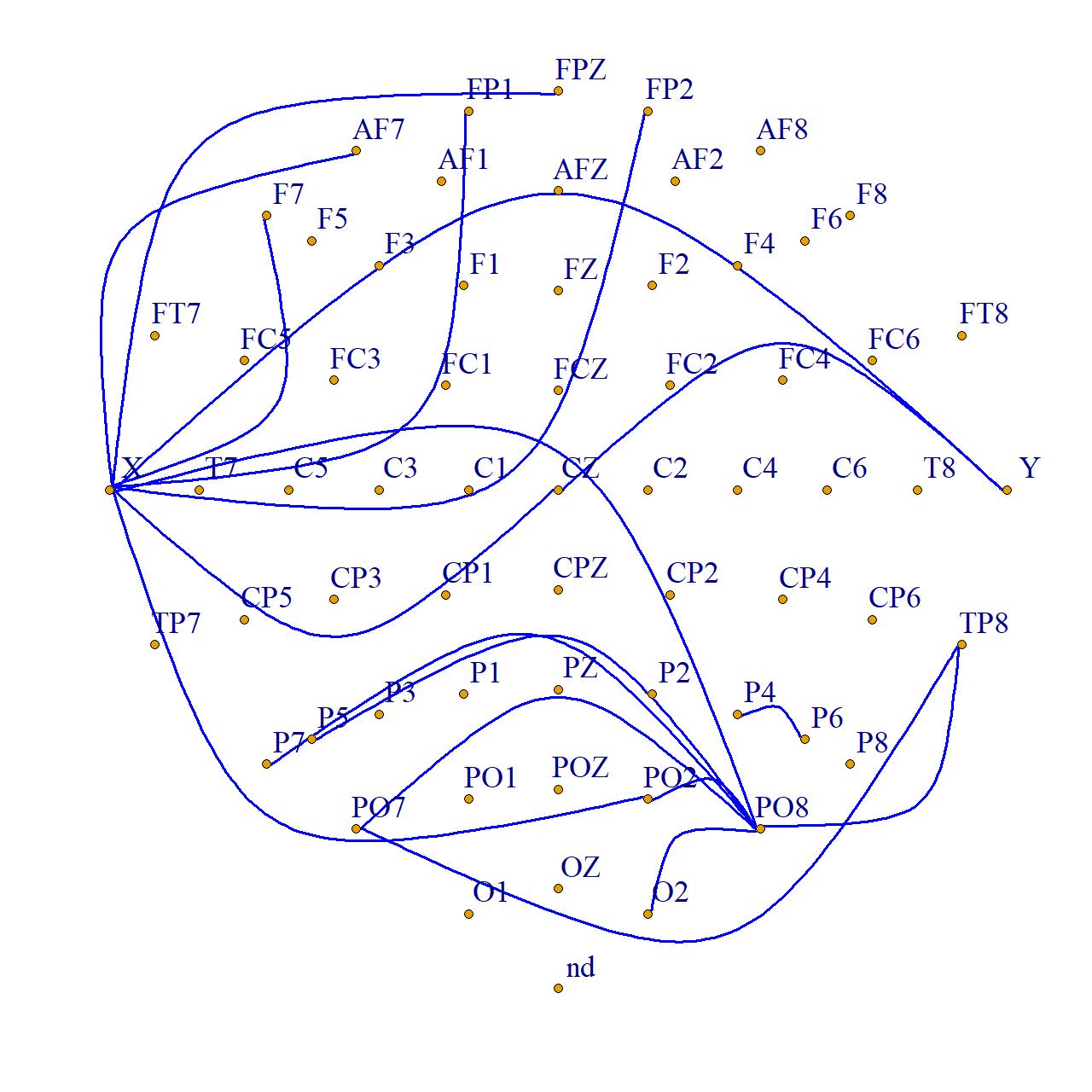}
	\caption{Estimated differential graph for EEG data. The anterior region is the top of the figure and the posterior region is the bottom of the figure.}
	\label{fig:EEG}
\end{figure}

We give two specific observations. First, edges are generally between nodes located in the same region---either the anterior region or the posterior region---and there is no edge that crosses between regions. This observation is consistent with the result in \cite{Qiao2015Functional} where there are no connections between the anterior and posterior regions for both groups. We also note that electrode X, lying in the middle left region, has a high degree in the estimated differential graph. Although there is no direct connection between the anterior and posterior regions, this region may play a role in helping the two parts communicate and may be greatly affected by AUD. Similarly, P08 in the anterior region also has a high degree and is connected to other nodes in the anterior region, which may indicate that this region can be an information exchange center for the anterior regions, which, at the same time, may be heavily affected by AUD.

\section{Discussion}

We proposed a method to directly estimate the differential graph for functional graphical models. In certain settings, direct estimation allows the differential graph to be recovered consistently, even if each underlying graph cannot be consistently recovered.  Experiments with simulated data also show that preserving the functional nature of the data rather than treating the data as multivariate scalars can also result in better estimation of the differential graph.

A key step in the procedure is to first represent the functions with an $M$-dimensional basis using FPCA. Definition~\ref{def:DGO} ensures that there exists some $M$ large enough so that the signal, $\nu_1(M)$, is larger than the bias, $\nu_2(M)$, due to the use of a finite-dimensional representation. Intuitively, $\tau = \nu_1(M) - \nu_2(M)$ is tied to the eigenvalue decay rate; however, we defer the derivation of the explicit connection for future work. 

We have provided a method for direct estimation of the differential graph, but the development of methods that allow for inference and hypothesis testing in functional differential graphs is a fruitful avenue for future work. In recent years, a number of studies have focused on inference in high-dimensional linear models \citep{Zhang2011Confidence,Geer2013asymptotically,Javanmard2013Confidence,Zhao2014General,Bradic2017Uniform,Wang2021Robust}. Subsequently, these approaches were extended for statistical inference of low-dimensional parameters in graphical models \citep{Ren2013Asymptotic,Wasserman2014Berry,Jankova2014Confidence,Jankova2017Honest,Barber2015ROCKET,Yu2016Statistical,Yu2019Simultaneous} and differential graphical models \citep{Xia2015Testing,Liu2017Structural,kim2019two}. Future work may extend these results to the functional graph setting. A promising approach would be to extend the inference procedures developed for semi- and non-parametric models \citep[see, e.g.,][]{Lu2019Kernel,Dai2021Inference}.

\acks{We thank the associate editor Daniela Witten and reviewers for their helpful feedback which has greatly improved the manuscript. This work is partially supported by the William S. Fishman Faculty Research Fund at the University of Chicago Booth School of Business. This work was completed in part with resources provided by the University of Chicago Research Computing Center.}

\newpage

\begin{appendices}


\numberwithin{equation}{section}

\section{Derivation of Optimization Algorithm} \label{apex:opt_alg}

In this section, we derive the key steps for optimization algorithms.

\subsection{Optimization Algorithm for FuDGE}

We derive closed-form updates for the proximal method stated in \eqref{con:itersimplified}. In particular, recall that for all $1\leq{j,l}\leq{p}$, we have
\begin{equation*}
\Delta^{\text{new}}_{jl} \;=\; \left[\left(\Vert A^{\text{old}}_{jl}\Vert_{F}-\lambda_{n}\eta\right)/\Vert A^{\text{old}}_{jl}\Vert_{F}\right]_{+}\times A^{\text{old}}_{jl},
\end{equation*}
where $A^{\text{old}}=\Delta^{\text{old}}-\eta\nabla \mathcal{L}(\Delta^{\text{old}})$ and $x_{+}=\max\{0,x\}$ represents the positive part of~$x\in{\mathbb{R}}$.

\begin{proof}[Proof of \eqref{con:itersimplified}]
Let $A^{\text{old}}=\Delta^{\text{old}}-\eta\nabla \mathcal{L}(\Delta^{\text{old}})$ and let $f_{jl}$ denote the loss decomposed over each $j,l$ block so that
\begin{equation}{\label{con:splitlost}}
f_{jl}(\Delta_{jl})\;=\;\frac{1}{2\lambda_{n}\eta}\|\Delta_{jl}-A^{\text{old}}_{jl} \|^{2}_{F}+\|\Delta_{jl}\|_{F}
\end{equation}
and
\begin{equation}
\Delta^{\text{new}}_{jl}\;=\;\argmin_{\Delta_{jl}\in{\mathbb{R}^{M\times{M}}}}f_{jl}(\Delta_{jl}).
\end{equation}
The loss $f_{jl}(\Delta_{jl})$ is convex, so the first-order optimality condition implies that:
\begin{equation}{\label{con:optimcondition}}
0\in \partial f_{jl}\left(\Delta^{\text{new}}_{jl}\right),
\end{equation}
where $\partial f_{jl}\left(\Delta_{jl}\right)$ is the subdifferential of $f_{jl}$ at $\Delta_{jl}$:
\begin{equation}{\label{con:subdiff}}
\partial f_{jl}(\Delta_{jl})\;=\; \frac{1}{\lambda_{n}\eta}\left(\Delta_{jl}-A^{\text{old}}_{jl}\right)+Z_{jl},
\end{equation}
where
\begin{equation}{\label{con:subdiff2}}
Z_{jl} \;=\;
\begin{cases}
\frac{\Delta_{jl}}{\|\Delta_{jl}\|_{F}}\qquad & \text{ if } \Delta_{jl}\neq{0}\\[10pt]
\left\{Z_{jl}\in{\mathbb{R}^{M\times{M}}}\colon \|Z_{jl}\|_{F}\leq{1}\right\} \qquad & \text{ if } \Delta_{jl}=0. 
\end{cases}
\end{equation}
\textbf{Claim 1} If $\|A^{\text{old}}_{jl}\|_{F}>\lambda_{n}\eta > 0$, then $\Delta^{\text{new}}_{jl}\neq{0}$.

We verify this claim by proving the contrapositive. Suppose $\Delta^{\text{new}}_{jl} = {0}$. Then by \eqref{con:optimcondition} and \eqref{con:subdiff2}, there exists a $Z_{jl}\in{\mathbb{R}^{M\times{M}}}$ such that $\|Z_{jl}\|_{F}\leq{1}$ and
\begin{equation*}
0=-\frac{1}{\lambda_{n}\eta}A^{\text{old}}_{jl}+Z_{jl}.
\end{equation*}
Thus, $\|A^{\text{old}}_{jl}\|_{F}=\|\lambda_{n}\eta\cdot Z_{jl}\|_{F}\leq{\lambda_{n}\eta}$, so that Claim 1 holds.

Combining Claim 1 with \eqref{con:optimcondition} and \eqref{con:subdiff2}, for any $j,l$ such that $\|A^{\text{old}}_{jl}\|_{F}>\lambda_{n}\eta$, we have
\begin{equation*}
0=\frac{1}{\lambda_{n}\eta}\left(\Delta^{\text{new}}_{jl}-A^{\text{old}}_{jl}\right)+\frac{\Delta^{\text{new}}_{jl}}{\|\Delta^{\text{new}}_{jl}\|_{F}},
\end{equation*}
which is solved by
\begin{equation}{\label{con:solution1}}
\Delta^{\text{new}}_{jl}=\frac{\|A^{\text{old}}_{jl}\|_{F}-\lambda_{n}\eta}{\|A^{\text{old}}_{jl}\|_{F}}A^{\text{old}}_{jl}.
\end{equation}
\textbf{Claim 2} If $\|A^{\text{old}}_{jl}\|_{F}\leq\lambda_{n}\eta$, then $\Delta^{\text{new}}_{jl}=0$.

Again, we verify the claim by proving the contrapositive. Suppose $\Delta^{\text{new}}_{jl}\neq 0$. Then the first-order optimality implies the updates in \eqref{con:solution1}. However, taking the Frobenius norm on both sides of the equation gives $\|\Delta^{\text{new}}_{jl}\|_{F}=\|A^{\text{old}}_{jl}\|_{F}-\lambda_{n}\eta$, which implies that $\|A^{\text{old}}_{jl}\|_{F}-\lambda_{n}\eta\geq{0}$.

Updates in \eqref{con:itersimplified} follow immediately by combining Claim 2 and \eqref{con:solution1}.
\end{proof}

\subsection{Solving the Joint Functional Graphical Lasso}

As in \cite{Danaher2011Joint}, we use the alternating directions method of multipliers (ADMM) algorithm to solve \eqref{eq:JGLobj}; see \cite{Boyd2011Distributed} for a detailed exposition of ADMM.

To solve \eqref{eq:JGLobj}, we first rewrite the problem as:
\begin{equation*}
\max_{\{\Theta\},\{ Z \}}\left\{ -\sum^{Q}_{q=1}n_q\left( \log\text{det}\Theta^{(q)}-\text{trace}\left(S^{(q)}\Theta^{(q)}\right) \right)+P(\{Z\}) \right\},
\end{equation*}
subject to $\Theta^{(q)}\succ0$ and $Z^{(q)}=\Theta^{(q)}$, where $\{Z\}=\{Z^{(1)},Z^{(2)},\dots,Z^{(Q)}\}$. The scaled augmented Lagrangian \citep{Boyd2011Distributed} is given by
\begin{multline}\label{eq:scaledAgumentedLagrange}
L_{\rho}\left(\{\Theta\},\{Z\},\{U\}\right)=-\sum^{Q}_{q=1}n_q\left( \log\text{det}\Theta^{(q)}-\text{trace}\left(S^{(q)}\Theta^{(q)}\right) \right)+P(\{Z\})\\
+\frac{\rho}{2}\sum^{Q}_{q=1}\Vert \Theta^{(q)}-Z^{(q)}+U^{(q)} \Vert^{2}_{\text{F}},
\end{multline}
where $\rho>0$ is a tuning parameter and $\{U\}=\{U^{(1)},U^{(2)},\dots,U^{(Q)}\}$ are dual variables. The ADMM algorithm will then solve \eqref{eq:scaledAgumentedLagrange} by iterating the following three steps. At the $i$-th iteration, they are as follows:

\begin{enumerate}[topsep=0pt,itemsep=-1ex,partopsep=1ex,parsep=1ex]
\item $\{\Theta_{(i)}\}\leftarrow\argmin_{\{\Theta \}}L_{\rho}\left(\{\Theta\},\{Z_{(i-1)}\},\{U_{(i-1)}\}\right)$.
\item $\{Z_{(i)}\}\leftarrow\argmin_{\{Z \}}L_{\rho}\left(\{\Theta_{(i)}\},\{Z\},\{U_{(i-1)}\}\right)$.
\item $\{U_{(i)}\}\leftarrow \{U_{(i-1)}\}+(\{\Theta_{(i)}\}-\{Z_{(i)}\})$.
\end{enumerate}

We now give more details on the above three steps.

\centerline{\textbf{ADMM algorithm for solving the joint functional graphical lasso problem}}

\smallskip

\textbf{Input:} $\{S^{(q)}\}^Q_{q=1}$, $\{n_q\}^{Q}_{q=1}$, and the penalty term $P(\cdot)$.

\textbf{Output:} $\{\hat{\Theta}^{(q)} \}^Q_{q=1}$.

(a) Initialize the variables: $\Theta^{(q)}_{(0)}=I_{pM}$, $U^{(q)}_{(0)}=0_{pM}$, and $Z^{(q)}_{(0)}=0_{pM}$, $q=1,\ldots,Q$.

(b) Select a scalar $\rho>0$.

(c) For $i=1,2,3,\dots$ until convergence

\hspace*{15pt} (i) For $q=1,\ldots,Q$, update $\Theta^{(q)}_{(i)}$ as the minimizer (with respect to $\Theta^{(q)}$) of
\begin{equation*}
-n_q\left( \log\text{det}\Theta^{(q)}-\text{trace}\left(S^{(q)}\Theta^{(q)}\right) \right)+\frac{\rho}{2}\Vert \Theta^{(q)}-Z^{(q)}_{(i-1)}+U^{(q)}_{(i-1)} \Vert^{2}_{\text{F}}
\end{equation*}
\begin{adjustwidth}{1.8cm}{}
	Let $VDV^{\top}$ denote the eigendecomposition of $S^{(q)}-\rho Z^{(q)}_{(i-1)}/n_q+\rho U^{(q)}_{(i-1)}/n_q$. The solution is given by $V\tilde{D}V^{\top}$, where $\tilde{D}$ is the diagonal matrix with the $j$-th diagonal element being
	\begin{equation*}
	\frac{n_q}{2\rho}\left(-D_{jj}+\sqrt{D^{2}_{jj}+4\rho/n_q}\right),
	\end{equation*}
	where $D_{jj}$ is the $(j,j)$-th entry of $D$.

\end{adjustwidth}

\hspace*{15pt} (ii) Update $\{Z_{(i)}\}$ as minimizer (with respect to $\{Z\}$) of
\begin{equation}\label{eq:ADMMkeystep}
\min_{\{Z\}} \frac{\rho}{2}\sum^{Q}_{q=1} \Vert Z^{(q)}-A^{(q)} \Vert^{2}_{\text{F}}+P(\{Z\}),
\end{equation}
\begin{adjustwidth}{1.8cm}{}
	where $A^{(q)}=\Theta^{(q)}_{(i)}+U^{(q)}_{(i-1)}$, $q=1,\ldots,Q$.
\end{adjustwidth}

\hspace*{15pt} (iii) $U^{(q)}_{(i)}\leftarrow U^{(q)}_{(i-1)}+(\Theta^{(q)}_{(i)}-Z^{(q)}_{(i)})$, $q=1,\ldots,Q$.

(d) Output $\hat{\Theta}^{(q)}$ as $\Theta^{(q)}_{(i)}$, $q=1,\ldots,Q$, from the final round.

\smallskip

There are three things that are worth noting. \textbf{1.} The key step is to solve \eqref{eq:ADMMkeystep}, which depends on the form of penalty term $P(\cdot)$; \textbf{2.} This algorithm is guaranteed to converge to the global optimum when $P(\cdot)$ is convex \citep{Boyd2011Distributed}; \textbf{3.} The positive-definiteness constraint on $\{\hat{\Theta}\}$ is naturally enforced by step (c) (i).

\subsection{Solving \texorpdfstring{\eqref{eq:ADMMkeystep}}{} for different penalty functions}

We provide solutions to \eqref{eq:ADMMkeystep} for three problems (GFGL, FFGL, FFGL2) defined by \eqref{eq:GFGLPen}, \eqref{eq:FFGLPen}, and \eqref{eq:FFGL2Pen}.

\subsubsection{Solution to \texorpdfstring{\eqref{eq:ADMMkeystep}}{} for GFGL}

Let the solution for
\begin{equation}
\min_{\{Z\}}\frac{\rho}{2}\sum^{Q}_{q=1}\Vert Z^{(q)}-A^{(q)} \Vert^{2}_{\text{F}}+\lambda_1\sum^{Q}_{q=1}\sum_{j \neq l}\Vert Z^{(q)}_{jl} \Vert_{\text{F}}+\lambda_2\sum_{j \neq l}\left( \sum^{Q}_{q=1}\Vert Z^{(q)}_{jl} \Vert^{2}_{\text{F}} \right)^{1/2}
\end{equation}
be denoted as $\{\hat{Z} \}=\{\hat{Z}^{(1)}, \hat{Z}^{(2)},\dots, \hat{Z}^{(Q)}\}$. Let $Z^{(q)}_{jl}$, $\hat{Z}^{(q)}_{jl}$ be the $(j,l)$-th $M \times M$ block of $Z^{(q)}$ and $\hat{Z}^{(q)}$, $q=1,\ldots,Q$. Then, for $j=1,\ldots,p$, we have
\begin{equation}\label{eq:GFGLkeysetp1}
\hat{Z}^{(q)}_{jj}=A^{(q)}_{jj},\qquad q=1,\ldots,Q,
\end{equation}
and, for $j \neq l$, we have
\begin{equation}\label{eq:GFGLkeysetp2}
\hat{Z}^{(q)}_{jl}=\left( \frac{\Vert A^{(q)}_{jl} \Vert_{\text{F}}-\lambda_1/\rho}{\Vert A^{(q)}_{jl} \Vert_{\text{F}}} \right)_{+}\left( 1-\frac{\lambda_2}{\rho\sqrt{\sum^{Q}_{q=1}\left(\Vert A^{(q)}_{jl} \Vert_{\text{F}}-\lambda_1/\rho  \right)^{2}_{+} }} \right)_{+}A^{(q)}_{jl},
\end{equation}
where $q=1,\ldots,Q$.
Details of the proof of \eqref{eq:GFGLkeysetp1} and \eqref{eq:GFGLkeysetp2} are given in Appendix~\ref{subsec:OptGFGL}.

\subsubsection{Solution to \texorpdfstring{\eqref{eq:ADMMkeystep}}{} for FFGL}\label{subsubsec:solFFGL}

For FFGL, there is no simple closed-form solution. When $Q=2$, \eqref{eq:ADMMkeystep} becomes
\begin{equation*}
\min_{\{Z\}}\; \frac{\rho}{2}\sum^{2}_{q=1} \Vert Z^{(q)}-A^{(q)} \Vert^{2}_{\text{F}}+\lambda_1\left(\sum^{2}_{q=1}\sum_{j \neq l}\Vert Z^{(q)}_{jl} \Vert_{\text{F}} \right)+\lambda_2 \sum_{j,l}\Vert Z^{(1)}_{jl}-Z^{(2)}_{jl} \Vert_{\text{F}}.
\end{equation*}
For each $1 \leq j,l \leq p$, we compute $\hat{Z}^{(1)}_{jl}$, $\hat{Z}^{(2)}_{jl}$ by solving
\begin{equation}\label{eq:FFGLQ2keysetp}
\min_{\{Z^{(1)}_{jl},Z^{(2)}_{jl}\}}\; \frac{1}{2}\sum^{2}_{q=1} \Vert Z^{(q)}_{jl}-A^{(q)}_{jl} \Vert^{2}_{\text{F}}+\frac{\lambda_1}{\rho}\mathbbm{1}_{j \neq l} \sum^{2}_{q=1}\Vert Z^{(q)}_{jl} \Vert_{\text{F}} +\frac{\lambda_2}{\rho}\Vert Z^{(1)}_{jl}-Z^{(2)}_{jl} \Vert_{\text{F}},
\end{equation}
where $\mathbbm{1}_{j \neq l}=1$ when $j \neq l$ and $0$ otherwise.

When $j=l$, by Lemma~\ref{lemma:LemmaOptFFGL}, we have the following closed-form updates for $\{\hat{Z}^{(1)}_{jj},\hat{Z}^{(2)}_{jj}\}$, $j=1,\ldots,p$. If $\Vert A^{(1)}_{jj}-A^{(2)}_{jj} \Vert_{\text{F}} \leq 2\lambda_2/\rho$, then 
\begin{equation*}
\hat{Z}^{(1)}_{jj}=\hat{Z}^{(2)}_{jj}=\frac{1}{2}\left(A^{(1)}_{jj}+A^{(2)}_{jj} \right).
\end{equation*}
If $\Vert A^{(1)}_{jj}-A^{(2)}_{jj} \Vert_{\text{F}} > 2\lambda_2/\rho$, then
\begin{equation*}
\begin{aligned}
\hat{Z}^{(1)}_{jj}&=A^{(1)}_{jj}-\frac{\lambda_2/\rho}{\Vert A^{(1)}_{jj}-A^{(2)}_{jj} \Vert_{\text{F}}}\left(A^{(1)}_{jj}-A^{(2)}_{jj}\right),\\
\hat{Z}^{(2)}_{jj}&=A^{(2)}_{jj}+\frac{\lambda_2/\rho}{\Vert A^{(1)}_{jj}-A^{(2)}_{jj} \Vert_{\text{F}}}\left(A^{(1)}_{jj}-A^{(2)}_{jj}\right).
\end{aligned}
\end{equation*}

For $j \neq l$, we get $\{\hat{Z}^{(1)}_{jl},\hat{Z}^{(2)}_{jl}\}$ using the ADMM algorithm again. We construct the scaled augmented Lagrangian as:
\begin{multline*}
{L^{\prime}}_{\rho^{\prime}}\left(\{W\},\{R\},\{V\} \right)
=\frac{1}{2}\sum^{2}_{q=1}\Vert W^{(q)}-B^{(q)} \Vert_{\text{F}}+\frac{\lambda_1}{\rho}\sum^{2}_{q=1}\Vert W^{(q)} \Vert_{\text{F}}\\
+\frac{\lambda_2}{\rho}\Vert R^{(1)}-R^{(2)} \Vert_{\text{F}}+\frac{\rho^{\prime}}{2}\sum^{2}_{q=1}\Vert W^{(q)}-R^{(q)}+V^{(q)} \Vert^{2}_{\text{F}},
\end{multline*}
where $\rho^{\prime}>0$ is a tuning parameter, $B^{(q)}=A^{(q)}_{jl}$, $q=1,2$, and $W^{(q)},R^{(q)},V^{(q)} \in \mathbb{R}^{M\times M}$, $q=1,2$.  $\{W\}=\{W^{(1)},W^{(2)}\}$, $\{R\}=\{R^{(1)},R^{(2)}\}$, and  $\{V\}=\{V^{(1)},V^{(2)}\}$. The detailed ADMM algorithm is described as below:

\centerline{\textbf{ADMM algorithm for solving \eqref{eq:FFGLQ2keysetp} for $j \neq l$}}

\smallskip

\textbf{Input:} $A^{(q)}_{jl}$, $q=1,2$; $\lambda_1,\lambda_2 \geq 0$.

\textbf{Output:} $\{\hat{Z}^{(1)}_{jl},\hat{Z}^{(2)}_{jl}\}$.

(a) Initialize the variables: $W^{(q)}_{(0)}=I_{M}$, $R^{(q)}_{(0)}=0_{M}$,  $V^{(q)}_{(0)}=0_{M}$, $B^{(q)}=A^{(q)}_{jl}$, $q=1,2$.

(b) Select a scalar $\rho^{\prime}>0$.

(c) For $i=1,2,3,\dots$ until convergence

\hspace*{15pt} (i) $\{W_{(i)}\}\leftarrow \argmin_{\{W\}} {L^{\prime}}_{\rho^{\prime}}\left(\{W\},\{R_{(i-1)}\},\{V_{(i-1)}\}\right)$.
\begin{adjustwidth}{1.8cm}{}
	This is equivalent to
	\begin{equation*}
	\{W_{(i)}\}\leftarrow \argmin_{\{W\}}\frac{1}{2}\sum^{2}_{q=1}\Vert W^{(q)}-C^{(q)} \Vert^{2}_{\text{F}}+\frac{\lambda_1}{\rho(1+\rho^{\prime})}\sum^{2}_{q=1}\Vert W^{(q)} \Vert_{\text{F}},
	\end{equation*}
	where
	\begin{equation*}
	C^{(q)}=\frac{1}{1+\rho^{\prime}}\left[B^{(q)}+\rho^{\prime}\left(R^{(q)}_{(i-1)}-V^{(q)}_{(i-1)}\right) \right].
	\end{equation*}
	Similar to~\eqref{con:iterorign}, we have
	\begin{equation*}
	W^{(q)}_{(i)}\leftarrow \left(\frac{\Vert C^{(q)} \Vert_{\text{F}}-\lambda_1/(\rho(1+\rho^{\prime}))}{\Vert C^{(q)} \Vert_{\text{F}}} \right)_{+}\cdot C^{(q)},\qquad q=1,2.
	\end{equation*}
\end{adjustwidth}

\hspace*{15pt} (ii) $\{R_{(i)}\}\leftarrow \argmin_{\{R\}}{L^{\prime}}_{\rho^{\prime}}\left(\{W_{(i)}\},\{R\},\{V_{(i-1)}\} \right)$.
\begin{adjustwidth}{1.8cm}{}
	This is equivalent to
	\begin{equation*}
	\{R_{(i)}\}\leftarrow\argmin_{\{R\}}\frac{1}{2}\sum^{2}_{q=1}\Vert R^{(q)}-D^{(q)} \Vert^{2}_{\text{F}}+\frac{\lambda_2}{\rho\rho^{\prime}}\Vert R^{(1)}-R^{(2)} \Vert_{\text{F}},
	\end{equation*}
	where $D^{(q)}=W^{(q)}_{(i)}+V^{(q)}_{(i-1)}$. By Lemma~\ref{lemma:LemmaOptFFGL}, if $\Vert D^{(1)}-D^{(2)} \Vert_{\text{F}}\leq 2\lambda_2/(\rho\rho^{\prime})$, then
	\begin{equation*}
	R^{(1)}_{(i)}=R^{(2)}_{(i)}\leftarrow \frac{1}{2}\left(D^{(1)}+D^{(2)}\right),
	\end{equation*}
	and if $\Vert D^{(1)}-D^{(2)} \Vert_{\text{F}}> 2\lambda_2/(\rho\rho^{\prime})$, then
	\begin{equation*}
	\begin{aligned}
	&R^{(1)}\leftarrow D^{(1)}-\frac{\lambda_2/(\rho\rho^{\prime})}{\Vert D^{(1)}-D^{(2)} \Vert_{\text{F}}}\left(D^{(1)}-D^{(2)}\right),\\
	&R^{(2)}\leftarrow D^{(2)}+\frac{\lambda_2/(\rho\rho^{\prime})}{\Vert D^{(1)}-D^{(2)} \Vert_{\text{F}}}\left(D^{(1)}-D^{(2)}\right).
	\end{aligned}
	\end{equation*}
\end{adjustwidth}

\hspace*{15pt} (iii) $V^{(q)}_{(i)} \leftarrow V^{(q)}_{(i-1)}+W^{(q)}_{(i)}-R^{(q)}_{(i)}$, $q=1,2$.

(d) Output $\{\hat{Z}^{(1)}_{jl},\hat{Z}^{(2)}_{jl}\}$ as $\{ W^{(1)}_{(i)}, W^{(2)}_{(i)} \}$ from the final round.

\subsubsection{Solution to \texorpdfstring{\eqref{eq:ADMMkeystep}}{} for FFGL2}

For FFGL2, there is also no closed-form solution. Similarly to Section~\ref{subsubsec:solFFGL}, we compute a closed-form solution for $\{\hat{Z}^{(1)}_{jj},\hat{Z}^{(2)}_{jj}\}$, $j=1,\ldots,p$, and use an ADMM algorithm to compute $\{\hat{Z}^{(1)}_{jl},\hat{Z}^{(2)}_{jl}\}$, $1 \leq j \neq l \leq p$.

For any $1 \leq j,l \leq p$, we solve: 
\begin{equation}\label{eq:FFGL2Q2keysetp}
\min_{\{Z^{(1)}_{jl},Z^{(2)}_{jl}\}}\; \frac{1}{2}\sum^{2}_{q=1} \Vert Z^{(q)}_{jl}-A^{(q)}_{jl} \Vert^{2}_{\text{F}}+\frac{\lambda_1}{\rho}\mathbbm{1}_{j \neq l} \sum^{2}_{q=1}\Vert Z^{(q)}_{jl} \Vert_{\text{F}} +\frac{\lambda_2}{\rho}\sum_{1\leq a,b\leq M}\vert Z^{(1)}_{jl,ab}-Z^{(2)}_{jl,ab} \vert,
\end{equation}
where $\mathbbm{1}_{j \neq l}=1$ when $j \neq l$ and $0$ otherwise.

By Lemma~\ref{lemma:LemmaOptFFGL}, when $j=l$ we have
\begin{align*}
& \left(\hat{Z}^{(1)}_{jj,ab},\hat{Z}^{(2)}_{jj,ab}\right)\\
&\quad =\left\{
\begin{array}{ll}
\left(A^{(1)}_{jl,ab}-\lambda_2/\rho,A^{(2)}_{jl,ab}+\lambda_2/\rho\right)& \text{if}\; A^{(1)}_{jl,ab}>A^{(2)}_{jl,ab}+2\lambda_2/\rho\\
\left(A^{(1)}_{jl,ab}+\lambda_2/\rho,A^{(2)}_{jl,ab}-\lambda_2/\rho\right)& \text{if}\; A^{(1)}_{jl,ab}<A^{(2)}_{jl,ab}-2\lambda_2/\rho\\
\left(\left(A^{(1)}_{jl,ab}+A^{(2)}_{jl,ab}\right)/2,\left(A^{(1)}_{jl,ab}+A^{(2)}_{jl,ab}\right)/2\right)& \text{if}\; \left\vert A^{(1)}_{jl,ab}-A^{(2)}_{jl,ab}\right\vert \leq 2\lambda_2/\rho
\end{array}
\right.,
\end{align*}
where the subscript denotes the $(a,b)$-th entry, $1\leq a,b \leq M$ and $j=1,\ldots,p$.

For $j\neq l$, we get $\{\hat{Z}^{(1)}_{jl},\hat{Z}^{(2)}_{jl}\}$, $1 \leq j \neq l \leq p$ using an ADMM algorithm. Let $B^{(q)}=A^{(q)}_{jl}$, $q=1,2$. We first construct the scaled augmented Lagrangian:
\begin{multline*}
{L^{\prime}}_{\rho^{\prime}}\left(\{W\},\{R\},\{V\} \right)
=\frac{1}{2}\sum^{2}_{q=1}\Vert W^{(q)}-B^{(q)} \Vert_{\text{F}}+\frac{\lambda_1}{\rho}\sum^{2}_{q=1}\Vert W^{(q)} \Vert_{\text{F}}\\
+\frac{\lambda_2}{\rho}\sum_{a,b}\vert R^{(1)}_{a,b}-R^{(2)}_{a,b}\vert+\frac{\rho^{\prime}}{2}\sum^{2}_{q=1}\Vert W^{(q)}-R^{(q)}+V^{(q)} \Vert^{2}_{\text{F}},
\end{multline*}
where $\rho^{\prime}>0$ is a tuning parameter, $W^{q},R^{(q)},V^{(q)} \in \mathbb{R}^{M\times M}$, $q=1,2$, $\{W\}=\{W^{(1)},W^{(2)}\}$, $\{R\}=\{R^{(1)},R^{(2)}\}$, and $\{V\}=\{V^{(1)},V^{(2)}\}$. The detailed ADMM algorithm is described below.

\centerline{\textbf{ADMM algorithm for solving \eqref{eq:FFGL2Q2keysetp} for $j \neq l$}}

\smallskip

\textbf{Input:} $A^{(q)}_{jl}$, $q=1,2$; $\lambda_1,\lambda_2 \geq 0$.

\textbf{Output:} $\{\hat{Z}^{(1)}_{jl},\hat{Z}^{(2)}_{jl}\}$.

(a) Initialize the variables: $W^{(q)}_{(0)}=I_{M}$, $R^{(q)}_{(0)}=0_{M}$, $V^{(q)}_{(0)}=0_{M}$, $B^{(q)}=A^{(q)}_{jl}$, $q=1,2$.

(b) Select a scalar $\rho^{\prime}>0$.

(c) For $i=1,2,3,\dots$ until convergence

\hspace*{15pt} (i) $\{W_{(i)}\}\leftarrow \argmin_{\{W\}}. {L^{\prime}}_{\rho^{\prime}}\left(\{W\},\{R_{(i-1)}\},\{V_{(i-1)}\}\right)$
\begin{adjustwidth}{1.8cm}{}
	This is equivalent to
	\begin{equation*}
	\{W_{(i)}\}\leftarrow \argmin_{\{W\}}\frac{1}{2}\sum^{2}_{q=1}\Vert W^{(q)}-C^{(q)} \Vert^{2}_{\text{F}}+\frac{\lambda_1}{\rho(1+\rho^{\prime})}\sum^{2}_{q=1}\Vert W^{(q)} \Vert_{\text{F}},
	\end{equation*}
	where
	\begin{equation*}
	C^{(q)}=\frac{1}{1+\rho^{\prime}}\left[B^{(q)}+\rho^{\prime}\left(R^{(q)}_{(i-1)}-V^{(q)}_{(i-1)}\right) \right].
	\end{equation*}
	Similarly to \eqref{con:iterorign}, we have
	\begin{equation*}
	W^{(q)}_{(i)}\leftarrow \left(\frac{\Vert C^{(q)} \Vert_{\text{F}}-\lambda_1/(\rho(1+\rho^{\prime}))}{\Vert C^{(q)} \Vert_{\text{F}}} \right)_{+}\cdot C^{(q)},\qquad  q=1,2.
	\end{equation*}
\end{adjustwidth}

\hspace*{15pt} (ii) $\{R_{(i)}\}\leftarrow \argmin_{\{R\}}{L^{\prime}}_{\rho^{\prime}}\left(\{W_{(i)}\},\{R\},\{V_{(i-1)}\} \right)$
\begin{adjustwidth}{1.8cm}{}
	This is equivalent to
	\begin{equation*}
	\{R_{(i)}\}\leftarrow\argmin_{\{R\}}\frac{1}{2}\sum^{2}_{q=1}\Vert R^{(q)}-D^{(q)} \Vert^{2}_{\text{F}}+\frac{\lambda_2}{\rho\rho^{\prime}}\sum_{a,b}\left\vert R^{(1)}_{ab}-R^{(2)}_{ab}\right\vert,
	\end{equation*}
	where $D^{(q)}=W^{(q)}_{(i)}+V^{(q)}_{(i-1)}$. Then, by Lemma~\ref{lemma:LemmaOptFFGL}, we have
	\begin{align*}
	& \left(R^{(1)}_{(i),ab},R^{(2)}_{(i),ab}\right)\\
	&\quad =\left\{
	\begin{array}{ll}
	\left(D^{(1)}_{ab}-\lambda_2/(\rho\rho^{\prime}),D^{(2)}_{ab}+\lambda_2/(\rho\rho^{\prime})\right) & \text{if}\; D^{(1)}_{ab}>D^{(2)}_{ab}+2\lambda_2/(\rho\rho^{\prime})\\
	\left(D^{(1)}_{ab}+\lambda_2/(\rho\rho^{\prime}),D^{(2)}_{ab}-\lambda_2/(\rho\rho^{\prime})\right) & \text{if}\; D^{(1)}_{ab}<D^{(2)}_{ab}-2\lambda_2/(\rho\rho^{\prime})\\
	\left(\left(D^{(1)}_{ab}+D^{(2)}_{ab}\right)/2,\left(D^{(1)}_{ab}+D^{(2)}_{ab}\right)/2\right) & \text{if}\; \left\vert D^{(1)}_{ab}-D^{(1)}_{ab}\right\vert \leq 2\lambda_2/(\rho\rho^{\prime})
	\end{array}
	\right.,
	\end{align*}
	where the subscript denotes the $(a,b)$-th entry, $1\leq a,b \leq M$ and $1 \leq j,l \leq p$.
\end{adjustwidth}

\hspace*{15pt} (iii) $V^{(q)}_{(i)} \leftarrow V^{(q)}_{(i-1)}+W^{(q)}_{(i)}-R^{(q)}_{(i)}$, $q=1,2$.

(d) Output $\{\hat{Z}^{(1)}_{jl},\hat{Z}^{(2)}_{jl}\}$ as $\{ W^{(1)}_{(i)}, W^{(2)}_{(i)} \}$ from the final round.

\subsection{Derivation of \texorpdfstring{\eqref{eq:GFGLkeysetp1}}{} and \texorpdfstring{\eqref{eq:GFGLkeysetp2}}{} }\label{subsec:OptGFGL}

Note that for any $1 \leq j,l \leq p$, we can obtain $\hat{Z}^{(1)}_{jl},\hat{Z}^{(2)}_{jl},\dots,\hat{Z}^{(Q)}_{jl}$ by solving
\begin{equation}\label{eq:GFGLdrivsplit}
	\argmin_{Z^{(1)}_{jl},Z^{(2)}_{jl},\dots,Z^{(Q)}_{jl}} \frac{\rho}{2}\sum^{Q}_{q=1}\Vert Z^{(q)}_{jl}-A^{(q)}_{jl} \Vert^{2}_{\text{F}}+\lambda_1\mathbbm{1}_{j\neq l}\sum^{Q}_{q=1}\Vert Z^{(q)}_{jl} \Vert_{\text{F}}+\lambda_2\mathbbm{1}_{j\neq l}\left( \sum^{Q}_{q=1}\Vert Z^{(q)}_{jl} \Vert^{2}_{\text{F}} \right)^{1/2},
\end{equation}
where $\mathbbm{1}_{j \neq l}=1$ when $j\neq l$ and $0$ otherwise. By \eqref{eq:GFGLdrivsplit}, we have $\hat{Z}^{(q)}_{jj}=A^{(q)}_{jj}$ for any $j=1,\ldots,p$ and $q=1,\ldots,Q$, which is \eqref{eq:GFGLkeysetp1}. We then prove \eqref{eq:GFGLkeysetp2}. Denote the objective function in \eqref{eq:GFGLdrivsplit} by $\tilde{L}_{jl}$. Then, for $j \neq l$, the subdifferential of $\tilde{L}_{jl}$ with respect to $Z^{(q)}_{jl}$ is
\begin{equation*}
	\partial_{Z^{(q)}_{jl}}\tilde{L}_{jl}=\rho(Z^{(q)}_{jl}-A^{(q)}_{jl})+\lambda_1 G^{(q)}_{jl}+\lambda_2 D^{(q)}_{jl},
\end{equation*}
where
\begin{equation*}
	G^{(q)}_{jl}=\left\{
	\begin{array}{ll}
	\frac{Z^{(q)}_{jl}}{\Vert Z^{(q)}_{jl} \Vert_{\text{F}}}& \text{when}\; Z^{(q)}_{jl}\neq 0\\
	\{G^{(q)}_{jl} \in \mathbb{R}^{M\times M}: \Vert G^{(q)}_{jl} \Vert_{\text{F}} \leq 1  \}& \text{otherwise}
	\end{array}
	\right.,
\end{equation*}
and
\begin{equation*}
	D^{(q)}_{jl}=\left\{
	\begin{array}{ll}
	\frac{Z^{(q)}_{jl}}{\left( \sum^{Q}_{q=1}\Vert Z^{(q)}_{jl} \Vert^{2}_{\text{F}} \right)^{1/2}}& \text{when}\; \sum^{Q}_{q=1}\Vert Z^{(q)}_{jl} \Vert^{2}_{\text{F}}>0\\
	\{D^{(q)}_{jl}\in \mathbb{R}^{M \times M}:\sum^{Q}_{q=1}\Vert D^{(q)}_{jl} \Vert^{2}_{\text{F}}\leq 1  \} & \text{otherwise}
	\end{array}
	\right..
\end{equation*}
To obtain the optimum, we need
\begin{equation*}
	0 \in \partial_{Z^{(q)}_{jl}}\tilde{L}_{jl} (\hat{Z}^{(q)}_{jl})
\end{equation*}
for all $q=1,\ldots,Q$. Now we split our discussion into two cases.
	
	(a) Suppose $\sum^{Q}_{q=1}\Vert \hat{Z}^{(q)}_{jl} \Vert^{2}_{\text{F}}=0$ or equivalently $\hat{Z}^{(q)}_{jl}=0$ for all $q=1,\ldots,Q$.
	
	In this case, there exist $G^{(q)}_{jl}$, where $\Vert G^{(q)}_{jl} \Vert_{\text{F}}\leq 1$, $q=1,\ldots,Q$; and also $D^{(q)}_{jl}$, where $\sum^{Q}_{q=1} \Vert D^{(q)}_{jl} \Vert^{2}_{\text{F}} \leq 1$, such that
	\begin{equation*}
	0=-\rho\cdot A^{(q)}_{jl}+\lambda_1 G^{(q)}_{jl}+\lambda_2 D^{(q)}_{jl}.
	\end{equation*}
	This implies that
	\begin{equation*}
	D^{(q)}_{jl}=\frac{\rho}{\lambda_2}\left(A^{(q)}_{jl}-\frac{\lambda_1}{\rho}G^{(q)}_{jl}  \right).
	\end{equation*}
	Thus, we have
	\begin{equation*}
	\begin{aligned}
	\Vert D^{(q)}_{jl} \Vert_{\text{F}}&=\frac{\rho}{\lambda_2}\left\Vert A^{(q)}_{jl}-\frac{\lambda_1}{\rho}G^{(q)}_{jl} \right\Vert_{\text{F}}\geq \frac{\rho}{\lambda_2}\left( \Vert A^{(q)}_{jl} \Vert_{\text{F}} -\frac{\lambda_1}{\rho}\Vert G^{(q)}_{jl}\Vert_{\text{F}} \right)_{+}\\
	&\geq \frac{\rho}{\lambda_2}\left( \Vert A^{(q)}_{jl} \Vert_{\text{F}} -\frac{\lambda_1}{\rho}\right)_{+},
	\end{aligned}
	\end{equation*}
	which implies that
	\begin{equation*}
	\frac{\rho^2}{\lambda^2_2}\sum^{Q}_{q=1}\left( \Vert A^{(q)}_{jl} \Vert_{\text{F}} -\frac{\lambda_1}{\rho}\right)^{2}_{+}\leq \sum^{Q}_{q=1}\Vert D^{(q)}_{jl} \Vert^{2}_{\text{F}}\leq 1.
	\end{equation*}
	Therefore, 
	\begin{equation}\label{eq:GFGLderiv3}
	\sqrt{\sum^{Q}_{q=1}\left( \Vert A^{(q)}_{jl} \Vert_{\text{F}} -\lambda_1/\rho\right)^{2}_{+}}\leq \lambda_2/\rho.
	\end{equation}
	
	(b) Suppose $\sum^{Q}_{q=1}\Vert \hat{Z}^{(q)}_{jl} \Vert^{2}_{\text{F}}>0$.
	
	For those $q$'s such that $\hat{Z}^{(q)}_{jl}=0$, there exists $G^{(q)}_{jl}$, where $\Vert G^{(q)}_{jl} \Vert_{\text{F}}=1$, such that
	\begin{equation*}
	0=-\rho A^{(q)}_{jl}+\lambda_1 G^{(q)}_{jl}.
	\end{equation*}
	Thus, we have
	\begin{equation*}
	\Vert A^{(q)}_{jl} \Vert_{\text{F}}=\frac{\lambda_1}{\rho}\Vert G^{(q)}_{jl} \Vert_{\text{F}} \leq \frac{\lambda_1}{\rho},
	\end{equation*}
	which implies that
	\begin{equation}\label{eq:GFGLderiv4}
	\left(\Vert A^{(q)}_{jl} \Vert_{\text{F}}-\lambda_1/\rho  \right)_{+}=0.
	\end{equation}
	
	On the other hand, for those $q$'s such that $\hat{Z}^{(q)}_{jl}\neq 0$, we have
	\begin{equation*}
	0=\rho\left( \hat{Z}^{(q)}_{jl}-A^{(q)}_{jl} \right)+\lambda_1 \frac{\hat{Z}^{(q)}_{jl}}{\Vert \hat{Z}^{(q)}_{jl} \Vert_{\text{F}}}+\lambda_2\frac{\hat{Z}^{(q)}_{jl}}{\left(\sum^{Q}_{q=1}\Vert \hat{Z}^{(q)}_{jl} \Vert^{2}_{\text{F}}\right)^{1/2}},
	\end{equation*}
	which implies that
	\begin{equation}\label{eq:GFGLderiv5}
	A^{(q)}_{jl}=\hat{Z}^{(q)}_{jl}\left(1+\frac{\lambda_1}{\rho\Vert \hat{Z}^{(q)}_{jl} \Vert_{\text{F}}}+\frac{\lambda_2}{\rho\left(\sum^{Q}_{q=1}\Vert \hat{Z}^{(q)}_{jl} \Vert^{2}_{\text{F}}\right)^{1/2}}\right),
	\end{equation}
	and
	\begin{equation}\label{eq:GFGLderiv6}
	\Vert A^{(q)}_{jl} \Vert_{\text{F}}=\Vert \hat{Z}^{(q)}_{jl} \Vert_{\text{F}}+\lambda_1/\rho+(\lambda_2/\rho)\cdot\frac{\Vert \hat{Z}^{(q)}_{jl} \Vert_{\text{F}}}{\left(\sum^{Q}_{q=1}\Vert \hat{Z}^{(q)}_{jl} \Vert^{2}_{\text{F}}\right)^{1/2}}.
	\end{equation}
	By \eqref{eq:GFGLderiv6}, we have
	\begin{equation}\label{eq:GFGLderiv7}
	\left(\Vert A^{(q)}_{jl} \Vert_{\text{F}}-\lambda_1/\rho  \right)_{+}>\frac{\lambda_2}{\rho}\cdot\frac{\Vert \hat{Z}^{(q)}_{jl} \Vert_{\text{F}}}{\sqrt{\sum^{Q}_{q=1} \Vert \hat{Z}^{(q)}_{jl} \Vert^{2}_{\text{F}} }}>0.
	\end{equation}
	By \eqref{eq:GFGLderiv4} and \eqref{eq:GFGLderiv7}, we have
	\begin{equation}\label{eq:GFGLderiv8}
	\begin{aligned}
	\sum^{Q}_{q=}\left(\Vert A^{(q)}_{jl} \Vert_{\text{F}}-\lambda_1/\rho  \right)^{2}_{+}&=\sum_{q:\Vert \hat{Z}^{(q)}_{jl} \Vert_{\text{F}}\neq 0}\left(\Vert A^{(q)}_{jl} \Vert_{\text{F}}-\lambda_1/\rho  \right)^{2}_{+}\\
	&>\frac{\lambda^{2}_2}{\rho^2}\sum_{q:\Vert \hat{Z}^{(q)}_{jl} \Vert_{\text{F}}\neq 0} \frac{\Vert \hat{Z}^{(q)}_{jl} \Vert^{2}_{\text{F}}}{\sum^{Q}_{q=1} \Vert \hat{Z}^{(q)}_{jl} \Vert^{2}_{\text{F}} }\\
	&> \lambda^2_2/\rho^2.
	\end{aligned}
	\end{equation}
	
	Now we make the following claims.
	
	\textbf{Claim 1.} $\sum^{Q}_{q=1} \Vert \hat{Z}^{(q)}_{jl} \Vert^{2}_{\text{F}}=0\Longleftrightarrow\sqrt{\sum^{Q}_{q=}\left(\Vert A^{(q)}_{jl} \Vert_{\text{F}}-\lambda_1/\rho  \right)^{2}_{+}}\leq \lambda_2/\rho$.
	
	This claim is easily shown by \eqref{eq:GFGLderiv3} and \eqref{eq:GFGLderiv8}.
	
	\textbf{Claim 2.} When $\sum^{Q}_{q=1} \Vert \hat{Z}^{(q)}_{jl} \Vert^{2}_{\text{F}}>0$, we have $\Vert \hat{Z}^{(q)}_{jl} \Vert_{\text{F}}=0\Longleftrightarrow \Vert A^{(q)}_{jl} \Vert_{\text{F}}\leq \lambda_1/\rho$.
	
	This claim is easily shown by \eqref{eq:GFGLderiv4} and \eqref{eq:GFGLderiv7}.
	
	\textbf{Claim 3.} When $\Vert \hat{Z}^{(q)}_{jl} \Vert_{\text{F}}\neq 0$, then we have
	\begin{equation*}
	\hat{Z}^{(q)}_{jl}=\left(\frac{\Vert A^{(q)}_{jl} \Vert_{\text{F}}-\lambda_1/\rho}{\Vert A^{(q)}_{jl} \Vert_{\text{F}}} \right)\left(1-\frac{\lambda_2}{\rho\sqrt{\sum^{Q}_{q=}\left(\Vert A^{(q)}_{jl} \Vert_{\text{F}}-\lambda_1/\rho  \right)^{2}_{+}}}\right)A^{(q)}_{jl}.
	\end{equation*}
	
	To prove this claim, note that by Claim 2 and \eqref{eq:GFGLderiv6}, we have
	\begin{equation*}
	\left(\Vert A^{(q)}_{jl} \Vert_{\text{F}}-\lambda_1/\rho  \right)_{+}=\Vert \hat{Z}^{(q)}_{jl} \Vert_{\text{F}}\left(1+\frac{\lambda_2}{\rho \left(\sum^{Q}_{q=1} \Vert \hat{Z}^{(q)}_{jl} \Vert^{2}_{\text{F}}\right)^{1/2}}\right),
	\qquad
	q=1,\ldots,Q.
	\end{equation*}
	Thus,
	\begin{equation*}
	\sqrt{\sum^{Q}_{q=1}\left(\Vert A^{(q)}_{jl} \Vert_{\text{F}}-\lambda_1/\rho  \right)^{2}_{+}}=\sqrt{\sum^{Q}_{q=1}\Vert \hat{Z}^{(q)}_{jl} \Vert^{2}_{\text{F}}}+\lambda_2/\rho,
	\end{equation*}
	which implies that
	\begin{equation*}
	\sqrt{\sum^{Q}_{q=1}\Vert \hat{Z}^{(q)}_{jl} \Vert^{2}_{\text{F}}}=\sqrt{\sum^{Q}_{q=1}\left(\Vert A^{(q)}_{jl} \Vert_{\text{F}}-\lambda_1/\rho  \right)^{2}_{+}}-\lambda_2/\rho.
	\end{equation*}
	Thus, by \eqref{eq:GFGLderiv6}, we have
	\begin{equation*}
	\begin{aligned}
	\Vert \hat{Z}^{(q)}_{jl} \Vert_{\text{F}}&=\frac{\Vert A^{(q)}_{jl} \Vert_{\text{F}}-\lambda_1/\rho }{1+\frac{\lambda_2/\rho}{\sqrt{\sum^{Q}_{q^{\prime}=1}\left(\Vert A^{(q^{\prime})}_{jl} \Vert_{\text{F}}-\lambda_1/\rho  \right)^{2}_{+}}-\lambda_2/\rho }}\\
	&=\left(1-\frac{\lambda_2}{\rho\sqrt{\sum^{Q}_{q^{\prime}=1}\left(\Vert A^{(q^{\prime})}_{jl} \Vert_{\text{F}}-\lambda_1/\rho  \right)^{2}_{+}} }\right)\left(\Vert A^{(q)}_{jl} \Vert_{\text{F}}-\lambda_1/\rho\right).
	\end{aligned}
	\end{equation*}
	 Claim 3 follows by combining the above display with \eqref{eq:GFGLderiv5}.
	
Finally, combining Claims 1-3, we obtain \eqref{eq:GFGLkeysetp2}.

\newpage

\section{Main Technical Proofs}

We give proofs of the results given in the main text.

\subsection{Proof of Lemma~\ref{lemma:Delta-indp-funB}}
\label{sec:proof-lemma-delta-indp-funB}

We only need to prove that when we use two sets of orthonormal function basis $e^M(t)=\{ e^M_j (t) \}^p_{j=1}$ and $\tilde{e}^M(t)=\{ \tilde{e}^M_j (t) \}^p_{j=1}$ to expand the same subspace $\mathbb{V}^M_{[p]}$, the definition of $E^{\pi}_{\Delta}$ will not change. Since both $e^M_j(t)=( e^M_{j1}(t),e^M_{j2}(t),\dots,e^M_{jM}(t) )^{\top}$ and $\tilde{e}^M_j(t)=(\tilde{e}^M_{j1}(t),\tilde{e}^M_{j2}(t),\dots,\tilde{e}^M_{jM}(t) )^{\top}$ are orthonormal function basis of $\mathbb{V}^M_j$, there must exist an orthonormal matrix $U_j \in \mathbb{R}^{M \times M}$ satisfying $U^{\top}_j U_j = U_j U^{\top}_j=I_M$, such that $\tilde{e}^M_j(t)=U_j e^M_j(t)$. Let $a^{X,M}_{ij}$ be the projection score vectors of $X_{ij}(t)$ onto $e^M_j(t)$ and $\tilde{a}^{X,M}_{ij}$ be the projection score vectors of $X_{ij}(t)$ onto $\tilde{e}^M_j(t)$. Then $\tilde{a}^{X,M}_{ij}=U_j a^{X,M}_{ij}$. Denote 
\[U={\rm diag} \{ U_1,U_2,\dots,U_p \} \in \mathbb{R}^{pM \times pM}.\]
We then have 
\begin{align*}
\tilde{a}^{X,M}_i &= ( (\tilde{a}^{X,M}_{i1})^{\top},(\tilde{a}^{X,M}_{i2})^{\top},\dots,(\tilde{a}^{X,M}_{ip})^{\top} )^{\top} \\
&= ( (a^{X,M}_{i1})^{\top} U^{\top}_1, (a^{X,M}_{i2})^{\top} U^{\top}_2,\dots,  (a^{X,M}_{ip})^{\top} U^{\top}_p )^{\top}=U a^{X,M}_i    
\end{align*}
and
\begin{equation*}
\tilde{\Sigma}^{X,M} = {\rm Cov} \left( \tilde{a}^{X,M} \right) = U {\rm Cov} \left( \tilde{a}^{X,M} \right) U^{\top} = U \Sigma^{X,M} U^{\top}.
\end{equation*}
Thus
\begin{equation*}
\tilde{\Theta}^{X,M} = \left( \tilde{\Sigma}^{X,M} \right)^{-1} = U \left( \Sigma^{X,M} \right)^{-1} U^{\top} = U \Theta^{X,M} U^{\top}.
\end{equation*}
Therefore, $\tilde{\Theta}^{X,M}_{jl} = U_j \Theta^{X,M}_{jl} U^{\top}_l$ for all $j,l \in V^2$ and, therefore, $\Vert \tilde{\Theta}^{X,M}_{jl} \Vert_{\text{F}}=\Vert \Theta^{X,M}_{jl} \Vert_{\text{F}}$ for all $j,l \in V^2$. This implies the final result.

\subsection{Proof of Lemma~\ref{lemma:M-dim-subspace-find}}
\label{sec:proof-lemma-M-dim-subspace-find}

We first show that $X_{ij}, Y_{ij} \in {\rm Span} \left\{ \phi_{j1},\dots,\phi_{jM^{\star}_j}  \right\}$ almost surely.
Let 
\[
M^X_j = \sup\{ M \in \mathbb{N}^{+}: \lambda^X_{jM } > 0 \}.
\]
By Karhunen–Loève theorem, we have $X_{ij}=\sum^{M^X_j}_{k=1} \langle X_{ij},\phi^X_{jk} \rangle \phi^X_{jk}$ almost surely. Thus, we have $X_{ij} \in {\rm Span} \left\{ \phi^X_{j1},\dots,\phi^X_{j,M^X_j}  \right\}$ almost surely. For any $1 \leq k \leq M^X_j$, we have that
\begin{equation*}
\int_{\mathcal{T}} K_{jj}(s,t) \phi^X_k(s) \phi^X_k(t) ds dt \geq \int_{\mathcal{T}} K^X_{jj}(s,t) \phi^X_k(s) \phi^X_k(t) ds dt = \lambda^X_{jk}>0,
\end{equation*}
which implies that $\phi^X_k \in {\rm Span} \left\{ \phi_{j1},\dots,\phi_{jM^{\star}_j}  \right\}$. Thus, we have ${\rm Span} \left\{ \phi^X_{j1},\dots,\phi^X_{j,M^X_j}  \right\} \subseteq {\rm Span} \left\{ \phi_{j1},\dots,\phi_{jM^{\star}_j}  \right\}$ and $X_{ij} \in {\rm Span} \left\{ \phi_{j1},\dots,\phi_{jM^{\star}_j}  \right\}$ almost surely. Similarly, we have that $Y_{ij} \in  {\rm Span} \left\{ \phi_{j1},\dots,\phi_{jM^{\star}_j}  \right\}$ almost surely.

Next, we show that $M^{\prime}_j=M^{\star}_j$ by contradiction. By the definition of $M^{\prime}_j$, we have that $M^{\prime}_j \leq M^{\star}_j$. If $M^{\prime}_j \neq M^{\star}_j$, then we have $\mathbb{V}^{M^{\prime}_j}_j \subseteq \mathbb{H}$ such that $M^{\prime}_j<M^{\star}_j$ and $X_{ij},Y_{ij} \in \mathbb{V}^{M^{\prime}_j}_j$ almost surely. This implies that there exists $\phi \in {\rm Span} \left\{ \phi_{j1},\dots,\phi_{jM^{\star}_j}  \right\} \setminus \mathbb{V}^{M^{\prime}_j}_j$ such that
\begin{align*}
& \mathbb{E} \left[ \left( \langle \phi_{jk}(t), X_{ij}(t) \rangle \right)^2 \right] = 0 \quad \text{and} \quad \mathbb{E} \left[ \left( \langle \phi_{jk}(t), Y_{ij}(t) \rangle \right)^2 \right] = 0 \\
\Rightarrow & \int_{\mathcal{T}} K^X_{jj}(s,t) \phi_{jk}(s) \phi_{jk}(t) ds dt = 0 \quad \text{and} \quad \int_{\mathcal{T}} K^Y_{jj}(s,t) \phi_{jk}(s) \phi_{jk}(t) ds dt = 0 \\
\Rightarrow & \int_{\mathcal{T}} K_{jj}(s,t) \phi_{jk}(s) \phi_{jk}(t) ds dt = 0, \\
\Rightarrow & \lambda_{jk} = 0,
\end{align*}
which contradicts the definition of $M^{\star}_j$. Thus, we must have $M^{\prime}_j=M^{\star}_j$.

\subsection{Proof of Lemma \ref{lemma:lemma1}}
\label{sec:proof-lemma-lemma1}

Let $U=V\backslash\{j,l\}$, and $a^{X,M}_{U}=\left((a^{X,M}_{j})^{\top},j \in U\right)^{\top}$. Without loss of generality, assume that $\Sigma^{X,M}$ and $\Theta^{X,M}$ take the following block structure:
\begin{equation*}
\begin{aligned}
\Sigma^{X,M}=\left[
\begin{matrix}
\Sigma^{X,M}_{jj} & \Sigma^{X,M}_{jl} & \Sigma^{X,M}_{jU} \\
\Sigma^{X,M}_{lj} & \Sigma^{X,M}_{ll} & \Sigma^{X,M}_{lU} \\
\Sigma^{X,M}_{Uj} & \Sigma^{X,M}_{Ul} & \Sigma^{X,M}_{UU} \\
\end{matrix}
\right],\quad
\Theta^{X,M}=\left[
\begin{matrix}
\Theta^{X,M}_{jj} & \Theta^{X,M}_{jl} & \Theta^{X,M}_{jU} \\
\Theta^{X,M}_{lj} & \Theta^{X,M}_{ll} & \Theta^{X,M}_{lU} \\
\Theta^{X,M}_{Uj} & \Theta^{X,M}_{Ul} & \Theta^{X,M}_{UU} \\
\end{matrix}
\right].
\end{aligned}
\end{equation*}
Let $P$ denote the submatrix:
\begin{equation*}
P=\left[
\begin{matrix}
\Theta^{X,M}_{jj} & \Theta^{X,M}_{jl} \\
\Theta^{X,M}_{lj} & \Theta^{X,M}_{ll}
\end{matrix}
\right].
\end{equation*}
By standard results for the multivariate Gaussian \citep{johnson2014applied}, we have
\begin{equation*}
\begin{aligned}
&\Var\left(a^{X,M}_{j} \mid a^{X,M}_{k},k \neq j\right)=H^{X,M}_{jj}=(\Theta^{X,M}_{jj})^{-1},\\
&\Var\left(\left[
\begin{matrix}
a^{X,M}_{j}\\
a^{X,M}_{l}
\end{matrix}
\right]
\mid a^{X,M}_{U} \right)= P^{-1}=
\left[
\begin{matrix}
(P^{-1})_{11} & (P^{-1})_{12} \\
(P^{-1})_{21} & (P^{-1})_{22}
\end{matrix}
\right].
\end{aligned}
\end{equation*}
Thus, the first statement directly follows from the first equation. To prove the second statement, we only need to note that
\begin{equation*}
\begin{aligned}
H^{X,M}_{jl} &= \Cov \left(a^{X,M}_{j}, a^{X,M}_{l} \mid a^{X,M}_{U}\right)\\
&=(P^{-1})_{12}\\
&=-(\Theta^{X,M}_{jj})^{-1}\Theta^{X,M}_{jl}(P^{-1})_{22}\\
&=-H^{X,M}_{jj}\Theta_{jl}^{X,M}H^{\backslash j,X,M}_{ll},
\end{aligned}
\end{equation*}
where the second to last equation follows from the $2\times 2$ block matrix inverse and the last equation follows from the property of multivariate Gaussian. This completes the proof.

\subsection{Proof of Theorem~\ref{Thm:smallboundThm}}

We provide the proof of Theorem~\ref{Thm:smallboundThm},
following the framework introduced in \cite{negahban2010unified}. We start by introducing some  notation.

We use $\otimes$ to denote the Kronecker product. For $\Delta \in \mathbb{R}^{pM \times pM}$, let $\theta=\vect(\Delta)\in{\mathbb{R}^{p^{2}M^{2}}}$ and $\theta^{*}=\vect({\Delta^{M}})$, where $\Delta^M$ is defined in Section 2.2. Let $\mathcal{G}=\{G_{t}\}_{t=1, \ldots, N_{\mathcal{G}}}$ be a set of indices, where $N_{\mathcal{G}}=p^{2}$ and $G_t \subset \{1,2,\cdots,p^2M^2\}$ is the set of indices for $\theta$ that correspond to the $t$-th $M\times M$ submatrix of $\Delta^M$. Thus, if $t=(j-1)p+l$, then $\theta_{G_{t}}=\vect{(\Delta_{jl})}\in{\mathbb{R}^{M^{2}}}$, where $\Delta_{jl}$ is the $(j,l)$-th $M\times{M}$ submatrix of $\Delta$. Denote the group indices of $\theta^{*}$ that belong to blocks corresponding to $E_\Delta$ as $S_{\mathcal{G}}\subseteq{\{1,2,\cdots,N_{\mathcal{G}}\}}$. Note that we define $S_\mathcal{G}$ using $E_\Delta$ and not $E_{\Delta^M}$. Therefore, as stated in Assumption~\ref{assump:SparseAssump}, $|S_{\mathcal{G}}| = s$. We further define the subspace $\mathcal{M}$ as
\begin{equation}{\label{con:Msubspace}}
\mathcal{M}\coloneqq{\{\theta\in{\mathbb{R}^{p^{2}M^{2}}} \mid \theta_{G_{t}}=0\text{ for all }t\notin{S_{\mathcal{G}}}\}}.
\end{equation}
Its orthogonal complement with respect to the Euclidean inner product is
\begin{equation}
\mathcal{M}^{\bot}\coloneqq{\{\theta\in{\mathbb{R}^{p^{2}M^{2}}} \mid \theta_{G_{t}}=0\text{ for all }t\in{S_{\mathcal{G}}}\}}.
\end{equation}
For a vector $\theta$, let $\theta_{\mathcal{M}}$ and $\theta_{\mathcal{M}^{\bot}}$ be the projection of $\theta$ on the subspaces $\mathcal{M}$ and $\mathcal{M}^{\bot}$, respectively. Let $\langle\cdot,\cdot\rangle$ represent the Euclidean inner product. Let
\begin{equation}{\label{con:Rnorm}}
\mathcal{R}(\theta)\coloneqq{\sum_{t=1}^{N_{\mathcal{G}}}\vert\theta_{G_{t}}\vert_{2}}\triangleq{\vert\theta\vert_{1,2}}.
\end{equation}
For any $v\in{\mathbb{R}^{p^{2}M^{2}}}$, the dual norm of $\mathcal{R}$ is given by
\begin{equation}{\label{con:Rdualnorm}}
\mathcal{R}^{*}(v)\coloneqq\sup_{u\in{\mathbb{R}^{p^{2}M^{2}}\backslash{\{0\}}}}\frac{\langle{u},{v}\rangle}{\mathcal{R}(u)}=\sup_{\mathcal{R}(u)\leq{1}}\langle{u},{v}\rangle.
\end{equation}
The subspace compatibility constant of $\mathcal{M}$ with respect to $\mathcal{R}$ is defined as
\begin{equation}{\label{con:SubspaceCompConst}}
\Psi(\mathcal{M})\coloneqq{\sup_{u\in{\mathcal{M}\backslash\{0\}}}}\frac{\mathcal{R}(u)}{\vert u\vert_{2}}.
\end{equation}

\begin{proof}
By Lemma~\ref{lema:S-ind-joint} and Assumption~\ref{assump:req-SXSY}, we have
\begin{equation}\label{con:KroneckerBoundSup}
\vert (S^{Y,M} \otimes{S^{X,M}})-(\Sigma^{Y,M}\otimes{\Sigma^{X,M}})\vert_{\infty}\leq \delta_n^2 + 2\delta_n\sigma_{\max}
\end{equation}
and
\begin{equation}\label{con:VectBoundSup}
\vert \vect{(S^{Y,M}-S^{X,M})}-\vect{(\Sigma^{Y,M}-\Sigma^{X,M})} \vert_{\infty}\leq 2\delta_n.
\end{equation}
The problem \eqref{eq:objectivefunc} can be written in the following form:
\begin{equation}
\hat{\theta}_{\lambda_{n}}\in\argmin_{\theta\in{\mathbb{R}^{p^{2}M^{2}}}}\mathcal{L}(\theta)+\lambda_{n}\mathcal{R}(\theta),
\end{equation}
where
\begin{equation}{\label{con:lossdef}}
\mathcal{L}(\theta)=\frac{1}{2}\theta^{\top}(S^{Y,M}\otimes{S^{X,M}})\theta-\theta^{\top}\vect({S^{Y,M}-S^{X,M}}).
\end{equation}
Here, we slightly abuse the notation and use $\mathcal{L}(\cdot)$ to denote the function of $\theta$ rather than $\Delta$.
The loss $\mathcal{L}(\theta)$ is convex and differentiable with respect to $\theta$, and it can easily be verified that $\mathcal{R}(\cdot)$ defines a vector norm. For $h \in \mathbb{R}^{p^2M^2}$, the error of the first-order Taylor series expansion of $\mathcal{L}$ is:
\begin{equation}{\label{con:errorTaylor}}
\begin{aligned}
\delta{\mathcal{L}}(h,\theta^{*})\coloneqq\mathcal{L}(\theta^{*} + h)-\mathcal{L}(\theta^{*})-\langle\nabla\mathcal{L}(\theta^{*}), h  \rangle
=\frac{1}{2}h^{\top}(S^{Y,M}\otimes{S^{X,M}}) h.
\end{aligned}
\end{equation}
From~\eqref{con:lossdef}, we see that $\nabla{\mathcal{L}}(\theta)=(S^{Y,M}\otimes{S^{X,M}})\theta-\vect({S^{Y,M}-S^{X,M}})$. By Lemma \ref{lemma:dualnorm}, we have
\begin{equation}{\label{con:RnormonL}}
\mathcal{R}^{*}(\nabla{\mathcal{L}}(\theta^{*}))=\max_{t=1,2,\cdots,N_{\mathcal{G}}} \left\vert \left[(S^{Y,M}\otimes{S^{X,M}})\theta^{*} - \vect({S^{Y,M}-S^{X,M}})\right]_{G_{t}}\right\vert_{2}.
\end{equation}

Now we establish an upper bound for $\mathcal{R}^{*}(\nabla{\mathcal{L}}(\theta^{*}))$. First, note that 
\[
(\Sigma^{Y,M}\otimes{\Sigma^{X,M}})\theta^{*}-\vect({\Sigma^{Y,M}-\Sigma^{X,M}})=\vect({\Sigma^{X,M}\Delta^{M}\Sigma^{Y,M}-(\Sigma^{Y,M}-\Sigma^{X,M})})=0.
\] 
Letting $(\cdot)_{jl}$ denote the $(j,l)$-th submatrix, we have
\begin{equation}{\label{con:transeq0}}
\begin{aligned}
&\left\vert \left[(S^{Y,M}\otimes{S^{X,M}})\theta^{*} - \vect({S^{Y,M}-S^{X,M}})\right]_{G_{t}}\right\vert_{2}\\
&=\left\vert \left[(S^{Y,M}\otimes{S^{X,M}}-\Sigma^{Y,M}\otimes{\Sigma^{X,M}})\theta^{*}-\vect{((S^{Y,M}-\Sigma^{Y,M})-(S^{X,M}-\Sigma^{X,M}))}\right]_{G_{t}}\right\vert_{2}\\
&={\Vert(S^{X,M}\Delta^{M}S^{Y,M}-\Sigma^{X,M}\Delta^{M}\Sigma^{Y,M})_{jl}-(S^{Y,M}-\Sigma^{Y,M})_{jl}-(S^{X,M}-\Sigma^{X,M})_{jl}\Vert_{F}}\\
&\leq{\|(S^{X,M}\Delta^{M}S^{Y,M}-\Sigma^{X,M}\Delta^{M}\Sigma^{Y,M})_{jl}\|_{F}+\|(S^{Y,M}-\Sigma^{Y,M})_{jl}\|_{F}+\|(S^{X,M}-\Sigma^{X,M})_{jl}\|_{F}}.
\end{aligned}
\end{equation}
For any $M\times{M}$ matrix $A$, $\|A\|_{F}\leq{M|A|_{\infty}}$, so
\begin{equation*}
\begin{aligned}
&\left\vert \left[(S^{Y,M}\otimes{S^{X,M}})\theta^{*}-\vect( {S^{Y,M}-S^{X,M}})\right]_{G_{t}} \right\vert_{2}\\
&\leq M\left[\left\vert (S^{X,M}\Delta^{M}S^{Y,M}-\Sigma^{X,M}\Delta^{M}\Sigma^{Y,M})_{jl}\right \vert_{\infty} \right. \\
& \quad \quad \left.+ \left|(S^{Y,M}-\Sigma^{Y,M})_{jl}\right|_{\infty} + \left|(S^{X,M}-\Sigma^{X,M})_{jl}\right|_{\infty}\right]\\
&\leq M\left[ \left\vert S^{X,M}\Delta^{M}S^{Y,M}-\Sigma^{X,M}\Delta^{M}\Sigma^{Y,M}\right\vert _{\infty} + \vert S^{Y,M}-\Sigma^{Y,M}\vert _{\infty} + \vert S^{X,M}-\Sigma^{X,M}\vert _{\infty} \right].
\end{aligned}
\end{equation*}

For any $A\in{\mathbb{R}^{k\times{k}}}$ and $v\in{\mathbb{R}^{k}}$, we have $|Av|_{\infty}\leq{|A|_{\infty} \vert v\vert_{1}}$. Thus, we also have
\begin{equation*}
\begin{aligned}
|S^{X,M}\Delta^{M}S^{Y,M}-\Sigma^{X,M}\Delta^{M}\Sigma^{Y,M}|_{\infty}&=|[(S^{Y,M}\otimes{S^{X,M}})-(\Sigma^{X,M}\otimes{\Sigma^{Y,M}})]\vect{(\Delta^{M})}|_{\infty}\\
&\leq{|(S^{Y,M}\otimes{S^{X,M}})-(\Sigma^{X,M}\otimes{\Sigma^{Y,M}})|_{\infty}}\vert\vect{(\Delta^{M})}\vert_{1}\\
&=|(S^{Y,M}\otimes{S^{X,M}})-(\Sigma^{X,M}\otimes{\Sigma^{Y,M}})|_{\infty}|\Delta^{M}|_{1}.
\end{aligned}
\end{equation*}

Combining the inequalities gives an upper bound uniform over $\mathcal{G}$ (i.e., for all $G_t$):
\begin{multline*}
\left\vert \left[(S^{Y,M}\otimes{S^{X,M}})\theta^{*}-\vect({S^{Y,M}-S^{X,M}})\right]_{G_{t}}\right\vert_{2} \\
\leq M [|(S^{Y,M}\otimes{S^{X,M}})-(\Sigma^{X,M}\otimes{\Sigma^{Y,M}})|_{\infty}|\Delta^{M}|_{1}  \\
  +|S^{Y,M}-\Sigma^{Y,M}|_{\infty}+|S^{X,M}-\Sigma^{X,M}|_{\infty}],
\end{multline*}
which implies
\begin{equation}{\label{con:RnormBoundMNorm}}
\begin{aligned}
\mathcal{R}^{*}\left(\nabla{\mathcal{L}}(\theta^{*})\right)
& \leq M[|(S^{Y,M}\otimes{S^{X,M}})-(\Sigma^{X,M}\otimes{\Sigma^{Y,M}})|_{\infty}|\Delta^{M}|_{1}\\
&\qquad\qquad\qquad\qquad+|S^{Y,M}-\Sigma^{Y,M}|_{\infty}+|S^{X,M}-\Sigma^{X,M}|_{\infty}].
\end{aligned}
\end{equation}
Assuming $|S^{X,M}-\Sigma^{X,M}|_{\infty}\leq{\delta_n}$ and $|S^{Y,M}-\Sigma^{Y,M}|_{\infty}\leq \delta_n$ implies
\begin{equation}{\label{con:RnormBoundDelta}}
\mathcal{R}^{*}\left(\nabla{\mathcal{L}}(\theta^{*})\right)\leq{M[(\delta_n^2 + 2\delta_n\sigma_{\max})|\Delta^{M}|_{1}+ 2\delta_n]}.
\end{equation}
Setting
\begin{equation}{\label{con:lambdavalueDelta}}
\lambda_{n}=2M\left[\left(\delta_n^2 + 2\delta_n\sigma_{\max}\right)\left \vert \Delta^M \right\vert_{1} + 2\delta_n\right],
\end{equation}
then implies that $\lambda_{n}\geq{2\mathcal{R}^{*}\left(\nabla{\mathcal{L}}(\theta^{*})\right)}$. Thus, invoking Lemma 1 in \cite{negahban2010unified}, $h=\hat{\theta}_{\lambda_{n}}-\theta^{*}$ must satisfy
\begin{equation}{\label{con:DeltathetaConstraintRnorm}}
\mathcal{R}(h_{\mathcal{M}^{\bot}})\leq{3\mathcal{R}(h_{\mathcal{M}})}+4\mathcal{R}(\theta^{*}_{\mathcal{M}^{\bot}}),
\end{equation}
where $\mathcal{M}$ is defined in \eqref{con:Msubspace}. Equivalently,
\begin{equation}{\label{con:DeltathetaConstraint12Norm}}
\vert h_{\mathcal{M}^{\bot}}\vert_{1,2}\leq{3 \vert h_{\mathcal{M}}\vert_{1,2}}+4\vert \theta^{*}_{\mathcal{M}^{\bot}}\vert_{1,2}.
\end{equation}
By the definition of $\nu_2$, we have
\begin{equation}{\label{con:DeltathetaMbot}}
\vert \theta^{*}_{\mathcal{M}^{\bot}}\vert_{1,2}=\sum_{t\notin{\mathcal{S}_{\mathcal{G}}}}\vert\theta^{*}_{G_{t}}\vert_{2}\leq \left(p(p+1)/2-s\right)\nu_2\leq p^2\nu_2.
\end{equation}

Next, we show that $\delta\mathcal{L}(h, \theta^{*})$, as defined in \eqref{con:errorTaylor}, satisfies the Restricted Strong Convexity property: $\delta\mathcal{L}(h,\theta^{*})\geq{\kappa_{\mathcal{L}}\vert h \vert^{2}_{2}} - \omega^2_\mathcal{L}\left(\theta^{*}\right)$ whenever $h$ satisfies \eqref{con:DeltathetaConstraint12Norm}.
We have
\begin{equation*}
\begin{aligned}
\theta^{\top}(S^{Y,M}\otimes{S^{X,M}})\theta&=\theta^{\top}(\Sigma^{Y,M}\otimes{\Sigma^{X,M}})\theta+\theta^{\top}(S^{Y,M}\otimes{S^{X,M}}-\Sigma^{Y,M}\otimes{\Sigma^{X,M}})\theta\\
&\geq{\theta^{\top}(\Sigma^{Y,M}\otimes{\Sigma^{X,M}})\theta-|\theta^{\top}(S^{Y,M}\otimes{S^{X,M}}-\Sigma^{Y,M}\otimes{\Sigma^{X,M}})\theta|}\\
&\geq{\lambda^{*}_{\min}}\vert \theta \vert^{2}_{2} - M^{2}|S^{Y,M}\otimes{S^{X,M}}-\Sigma^{Y,M}\otimes{\Sigma^{X,M}}|_{\infty}\vert\theta\vert^{2}_{1,2},
\end{aligned}
\end{equation*}
where the last inequality follows from Lemma~\ref{lemma:QRleq12Norm} and $\lambda^{*}_{\min}=\lambda_{\min}(\Sigma^{X,M})\times{\lambda_{\min}(\Sigma^{Y,M})}=\lambda_{\min}(\Sigma^{Y,M}\otimes{\Sigma^{X,M}})>0$. Thus,
\begin{equation*}
\begin{aligned}
\delta\mathcal{L}(h,\theta^{*})&=\frac{1}{2}h^{\top}(S^{Y,M}\otimes{S^{X,M}})h\\
&\geq{\frac{1}{2}\lambda^{*}_{\min}}\vert h \vert^{2}_{2} - \frac{1}{2}M^{2}|S^{Y,M}\otimes{S^{X,M}}-\Sigma^{Y,M}\otimes{\Sigma^{X,M}}|_{\infty}\vert h\vert^{2}_{1,2}.
\end{aligned}
\end{equation*}
By Lemma \ref{lemma:L12NormleqL2Norm} and \eqref{con:DeltathetaConstraint12Norm}, we have 
\begin{equation*}
\begin{aligned}
\vert h\vert^{2}_{1,2}&=(\vert h_{\mathcal{M}}\vert_{1,2}+\vert h_{\mathcal{M}^{\bot}}\vert_{1,2})^{2}\leq16({\vert h_{\mathcal{M}}\vert_{1,2}}+\vert \theta^{*}_{\mathcal{M}^{\bot}}\vert_{1,2})^{2}\\
&\leq 16(\sqrt{s}\vert h\vert_{2}+p^2\nu_2)^2\leq 32s\vert h\vert^{2}_{2}+32p^{4}\nu_2^2.
\end{aligned}
\end{equation*}
Combining with the above equation, we get
\begin{equation}
\begin{aligned}
\delta\mathcal{L}(h,\theta^{*})&\geq{\left[\frac{1}{2}\lambda^{*}_{\min}-16M^{2}s|S^{Y,M}\otimes{S^{X,M}}-\Sigma^{Y,M}\otimes{\Sigma^{X,M}}|_{\infty}\right]}\vert h\vert^{2}_{2}\\
&\qquad\qquad-16M^2p^4\nu_2^2|S^{Y,M}\otimes{S^{X,M}}-\Sigma^{Y,M}\otimes{\Sigma^{X,M}}|_{\infty}\\
&\geq \left[\frac{1}{2}\lambda^{*}_{\min}-8M^{2}s\left(\delta^2_{n}+2\delta^2_{n}\sigma_{\max}\right)\right] \vert h\vert^{2}_{2}\\
&\qquad\qquad-16M^2p^4\nu_2^2\left(\delta^2_{n} + 2\delta_{n}\sigma_{\max}\right).
\end{aligned}
\end{equation}
Thus, appealing to \eqref{con:KroneckerBoundSup}, the Restricted Strong Convexity property holds with 
\begin{equation}
\begin{aligned}
\kappa_{\mathcal{L}} &\; = \; \frac{1}{2}\lambda^{*}_{\min}-8M^{2}s\left(\delta^2 + 2\delta_n\sigma_{\max}\right),\\
\omega_{\mathcal{L}}& \;=\; 4Mp^2\nu_2\sqrt{\delta_n^2 + 2\delta_n\sigma_{\max}}.
\end{aligned}    
\end{equation}
When $\delta_n < \frac{1}{4}\sqrt{ \frac{\lambda^{*}_{\min} + 16M^2s (\sigma_{\max})^2}{ M^2 s}} - \sigma_{\max}$ as we assumed in the theorem, then $\kappa_{\mathcal{L}}>0$. 
By Theorem 1 of \cite{negahban2010unified} and Lemma~\ref{lemma:L12NormleqL2Norm}, letting  $$\lambda_{n}=2M\left[\left(\delta_n^2 +2 \delta_n \sigma_{\max}\right)\vert \Delta^{M}\vert_{1}+ 2\delta_n\right],$$ as in \eqref{con:lambdavalueDelta}, ensures that
\begin{equation}
\begin{aligned}
\|\hat{\Delta}^{M}-\Delta^{M}\|^{2}_{F}&=\vert\hat{\theta}_{\lambda_{n}}-\theta^{*}\vert^{2}_{2}\\
&\leq{9\frac{\lambda^{2}_{n}}{\kappa^{2}_{\mathcal{L}}}}\Psi^{2}(\mathcal{M})+\frac{\lambda_{n}}{\kappa_{\mathcal{L}}}\left(2\omega^{2}_{\mathcal{L}}+4\mathcal{R}(\theta^{*}_{\mathcal{M}^{\bot}}) \right)\\
&=\frac{9\lambda^{2}_{n}s}{\kappa^{2}_{\mathcal{L}}}+\frac{2\lambda_{n}}{\kappa_{\mathcal{L}}}(\omega^{2}_{\mathcal{L}}+2p^2\nu_2)\\
& = \Gamma^2_n.
\end{aligned}
\end{equation}

We then prove that $\hat{E}_{\Delta}=E_{\Delta}$. Recall that we have assumed that $0<\Gamma_n<\tau/2=(\nu_1-\nu_2)/2$ and $\nu_2+\Gamma_n \leq \epsilon_n < \nu_1-\Gamma_n$.
Note that we have $\|\hat{\Delta}^M_{jl}-\Delta^{M}_{jl}\|_{F}\leq{\|\hat{\Delta}^M-\Delta^{M}\|_{F}}\leq \Gamma_n$ for any $(j,l)\in{V^{2}}$. Recall that
\begin{equation}{\label{assump:IntrinsicDimLarge}}
E_{\Delta}\;=\;\{(j,l)\in{V^2}:\; j\neq{l},D_{jl}>0 \}.
\end{equation}
First, we prove that $E_{\Delta}\subseteq{\hat{E}_{{\Delta}}}$. For any $(j,l)\in{E_{\Delta}}$, by the definition of $\nu_1$ in Section~\ref{sec:thm-E-Delta}, we have
\begin{equation*}
\begin{aligned}
\|\hat{\Delta}^M_{jl}\|_{F}&\geq{\|\Delta^{M}_{jl}\|_{F}-\|\hat{\Delta}_{jl}^M-\Delta^{M}_{jl}\|_{F}}\\
&\geq \nu_1-\Gamma_n\\
&>\epsilon_{n}.
\end{aligned}
\end{equation*}
The last inequality holds because we have assumed that $\epsilon_{n} <\nu_1-\Gamma_n$. Thus, by the definition of $\hat{E}_{{\Delta}}$ in~\eqref{con:changedgeest}, we have $(j,l)\in{\hat{E}_{{\Delta}}}$, which further implies that $E_{\Delta}\subseteq{\hat{E}_{{\Delta}}}$.

We then show $\hat{E}_{{\Delta}}\subseteq{E_{\Delta}}$. Let $\hat{E}^{c}_{{\Delta}}$ and $E^{c}_{\Delta}$ denote the complement of $\hat{E}_{{\Delta}}$ and $E_{\Delta}$. For any $(j,l)\in{E^{c}_{\Delta}}$, which also means that $(l,j)\in{E^{c}_{\Delta}}$, by the definition of $\nu_2$, we have that
\begin{equation*}
\begin{aligned}
\|\hat{\Delta}^M_{jl}\|_{F}&\leq{\|\Delta^{M}_{jl}\|_{F}+\|\hat{\Delta}_{jl}^M-\Delta^{M}_{jl}\|_{F}}\\
&\leq \nu_2+\Gamma_n\\
& \leq \epsilon_n.
\end{aligned}
\end{equation*}
Again, the last inequality is true because we have assumed $\epsilon_{n}\geq \nu_2+\Gamma_n$. Thus, by the definition of $\hat{E}_{{\Delta}}$, we have $(j,l)\notin{\hat{E}_{{\Delta}}}$ or $(j,l)\in{\hat{E}^{c}_{{\Delta}}}$. This implies that $E^{c}_{\Delta}\subseteq{\hat{E}^{c}_{{\Delta}}}$, or $\hat{E}_{{\Delta}}\subseteq{E_{\Delta}}$. Combining with the previous conclusion that $E_{\Delta}\subseteq{\hat{E}_{{\Delta}}}$, the proof is complete.
\end{proof}

\subsection{Proof of Theorem~\ref{Thm:ErrSamCovDis}}
\label{sec:proof-thm-ErrSamCovDis}

We only need to prove that
\begin{multline}\label{eq:ErrCovDis-prob}
P\left( \lvert S^M-\Sigma^{M} \rvert_{\infty}>\delta \right) 
\leq C_1 np \exp\{ -C_2\Phi(T,L)M^{-(1+\beta)}\delta \} \\
+ C_3(pM)^2\exp\{ -C_4nM^{-2(1+\beta)}\delta^2 \} 
+ C_5 npL \exp \left\{ - \frac{ C_6 M^{-2(1+\beta)} \delta^2}{ \tilde{\psi}_2(T,L) } \right\},
\end{multline}
where $S^M$ can be understood as $S^{X,M}$ or $S^{Y,M}$ and $\Sigma^{M}$ can be understood as $\Sigma^{X,M}$ or $\Sigma^{Y,M}$, with $C_k=C^X_k$ or $C_k=C^Y_k$ for $k=1,2,3,4$. To see that \eqref{eq:ErrCovDis-prob} implies \eqref{eq:ErrCovDis}, we first note that \eqref{eq:ErrCovDis-prob} implies that
\begin{equation*}
\begin{aligned}
P&\left( \lvert S^{X,M}-\Sigma^{X,M} \rvert_{\infty} \leq \delta \,\text{and}\, \lvert S^{Y,M}-\Sigma^{Y,M} \rvert_{\infty} \leq \delta \right) \\
& \geq  1 - P\left( \lvert S^{X,M}-\Sigma^{X,M} \rvert_{\infty}>\delta \right) - P\left( \lvert S^{Y,M}-\Sigma^{Y,M} \rvert_{\infty}>\delta \right) \\
& \geq  1 - 2  \bar{C}_1 pM \exp\{ -\bar{C}_2\Phi(T,L)M^{-(1+\beta)}\delta \} - 2\bar{C}_3(pM)^2\exp\{ -\bar{C}_4nM^{-2(1+\beta)}\delta^2 \} ,
\end{aligned}
\end{equation*}
where $\bar{C}_k$ for $k=1,2,3,4$ are defined in Theorem~\ref{Thm:ErrSamCovDis}. Thus, letting the last two terms in the last line of the above equation be $\iota/2$, we then have \eqref{eq:ErrCovDis}. In this way, the rest of the proof will focus on proving \eqref{eq:ErrCovDis-prob}.

Denote the $(j,l)$-th submatrix of $S^M$ as $S^M_{jl}$, and the $(k,m)$-th entry of $S^M_{jl}$ as $\hat{\sigma}_{jl,km}$. We have $S^M=(\hat{\sigma}_{jl,km})_{1\leq j,l \leq p , \leq k,m \leq M}$ and $\Sigma^M=(\sigma_{jl,km})_{1\leq j,l \leq p , \leq k,m \leq M}$. Then, by the definition of $S^M$ and $\Sigma^M$, we have
\begin{equation*}
\begin{aligned}
\hat{\sigma}_{jl,km}
=\frac{1}{n}\sum^{n}_{i=1}\hat{a}_{ijk}\hat{a}_{ilm}
\qquad\text{and}\qquad
\sigma_{jl,km}=\mathbb{E}\left[ a_{ijk}a_{ilm} \right].
\end{aligned}
\end{equation*}
Note that
\begin{equation*}
\begin{aligned}
\hat{a}_{ijk}&=\langle \hat{g}_{ij},\hat{\phi}_{jk} \rangle\\
&=\langle g_{ij}+\hat{g}_{ij}-g_{ij},\phi_{jk}+\hat{\phi}_{jk}-\phi_{jk} \rangle\\
&=\langle g_{ij},\phi_{jk} \rangle+\langle g_{ij},\hat{\phi}_{jk}-\phi_{jk} \rangle+\langle \hat{g}_{ij}-g_{ij},\phi_{jk} \rangle+\langle \hat{g}_{ij}-g_{ij},\hat{\phi}_{jk}-\phi_{jk} \rangle\\
&=a_{ijk}+\langle g_{ij},\hat{\phi}_{jk}-\phi_{jk} \rangle+\langle \hat{g}_{ij}-g_{ij},\phi_{jk} \rangle+\langle \hat{g}_{ij}-g_{ij},\hat{\phi}_{jk}-\phi_{jk} \rangle.
\end{aligned}
\end{equation*}
Thus, we have
\begin{equation*}
\begin{aligned}
\hat{\sigma}_{jl,km}-\sigma_{jl,km}=\frac{1}{n}\sum^{n}_{i=1}\left( \hat{a}_{ijk}\hat{a}_{ilm}-\sigma_{jl,km} \right)
=\sum^{16}_{u=1}I_{u},
\end{aligned}
\end{equation*}
where
\begingroup
\allowdisplaybreaks
\begin{align*}
I_1&=\frac{1}{n}\sum^{n}_{i=1}\left( a_{ijk}a_{ilm}-\mathbb{E}(a_{ijk}a_{ilm}) \right),\\
I_2&=\frac{1}{n}\sum^{n}_{i=1}a_{ijk}\langle \hat{g}_{il}-g_{il},\phi_{lm} \rangle,\\
I_3&=\frac{1}{n}\sum^{n}_{i=1}a_{ijk}\langle g_{il},\hat{\phi}_{lm}-\phi_{lm} \rangle,\\
I_4&=\frac{1}{n}\sum^{n}_{i=1}a_{ijk}\langle \hat{g}_{il}-g_{il},\hat{\phi}_{lm}-\phi_{lm} \rangle,\\
I_5&=\frac{1}{n}\sum^{n}_{i=1}a_{ilm}\langle \hat{g}_{ij}-g_{ij},\phi_{jk} \rangle,\\
I_6&=\frac{1}{n}\sum^{n}_{i=1}\langle \hat{g}_{ij}-g_{ij},\phi_{jk} \rangle \langle \hat{g}_{il}-g_{il},\phi_{lm} \rangle,\\
I_7&=\frac{1}{n}\sum^{n}_{i=1}\langle \hat{g}_{ij}-g_{ij},\phi_{jk} \rangle \langle g_{il},\hat{\phi}_{lm}-\phi_{lm} \rangle,\\
I_8&=\frac{1}{n}\sum^{n}_{i=1}\langle \hat{g}_{ij}-g_{ij},\phi_{jk} \rangle \langle \hat{g}_{il}-g_{il},\hat{\phi}_{lm}-\phi_{lm} \rangle,\\
I_9&=\frac{1}{n}\sum^{n}_{i=1}\langle g_{ij},\hat{\phi}_{jk}-\phi_{jk} \rangle a_{ilm},\\
I_{10}&=\frac{1}{n}\sum^{n}_{i=1}\langle g_{ij},\hat{\phi}_{jk}-\phi_{jk} \rangle \langle \hat{g}_{il}-g_{il},\phi_{lm} \rangle,\\
I_{11}&=\frac{1}{n}\sum^{n}_{i=1}\langle g_{ij},\hat{\phi}_{jk}-\phi_{jk} \rangle \langle g_{il},\hat{\phi}_{lm}-\phi_{lm} \rangle,\\
I_{12}&=\frac{1}{n}\sum^{n}_{i=1}\langle g_{ij},\hat{\phi}_{jk}-\phi_{jk} \rangle \langle \hat{g}_{il}-g_{il},\hat{\phi}_{lm}-\phi_{lm} \rangle,\\
I_{13}&=\frac{1}{n}\sum^{n}_{i=1}\langle \hat{g}_{ij}-g_{ij},\hat{\phi}_{jk}-\phi_{jk} \rangle a_{ilm},\\
I_{14}&=\frac{1}{n}\sum^{n}_{i=1}\langle \hat{g}_{ij}-g_{ij},\hat{\phi}_{jk}-\phi_{jk} \rangle \langle \hat{g}_{il}-g_{il},\phi_{lm} \rangle,\\
I_{15}&=\frac{1}{n}\sum^{n}_{i=1}\langle \hat{g}_{ij}-g_{ij},\hat{\phi}_{jk}-\phi_{jk} \rangle \langle g_{il},\hat{\phi}_{lm}-\phi_{lm} \rangle,\\
I_{16}&=\frac{1}{n}\sum^{n}_{i=1}\langle \hat{g}_{ij}-g_{ij},\hat{\phi}_{jk}-\phi_{jk} \rangle \langle \hat{g}_{il}-g_{il},\hat{\phi}_{lm}-\phi_{lm} \rangle.
\end{align*}
\endgroup
Note that $I_{u}$, $u=1,\ldots,16$ depend on $j,l,k,m$. To simplify the notation, we do not explicitly denote this fact. Thus, for any $0<\delta\leq 1$, when for any $1\leq j,l \leq p$ and $1 \leq k,m \leq M$, if $\lvert I_{u} \rvert \leq \delta/16$, $u=1,\ldots,16$, we have $\lvert S^{M}-\Sigma^{M} \rvert_{\infty}\leq \delta$. We now calculate the probability of $\lvert I_{u} \rvert \leq \delta/16$, $u=1,\ldots,16$, $1\leq j,l \leq p$ and $1 \leq k,m \leq M$.

By Assumption~\ref{assump:EigenAssump} (i), we have constants $d_1,d_2>0$, such that $\lambda_{jk} \leq d_1 k^{-\beta}$, $d_{jk}\leq d_2 k^{1+\beta}$ for any $j=1,\ldots,p$ and $k \geq 1$. Let $d_0=\max\{1,\sqrt{d_1},d_2\}$ and $\xi_{ijk}=\lambda^{-1/2}_{jk}a_{ijk}$ so that $\xi_{ijk} \sim N(0,1)$ are i.i.d.~for $i=1,\ldots,n$. Let
\begin{equation}
\delta_1=\frac{\delta}{144d^{2}_{0}M^{1+\beta}\sqrt{3\lambda_{0,\max}}}
\quad\text{and}\quad
\delta_2=9\lambda_{0,\max}\delta_1=\frac{\delta}{16d^{2}_{0}M^{1+\beta}\sqrt{3\lambda_{0,\max}}},
\end{equation}
where $\lambda_{0,\max}=\max_{j \in V} \sum^{\infty}_{k=1} \lambda_{jk}$.
Recall that $\hat{K}_{jj}$, $j=1,\ldots,p$, are defined in \eqref{eq:CovEstDis}. We define five events $A_1$-$A_5$ as follows:
\begin{equation}
\begin{aligned}
A_1&: \; \lVert \hat{g}_{ij}-g_{ij} \rVert \leq \delta_1, \quad \forall i=1,\ldots,n \ \forall j=1,\ldots,p,\\
A_2&: \; \lVert \hat{K}_{jj}-K_{jj} \rVert_{\text{HS}} \leq \delta_2 \quad \forall j=1,\ldots,p,\\
A_3&: \; \frac{1}{n}\sum^{n}_{i=1} \xi^{2}_{ijk} \leq \frac{3}{2} \quad \forall j=1,\ldots,p \ \forall k=1,\ldots, M,\\
A_4&: \; \frac{1}{n}\sum^{n}_{i=1}\lVert g_{ij} \rVert^{2}\leq 2\lambda_{0,\max} \quad \forall j=1,\ldots,p,\\
A_5&: \; \lvert \frac{1}{n}\sum^{n}_{i=1}a_{ijk}a_{ilm}-\sigma_{jl,km} \rvert \leq \frac{\delta}{16} \quad \forall 1 \leq j,l \leq \ 1 \leq k,m \leq M.
\end{aligned}
\end{equation}
Without loss of generality, we assume that $\langle \hat{\phi}_{jl},\phi_{jl} \rangle\geq 0$ for any $1\leq j \leq p$ and $1 \leq k \leq M$ (if this is not true, we only need to use $-\phi_{jl}$ to substitute $\phi_{jl}$).
Then, by Lemma~\ref{lemma:I1bound}-Lemma~\ref{lemma:I16bound}, when $A_1$-$A_5$ hold simultaneously, we have $\lvert I_{u} \rvert \leq \delta/16$ for all $u=1,\ldots,16$, $1 \leq j,l \leq p$ and $ \ 1 \leq k,m \leq M$. Therefore,
\begin{equation*}
\begin{aligned}
P&\left(\lvert S^M-\Sigma^M \rvert_{\infty} \leq \delta \right)\\
&\geq P\left( \lvert I_{u} \rvert \leq \delta/16, \; \text{for all} \; 1 \leq u \leq 16, 1 \leq j,l \leq \ 1 \leq k,m \leq M \right)\\
&\geq P\left( \bigcap^{5}_{w=1}A_w \right),
\end{aligned}
\end{equation*}
which implies
\begin{equation}
P\left(\lvert S^M-\Sigma^M \rvert_{\infty} > \delta \right)\leq P\left( \bigcup^{5}_{w=1}\bar{A}_w \right) \leq \sum^{5}_{w=1}P\left(\bar{A}_w \right),
\end{equation}
where the last inequality follows Boole's inequality and $\bar{A}$ denotes the complement of $A$. Then we only need to give an upper bound for $P(\bar{A}_w)$, $w=1,\ldots,5$.

By Theorem~\ref{Thm:approxerrsinglefunc} and the definition of $\tilde{\psi}_1$-$\tilde{\psi}_4$, 
we have
\begin{multline*}
P(\bar{A}_1)= P\left(\lVert \hat{g}_{ij}-g_{ij} \rVert>\delta_1 \; \exists 1\leq i \leq n,1\leq j \leq p\right)\\
\leq 2(np)\left\{ \exp\left( -\frac{\delta_{1}^{2}}{72\tilde{\psi}^2_{1}(T,L)+6\sqrt{2}\tilde{\psi}_1(T,L)\delta_1} \right) 
+ L \exp \left( - \frac{\delta_1^2}{ \tilde{\psi}_2(T,L)} \right) 
\right.\\
+ \left. \exp\left(-\frac{\delta_{1}^{2}}{72\lambda_{0,\max}\tilde{\psi}_3(L)+6\sqrt{2\lambda_{0,\max}\tilde{\psi}_3(L)}\delta_1}\right) \right\}.
\end{multline*}
Let $\gamma_1=\sqrt{2}/(12\times144d^{2}_{0}3\sqrt{3\lambda_{0,\max}})$ and $\gamma_3=1/(72\lambda_{0,\max}\times(144d^{2}_{0}\sqrt{3\lambda_{0,\max}})^2)$. If $\tilde{\psi}_1<\gamma_1\cdot\delta/M^{1+\beta}$ and $\tilde{\psi}_3<\gamma_3\cdot\delta^2/M^{2+2\beta}$, then $72\tilde{\psi}^{2}_1<6\sqrt{2}\tilde{\psi}_1\delta_1$ and $72\lambda_{0,\max}\tilde{\psi}_3<6\sqrt{2\lambda_{0,\max}\tilde{\psi}_3}\delta_1$, which implies that
\begin{equation}\label{eq:A1bound}
\begin{aligned}
&P(\bar{A}_1)\\
&\leq 2np\left\{ \exp\left( -\frac{\delta_{1}}{12\sqrt{2}\tilde{\psi}_1(T,L)} \right)+\exp\left( -\frac{\delta_1}{12\sqrt{2\lambda_{0,\max}}\sqrt{\tilde{\psi}_3(L)}} \right) + L \exp \left( - \frac{\delta_1^2}{ \tilde{\psi}_2(T,L)} \right) \right\}\\
&\overset{(i)}{\leq} 2np\left\{ \exp\left( -\frac{\delta_1}{12\sqrt{2}}\Phi(T,L) \right)+ \exp\left( -\frac{\delta_1}{12\sqrt{2\lambda_{0,\max}}}\Phi(T,L) \right) + L \exp \left( - \frac{\delta_1^2}{ \tilde{\psi}_2(T,L)}\right) \right\}\\
&\overset{(ii)}{\leq} 4np \exp\left( -\frac{\delta_1}{12\sqrt{2\lambda_{0,\max}}}\Phi(T,L) \right) + 2 n p L \exp \left( - \frac{\delta_1^2}{ \tilde{\psi}_2(T,L)} \right)\\
&=4np\exp\left(-\frac{1}{1728\sqrt{6}\lambda_{0,\max}d^{2}_{0}}\cdot\frac{\delta}{M^{1+\beta}}\cdot\Phi(T,L)\right) 
\\
& \quad + 2 n p L \exp \left( - \frac{\delta^2  }{6228 d^{4}_{0} \lambda_{0,\max} M^{2+2\beta} \tilde{\psi}_2(T,L) }  \right) ,
\end{aligned}
\end{equation}
where $(i)$ follows the definition of $\Phi(T,L)$ and $(ii)$ follows the fact that $\lambda_{0,\max}>1$.

Next, we bound $P(\bar{A}_4)$. For any two real values $z_1,z_2$ and any positive integer $k$, we have
\begin{equation*}
    (z_1+z_2)^{k}\leq \left(\vert z_1 \vert + \vert z_2 \vert \right)^{k}
    =2^{k}\left(\frac{1}{2}\vert z_1 \vert + \frac{1}{2} \vert z_2 \vert \right)^{k} \leq 2^{k-1}\left(\vert z_1 \vert + \vert z_2 \vert \right),
\end{equation*}
where the last line follows from Jensen's inequality. Since $\mathbb{E}[\Vert g_{ij} \Vert^2]=\lambda_{j0}$, $i=1,\ldots,n$, $j=1,2\dots,p$, then, by Jensen's inequality and Lemma~\ref{lemma:2kmomentGuassFunc}, for any $k\geq 2$, we have
\begin{equation*}
\begin{aligned}
\mathbb{E}\left[ \left(\Vert g_{ij} \Vert^2-\lambda_{j0}\right)^{k} \right]
\leq 2^{k-1} \left( \mathbb{E}\left[ \Vert g_{ij} \Vert^{2k} +\lambda^{k}_{j0}\right] \right)
\leq 2^{k-1} \left( (2\lambda_{j0})^{k}k!+\lambda^{k}_{j0} \right)
\leq (4\lambda_{j0})^{k}k!\,.
\end{aligned}
\end{equation*}
Thus,
\begin{equation*}
\sum^{n}_{i=1}\mathbb{E}\left[ \left(\Vert g_{ij} \Vert^2-\lambda_{j0}\right)^{k} \right] \leq \frac{k!}{2}n\times (32\lambda^{2}_{j0})\times (4\lambda_{j0})^{k-2}.
\end{equation*}
Then by Lemma~\ref{lemma:HilbertBernstein}, for any $\epsilon>0$, we have
\begin{equation*}
P\left( \left\vert \frac{1}{n}\sum^{n}_{i=1}\left\Vert g_{ij} \right\Vert^2 - \lambda_{j0} \right\vert>\epsilon \right)\leq 2\exp\left( -\frac{n\epsilon^{2}}{64\lambda^{2}_{j0}+8\lambda_{j0}\epsilon} \right).
\end{equation*}
Finally, 
\begin{equation*}
\begin{aligned}
P\left( \frac{1}{n}\sum^{n}_{i=1}\left\Vert g_{ij} \right\Vert^2>2\lambda_{0,\max} \right)&\leq P\left( \frac{1}{n}\sum^{n}_{i=1}\left\Vert g_{ij} \right\Vert^2>2\lambda_{j0} \right)\\
&\leq P\left( \left\vert \frac{1}{n}\sum^{n}_{i=1}\left\Vert g_{ij} \right\Vert^2 - \lambda_{j0} \right\vert>\lambda_{j0} \right)\\
&\leq 2\exp\left(-\frac{n}{72}\right)
\end{aligned}
\end{equation*}
and
\begin{equation}\label{eq:A4bound}
P(\bar{A}_4)=P\left( \frac{1}{n}\sum^{n}_{i=1}\left\Vert g_{ij} \right\Vert^2>2\lambda_{0,\max},\; \exists j=1,\ldots,p \right)\leq 2p\exp\left(-\frac{n}{72}\right).
\end{equation}

Next, we bound $P(\bar{A}_2)$.
Let $\hat{K}^g_{jj}(s,t) = \frac{1}{n} \sum^n_{i=1} g_{ij}(s) g_{ij}(t)$ and $K_{jj}(s,t)=\mathbb{E}[ g_{ij}(s) g_{ij}(t) ]$, $j \in V$, and define
\begin{equation*}
A^{\prime}_2: \; \lVert \hat{K}^g_{jj}-K^g_{jj} \rVert_{\text{HS}} \leq \delta_2 \quad \forall j=1,\ldots,p.
\end{equation*}
Note that
\begin{equation*}
\begin{aligned}
&\Vert \hat{K}^g_{jj}(s,t)-K^g_{jj}(s,t) \Vert_{\text{HS}}\\
&=\left\Vert \frac{1}{n}\sum^{n}_{i=1}\left[ \hat{g}_{ij}(s)-g_{ij}(s)+g_{ij}(s) \right]\left[ \hat{g}_{ij}(t)-g_{ij}(t)+g_{ij}(t) \right] - K^g_{jj}(s,t) \right\Vert_{\text{HS}}\\
&\leq \frac{1}{n}\sum^{n}_{i=1}\Vert \hat{g}_{ij}-g_{ij} \Vert^2 + \frac{2}{n}\sum^{n}_{i=1}\Vert \hat{g}_{ij}-g_{ij} \Vert \cdot \Vert g_{ij} \Vert + \left\Vert \frac{1}{n}\sum^{n}_{i=1}\left[g_{ij}(s)g_{ij}(t)-K^g_{jj}(s,t)  \right] \right\Vert_{\text{HS}}.
\end{aligned}
\end{equation*}
Let
\begin{equation*}
\begin{aligned}
A_6&: \; \left\Vert \frac{1}{n}\sum^{n}_{i=1}\left[g_{ij}(s)g_{ij}(t)-K^g_{jj}(s,t)  \right] \right\Vert_{\text{HS}} \leq 4\lambda_{0,\max}\delta_{1},\; \forall j=1,\ldots,p.
\end{aligned}
\end{equation*}
We show that $A_1\cap A_4 \cap A_6\Longrightarrow A^{\prime}_2$. By Jensen's inequality, we have
\begin{equation*}
\frac{1}{n}\sum^{n}_{i=1}\left\Vert g_{ij} \right\Vert \leq \sqrt{\frac{1}{n}\sum^{n}_{i=1}\left\Vert g_{ij} \right\Vert^2}.
\end{equation*}
On the event $A_4$, we have $(1/n)\sum^{n}_{i=1}\left\Vert g_{ij} \right\Vert \leq \sqrt{2\lambda_{0,\max}}$ for any $j=1,\ldots,p$. When $A_1$, $A_4$, and $A_6$ hold simultaneously, we have
\begin{equation*}
\Vert \hat{K}^g_{jj}(s,t)-K^g_{jj}(s,t) \Vert_{\text{HS}} \leq \delta^{2}_{1}+2\sqrt{2\lambda_{0,\max}}\delta_1+4\lambda_{0,\max}\delta_{1}\leq 9\lambda_{0,\max}\delta_1,
\end{equation*}
which is $A_2$. Therefore, $A_1\cap A_4 \cap A_6\Longrightarrow A^{\prime}_2$, which implies that $\bar{A^{\prime}}_2\Longrightarrow \bar{A}_1\cup\bar{A}_4\cup\bar{A}_6$ and $P(\bar{A^{\prime}}_2)\leq P(\bar{A}_1)+P(\bar{A}_4)+P(\bar{A}_6)$. We upper bound $P(\bar{A}_6)$ next.

By Lemma~\ref{lemma:crosscovfuncbound}, for any $j=1,\ldots,p$, we have
\begin{equation*}
P\left( \left\Vert \frac{1}{n}\sum^{n}_{i=1}\left[g_{ij}(s)g_{ij}(t)-K^g(s,t)  \right] \right\Vert_{\text{HS}}>4\lambda_{0,\max}\delta_{1} \right)\leq 2\exp\left(-\frac{n\delta^{2}_{1}}{6}\right).
\end{equation*}
Thus,
\begin{equation}\label{eq:A6bound}
P(\bar{A}_6)\leq 2p\exp\left(-\frac{n\delta^{2}_{1}}{6}\right)=2p\exp\left(-\frac{1}{373248d^{4}_{0}\lambda^{2}_{0,\max}}\times n\frac{\delta^2}{M^{2+2\beta}}\right).
\end{equation}
Combining \eqref{eq:A1bound}, \eqref{eq:A4bound}, and \eqref{eq:A6bound}, we have
\begin{multline*}
P(\bar{A^{\prime}}_2) \leq 4pM\exp\left(-\frac{1}{1728\sqrt{6}\lambda_{0,\max}d^{2}_{0}}\cdot\frac{\delta}{M^{1+\beta}}\cdot\Phi(T,L)\right)+2p\exp\left(-\frac{n}{72}\right)\\ 
+ 2p\exp\left(-\frac{1}{373248d^{4}_{0}\lambda^{2}_{0,\max}}\times n\frac{\delta^2}{M^{2+2\beta}}\right).
\end{multline*}
Finally, $P(\bar{A}_2) \leq P(\bar{A^{\prime}}_{X,2})+P(\bar{A^{\prime}}_{Y,2})$, where $A^{\prime}_{X,2}$ and $A^{\prime}_{Y,2}$ are defined similarly to $A^{\prime}_2$ with $g$ being $X$ and $Y$, since $$\Vert \hat{K}_jj (s,t) - K_{jj}(s,t) \Vert_{\text{HS}} \leq \Vert \hat{K}^X_jj (s,t) - K^X_{jj}(s,t) \Vert_{\text{HS}} + \Vert \hat{K}^Y_jj (s,t) - K^Y_{jj}(s,t) \Vert_{\text{HS}}.$$ Thus, we have
\begin{multline*}
P(\bar{A}_2) \leq  8pM\exp\left(-\frac{1}{1728\sqrt{6}\lambda_{0,\max}d^{2}_{0}}\cdot\frac{\delta}{M^{1+\beta}}\cdot\Phi(T,L)\right)+4p\exp\left(-\frac{n}{72}\right)\\
+ 4p\exp\left(-\frac{1}{373248d^{4}_{0}\lambda^{2}_{0,\max}}\times n\frac{\delta^2}{M^{2+2\beta}}\right).
\end{multline*}

For $P(\bar{A}_3)$, note that $\sum^{n}_{i=1}\xi^{2}_{ijk}\sim \chi^{2}_{n}$ for any $j=1,\ldots,p$ and $k=1,\ldots,M$.
By Pages 28-29 of \citet{Boucheron2013Concentration}, for any $\epsilon>0$, we have
\begin{equation*}
P\left(\frac{1}{n}\sum^{n}_{i=1}\xi^{2}_{ijk}-1>\epsilon\right) \leq \exp\left(-\frac{n\epsilon^2}{4+4\epsilon}\right).
\end{equation*}
Letting $\epsilon=1/2$, we have
\begin{equation}\label{eq:A3bound}
P(\bar{A}_3)\leq pM\exp\left(-\frac{n}{24}\right).
\end{equation}

Finally, we upper bound $P(\bar{A}_5)$.
Note that
\begin{equation*}
\begin{aligned}
\mathbb{E}\left[ \left( a_{ijk}a_{ilm}-\mathbb{E}(a_{ijk}a_{ilm}) \right)^{k} \right]&=\lambda^{k/2}_{jk}\lambda^{k/2}_{lm}\mathbb{E}\left[ \left( \xi_{ijk}\xi_{ilm}-\mathbb{E}(\xi_{ijk}\xi_{ilm}) \right)^{k} \right]\\
&\leq d^{k}_0 \mathbb{E}\left[ \left( \xi_{ijk}\xi_{ilm}-\mathbb{E}(\xi_{ijk}\xi_{ilm}) \right)^{k} \right],
\end{aligned}
\end{equation*}
and
\begin{equation*}
\begin{aligned}
\mathbb{E}\left[\left( \xi_{ijk}\xi_{ilm}-\mathbb{E}(\xi_{ijk}\xi_{ilm}) \right)^{k} \right]&\leq 2^{k-1}\left(\mathbb{E}\left[ \vert\xi_{ijk}\xi_{ilm}\vert^{k}\right]+\vert \mathbb{E}(\xi_{ijk}\xi_{ilm}) \vert^{k}\right)\\
&\leq 2^{k-1} \left( \mathbb{E}[\xi^{2k}_{ij1}] +1\right)\\
&\leq 2^{k-1}(2^k k! +1 )\\
&\leq 4^k k!\,.
\end{aligned}
\end{equation*}
Thus
\begin{equation*}
\mathbb{E}\left[ \left( a_{ijk}a_{ilm}-\mathbb{E}(a_{ijk}a_{ilm}) \right)^{k} \right] \leq (4d_0)^k k!
\end{equation*}
and Lemma~\ref{lemma:HilbertBernstein} tells us that for any $1 \leq j,l \leq p$ and $1 \leq k,m \leq M$, we have
\begin{align*}
P\left(\left\vert \frac{1}{n}\sum^{n}_{i=1}a_{ijk}a_{ilm}-\sigma_{jl,km} \right\vert>\frac{\delta}{16} \right)
& \leq 2\exp\left(-\frac{n\delta^2}{16512d^{2}_{0}}\right).
\end{align*}
Therefore,
\begin{equation}\label{eq:A5bound}
P\left( \bar{A}_5 \right)\leq 2(pM)^2 \exp\left(-\frac{n\delta^2}{16512d^{2}_{0}}\right).
\end{equation}
Let $C_1=12$, $C_2=1/(1728\sqrt{6}\lambda_{0,\max})$, $C_3=9$, $C_4=1/(373248d^{4}_{0}\lambda^{2}_{0,\max})$, $C_5=2$, and $C_6=1/(6228 d^{4}_{0} \lambda_{0,\max})$. The final result follows by combining the upper bounds on $P(\bar{A}_w)$, $w=1,\ldots,5$.

\newpage

\section{Additional Results}\label{app:MoreThm}

In this section, we establish additional results that are needed to prove the main results.

\subsection{Theorem~\ref{Thm:approxerrsinglefunc} and Its Proof}\label{appsubsec:approxerrsinglefunc}

We give a non-asymptotic error bound on the function estimated using the basis expansion, which is subsequently used to establish Theorem~\ref{Thm:ErrSamCovDis}.

For a random function $g(t) \in \mathbb{H}$, where $t \in\mathcal{T}$, $\mathcal{T}$ is a closed interval of the real line, and $\mathbb{H}$ is a separable Hilbert space, we have noisy discrete observations at time points $t_1,t_2,\dots,t_T$ generated from the model below:
\begin{equation}
h_{k}=g(t_k)+\epsilon_{k},
\end{equation}
where $\epsilon_{k}\overset{\text{i.i.d.}}{\sim} N(0,\sigma^2_0)$, $k=1,\ldots,T$. Let $b(t)=(b_1(t),b_2(t),\dots,b_L(t))^{\top}$ be the vector of the basis functions. Let $\hat{g}(t)=\hat{\beta}^{\top}b(t)$ be the estimator of $g(t)$, where $\hat{\beta} \in \mathbb{R}^L$ is obtained by minimizing the least square loss:
\begin{equation}
\hat{\beta}=\argmin_{\beta \in \mathbb{R}^{L}}\sum^{T}_{k=1}\left(\beta^{\top}b(t_k)-h_k\right)^2.
\end{equation}
We define the design matrix $B$ as
\begin{equation}
B=
\left[
\begin{matrix}
b_{1}(t_{1}) & \cdots & b_{L}(t_{1}) \\
\vdots & \ddots & \vdots \\
b_{1}(t_{T}) & \cdots & b_{L}(t_{T})
\end{matrix}
\right] \in \mathbb{R}^{T \times L},
\end{equation}
so that 
\begin{equation}\label{eq:lsesolu}
\hat{\beta}=\left(B^{\top}B\right)^{-1}B^{\top}h,
\end{equation}
where $h=(h_1,h_2,\dots,h_T)^{\top} \in \mathbb{R}^T$.

We assume that $g(t)=\sum^{\infty}_{m=1} \beta^{*}_m b_m (t)$, and we can decompose $g(t)$ as $g=g^{\shortparallel}+g^{\bot}$, where $g^{\shortparallel} \in {\rm Span}(b)$ and $g^{\bot} \in {\rm Span}(b)^{\bot}$.  Let $\lambda_0 \coloneqq \mathbb{E}[ \Vert g \Vert^2 ]$ and $\lambda^{\bot}_0 \coloneqq \mathbb{E}[ \Vert g^{\bot} \Vert^2 ]$.
It is then easy to check that $\lambda_0=\sum^{\infty}_{m=1} \mathbb{E}[(\beta^{*}_m)^2]$ and $\lambda^{\bot}_0=\sum^{\infty}_{m>L} \mathbb{E}[(\beta^{*}_m)^2]$. We assume that the basis functions $\{b_{l}(t)\}^{\infty}_{l=1}$ make up a complete orthonormal system (CONS) of $\mathbb{H}$, that is, $\overline{{\rm Span}\left(\{b_l\}^{\infty}_{l=1}\right)}=\mathbb{H}$ (see Definition 2.4.11 of \cite{Hsing2015Theoretical}), and have continuous derivative functions with
\begin{equation}\label{eq:Assump3Lemma}
D_{0,b} \coloneqq \sup_{l \geq 1} \sup_{t \in \mathcal{T}} \lvert b_{l}(t)  \rvert < \infty, \qquad
D_{1,b}(l) \coloneqq \sup_{t \in \mathcal{T}} \lvert b^{\prime}_{l}(t)  \rvert < \infty, \qquad D_{1,b,L} \coloneqq \max_{1 \leq l \leq L} D_{1,b}(l).
\end{equation}
We further assume that the observation time points $\{t_{k}:1 \leq k \leq T\}$ satisfy
\begin{equation}\label{eq:Assump4Lemma}
\max_{1 \leq k \leq T+1}\left\vert \frac{t_{k}-t_{(k-1)}}{\lvert \mathcal{T} \rvert} - \frac{1}{T}\right\vert \leq \frac{\zeta_0}{T^2},
\end{equation}
where $t_{0}$ and $t_{(T+1)}$ are endpoints of $\mathcal{T}$ and $\zeta_0$ is a positive constant. We further assume that $\sum^{\infty}_{m=1} \mathbb{E} \left[ ( \beta^*_m )^2 \right] D^2_{1,b} (m) < \infty$, and we define
\begin{equation*}
\psi_4(L) =\sum_{m>L} \mathbb{E} \left[ ( \beta^*_m )^2 \right] D^2_{1,b} (m).
\end{equation*}
Let
\begin{equation*}
\begin{aligned}
\psi_1(T,L) &=\frac{\sigma_{0} L }{\sqrt{ \lambda_{\min}\left( B^{\top} B \right) } }, \qquad \psi_3(L)=\lambda^{\bot}_0/\lambda_0,
\end{aligned}
\end{equation*}
and
\begin{multline*}
\psi_2(T,L) = \frac{1}{ (\lambda^B_{\min})^2 } \left( 18 \lambda_0  \left[ D^2_{0,b} (\zeta_{0}+1)^4 \vert \mathcal{T} \vert^2 D^2_{1,b,L} + 2 D^4_{0,b} (2\zeta_{0}+1)^2 \right] L^2 \psi_3 (L) \right.  \\
\left.  +  D^2_{0,b} (\zeta_{0}+1)^4 \vert \mathcal{T} \vert^2 L^2 \psi_4(L)  \right),
\end{multline*}
Then we have the following theorem.
\begin{theorem}\label{Thm:approxerrsinglefunc}
For any $\delta>0$, we have
\begin{multline}\label{eq:errinqsingle}
P\left( \lVert g-\hat{g} \rVert > \delta \right)\leq  2\exp\left( -\frac{\delta^2}{72\psi^2_{1}(T,L)+6\sqrt{2}\psi_1(T,L)\delta} \right) + L \exp \left( - \frac{ \delta^2 }{ \psi_2(T,L) } \right) \\
+ 2\exp\left(-\frac{\delta^2}{72\lambda_0\psi_3(L)+6\sqrt{2\lambda_0}\sqrt{\psi_3(L)}\delta}\right).
\end{multline}
\end{theorem}

\begin{proof}
	For a fixed $g$, since $\overline{{\rm Span}\left(\{b_l\}^{\infty}_{l=1}\right)}=\mathbb{H}$, we can assume that $g(t)=\sum^{\infty}_{l=1} \beta^{*}_l b_l(t)$ where $\beta^*_l=\langle g,b_l \rangle=\int_{\mathcal{T}} g(t) b_l(t) dt$. Let $\beta^{*}=(\beta^{*}_1,\cdots,\beta^{*}_L)^{\top}\in\mathbb{R}^{L}$. We have $g^{\shortparallel}(t)=(\beta^{*})^{\top}b(t)=\sum^{L}_{l=1}\beta^{*}_{l}b_{l}(t)$ and $g^{\bot}(t)=\sum_{l>L}\beta^{*}_{l}b_{l}(t)$.
	Thus, we have
	\begin{equation*}
		h_k=g(t_k)+\epsilon_{k}=(\beta^{*})^{\top}b(t_k)+g^{\bot}(t_k)+\epsilon_{k}.
	\end{equation*}
	Let $h^{\bot}=\left( g^{\bot}(t_1),g^{\bot}(t_2),\dots,g^{\bot}(t_T) \right)^{\top}$ and $\epsilon=\left(\epsilon_1,\epsilon_2,\dots,\epsilon_T\right)^{\top}$, so that $h=B\beta^{*}+h^{\bot}+\epsilon$. Then,
    $\mathbb{E}(\mathbb{\hat{\beta}})=\beta^{*}+\left(B^{\top}B\right)^{-1}B^{\top}h^{\bot}$ and
	\begin{multline*}
		\hat{g}(t)-g(t)
		=\hat{g}(t)-g^{\shortparallel}(t)-g^{\bot}(t)
		=\hat{g}(t)-(\beta^{*})^{\top}b(t)-g^{\bot}(t)\\
		=\left( \hat{\beta}-\mathbb{E}(\hat{\beta}) \right)^{\top}b(t)+\left(\left(B^{\top}B\right)^{-1}B^{\top}h^{\bot}\right)^{\top}b(t)-g^{\bot}(t).
	\end{multline*}
	By Lemma~\ref{lemma:vecfuncnorm}, we then have
	\begin{equation*}
	\begin{aligned}
	\lVert \hat{g}-g \rVert 
	&\leq \left\lVert \left( \hat{\beta}-\mathbb{E}(\hat{\beta}) \right)^{\top}b(t) \right\rVert
	+
	\left\lVert \left(\left(B^{\top}B\right)^{-1}B^{\top}h^{\bot}\right)^{\top}b(t) \right\rVert 
	+ 
	\left\lVert g^{\bot} \right\rVert\\
	&\leq \left\lvert \hat{\beta}-\mathbb{E}(\hat{\beta}) \right\rvert_2 \times \left\lVert b \right\rVert_{\mathcal{L}^2,2} + \left\lvert \left(B^{\top}B\right)^{-1}B^{\top}h^{\bot} \right\rvert_2 \times \left\lVert b \right\rVert_{\mathcal{L}^2,2} 
	+ \left\lVert g^{\bot} \right\rVert\\
	&\leq \left\lvert \hat{\beta}-\mathbb{E}(\hat{\beta}) \right\rvert_2 \times \left\lVert b \right\rVert_{\mathcal{L}^2,2} + \frac{1}{\lambda_{\min}(B^{\top}B)} \times \left\lvert B^{\top}h^{\bot} \right\rvert_2 \times \left\lVert b \right\rVert_{\mathcal{L}^2,2} 
	+ \left\lVert g^{\bot} \right\rVert.
	\end{aligned}
	\end{equation*}
	Let
	\begin{equation}\label{eq:J1J2J3def}
	J_1=\left\lvert \hat{\beta}-\mathbb{E}(\hat{\beta}) \right\rvert_2 \times \left\lVert b \right\rVert_{\mathcal{L}^2,2}, 
	\quad
	J_2=\frac{1}{\lambda_{\min}(B^{\top}B)} \times \lvert B^{\top}h^{\bot} \rvert_2 \times \lVert b \rVert_{\mathcal{L}^2,2},\quad
	J_3 =\lVert g^{\bot} \rVert,
	\end{equation}
	where $\lvert \mathcal{T} \rvert$ denotes the length of the interval, then $\lVert \hat{g}-g \rVert \leq J_1+J_2+J_3$.
	This equation holds with probability one, since it holds for any $g \in \mathbb{H}$. We then bound $J_1$, $J_2$, and $J_3$ individually.
		
	We bound $J_1$. Recall that $\lVert b \rVert_{\mathcal{L}^2,2}=\sqrt{L}$ and $\psi_1(T,L)=\sigma_{0}\lVert b \rVert_{\mathcal{L}^2,2}\sqrt{L}/\sqrt{\lambda_{\min}\left(B^{\top}B\right)}$. Treating $g$ as fixed, we have
	$\hat{\beta}\sim N_{L}\left(\mathbb{E}(\hat{\beta}),\sigma^{2}_0 \left(B^{\top}B\right)^{-1} \right)$ and
	\begin{equation*}
		\frac{1}{\sigma_0} \left(B^{\top}B\right)^{1/2} \left(\hat{\beta}-\mathbb{E}(\hat{\beta})\right)\sim N_L\left(0,I_L\right).
	\end{equation*}
	Since
	\begin{multline*}
	J_1=\left\lvert \hat{\beta}-\mathbb{E}(\hat{\beta}) \right\rvert_2 \times \left\lVert b \right\rVert_{\mathcal{L}^2,2}
	= \left\lvert \left(B^{\top}B\right)^{-1/2} \left(B^{\top}B\right)^{1/2}\left(\hat{\beta}-\mathbb{E}(\hat{\beta}) \right)\right\rvert_2 \times \lVert b \rVert_{\mathcal{L}^2,2}\\
    \leq \frac{\sigma_0 \left\lVert b \right\rVert_{\mathcal{L}^2,2}}{\sqrt{\lambda_{\min}\left(B^{\top}B\right)}} \left\lvert \frac{1}{\sigma_0} \left(B^{\top}B\right)^{1/2}\left(\hat{\beta}-\mathbb{E}(\hat{\beta}) \right)\right\rvert_2,
	\end{multline*}
	we have
	\begin{equation}
	\label{eq:J1bound}
		\begin{aligned}
		P(J_1>\delta) 
		& \leq P\left( \left\lvert\frac{1}{\sigma_0} \left(B^{\top}B\right)^{1/2}\left(\hat{\beta}-\mathbb{E}(\hat{\beta}) \right)\right\rvert_2 > \frac{\delta}{\sigma_0 \lVert b \rVert_{\mathcal{L}^2,2}/\sqrt{\lambda_{\min}\left(B^{\top}B\right)}} \right)\\
		&\overset{(i)}{\leq}2\exp\left(-\frac{\delta^2}{8\psi^{2}_{1}(T,L)+2\sqrt{2}\psi_{1}(T,L)\delta}\right), 
		\end{aligned}
	\end{equation}
	where $(i)$ follows Lemma~\ref{lemma:boundnormalnorm}.
	The bound does not depend on $g$, so it holds when $g$ is also random.
	
    Next, we bound $J_2$. Let $(B^{\top}h^{\bot})_{l}$ denote the $l$-th element of the vector $B^{\top}h^{\bot}$. Then
	\[
		(B^{\top}h^{\bot})_{l}=\sum^{T}_{k=1}b_{l}(t_k)g^{\bot}(t_k)=\sum_{m>L} \beta^{*}_m \sum^{T}_{k=1}b_{l}(t_k)b_m(t_k)
	\]
	and $(B^{\top}h^{\bot})_{l}$ follows a mean zero Gaussian distribution, since $g$ is a Gaussian random function with mean zero. Furthermore,
	\begin{equation}\label{eq:Bhl-var}
		\mathbb{E} \left[ (B^{\top}h^{\bot})^2_{l} \right] = \sum_{m>L} \mathbb{E} \left[ \beta^{*2}_m \right] \left( \sum^{T}_{k=1}b_{l}(t_k)b_m(t_k) \right)^2
	\end{equation}
	
    From the definition of $D_{0,b}$, $D_{1,b}(\cdot)$,
    for any $l < m$, we have 
    \begin{align*}
    \sup_{t \in \mathcal{T}} (b_l(t) b_m(t)) & \leq D^2_{0,b}, \\
    \sup_{t \in \mathcal{T}} (b_l(t) b_m(t))^{\prime} = \sup_{t \in \mathcal{T}} \{ b^{\prime}_l(t) b_m(t) + b_l(t) b^{\prime}_m(t) \} \leq D_{0,b} ( D_{1,b}(l) + D_{1,b}(m)  ).
    \end{align*}
    Note that $\int_{\mathcal{T}} b_l(t) b_m(t) dt=0$, $l < m$. Then, by Lemma~\ref{lemma:integralapprox}, for all $1 \leq l < m < \infty$, we have
	\begin{equation*}
	\begin{aligned}
		\left\vert \frac{1}{T}\sum^{T}_{k=1}b_{l}(t_k)b_m(t_k)\right\vert 
		& = \left\vert \frac{1}{T}\sum^{T}_{k=1}b_{l}(t_k)b_m(t_k)-\frac{1}{\vert \mathcal{T} \vert}\int_{\mathcal{T}}b_{l}(t)b_m(t)dt \right\vert \\
	& 	\leq  \frac{ D_{0,b} ( D_{1,b}(l) + D_{1,b}(m)  ) (\zeta_{0}+1)^2\vert \mathcal{T} \vert/2+ D^2_{0,b} (2\zeta_{0}+1)}{T},
	\end{aligned}
	\end{equation*}
	which implies that
	\begin{equation*}
		\left\vert \sum^{T}_{k=1}b_{l}(t_k)b_m(t_k)\right\vert \leq \frac{1}{2}  D_{0,b} (\zeta_{0}+1)^2\vert \mathcal{T} \vert ( D_{1,b}(l) + D_{1,b}(m)  ) + D^2_{0,b} (2\zeta_{0}+1).
	\end{equation*}
	Then,
	\begin{equation*}
	\begin{aligned}
		\left( \sum^{T}_{k=1}b_{l}(t_k)b_m(t_k) \right)^2 
		& \leq \left( \frac{1}{2}  D_{0,b} (\zeta_{0}+1)^2\vert \mathcal{T} \vert ( D_{1,b}(l) + D_{1,b}(m)  ) + D^2_{0,b} (2\zeta_{0}+1) \right)^2 \\
		& \leq \frac{1}{2} D^2_{0,b} (\zeta_{0}+1)^4 \vert \mathcal{T} \vert^2 ( D_{1,b}(l) + D_{1,b}(m)  )^2 + 2 D^4_{0,b} (2\zeta_{0}+1)^2 \\
		& \leq D^2_{0,b} (\zeta_{0}+1)^4 \vert \mathcal{T} \vert^2 ( D^2_{1,b}(l) + D^2_{1,b}(m)  ) + 2 D^4_{0,b} (2\zeta_{0}+1)^2
	\end{aligned}
	\end{equation*}
	and, by~\eqref{eq:Bhl-var},
	\begin{equation*}
	\begin{aligned}
		\mathbb{E} \left[ (B^{\top}h^{\bot})^2_{l} \right] & \leq \left[ D^2_{0,b} (\zeta_{0}+1)^4 \vert \mathcal{T} \vert^2 D^2_{1,b}(l)  + 2 D^4_{0,b} (2\zeta_{0}+1)^2 \right] \sum_{m>L} \mathbb{E} \left[ \beta^{*2}_m \right] \\
		& \quad +  D^2_{0,b} (\zeta_{0}+1)^4 \vert \mathcal{T} \vert^2 \sum_{m>L} \mathbb{E} \left[ \beta^{*2}_m \right] D^2_{1,b}(m) \\
		& \leq \left[ D^2_{0,b} (\zeta_{0}+1)^4 \vert \mathcal{T} \vert^2 D^2_{1,b}(l)  + 2 D^4_{0,b} (2\zeta_{0}+1)^2 \right] \lambda^{\bot}_0 +  D^2_{0,b} (\zeta_{0}+1)^4 \vert \mathcal{T} \vert^2 \psi_4(L)\\
		& \leq \left[ D^2_{0,b} (\zeta_{0}+1)^4 \vert \mathcal{T} \vert^2 D^2_{1,b,L}  + 2 D^4_{0,b} (2\zeta_{0}+1)^2 \right] \lambda^{\bot}_0 +  D^2_{0,b} (\zeta_{0}+1)^4 \vert \mathcal{T} \vert^2 \psi_4(L) \\
		& = \lambda_0  \left[ D^2_{0,b} (\zeta_{0}+1)^4 \vert \mathcal{T} \vert^2 D^2_{1,b,L}  + 2 D^4_{0,b} (2\zeta_{0}+1)^2 \right] \psi_3 (L) \\
		& \quad +  D^2_{0,b} (\zeta_{0}+1)^4 \vert \mathcal{T} \vert^2 \psi_4(L).
	\end{aligned}
	\end{equation*}
    Using a tail bound for Gaussian random variable (e.g., Section 2.1.2 of~\cite{wainwright2019high}), we have
	\begin{equation}\label{eq:J2-bound-new}
	\begin{aligned}
		\mathbb{P} \left( J_2 > \delta \right) & \leq \mathbb{P} \left( \lvert B^{\top}h^{\bot} \rvert_2 > \frac{ \lambda^B_{\min} \delta}{\sqrt{L}} \right) \leq \mathbb{P} \left(  \max_{1 \leq l \leq L} (B^{\top}h^{\bot})_{l} > \frac{ \lambda^B_{\min} \delta}{L} \right) \\
		& \leq L \exp \left( - \frac{9\delta^2}{\psi_2(T,L)} \right).
	\end{aligned}
	\end{equation}    
    
    Finally, we bound $J_3$. By Lemma~\ref{lemma:2kmomentGuassFunc} and the definition of $\psi_{3}(L)$, we have
	$\mathbb{E}\left[ \Vert g^{\bot} \Vert^{2k} \right] \leq (2\lambda_0\psi_{3}(L))^{k}k!$.
	By Jensen's inequality, we have
	\begin{equation*}
		\mathbb{E}\left[ \Vert g^{\bot} \Vert^{k} \right]=\mathbb{E}\left[ \sqrt{\Vert g^{\bot} \Vert^{2k}} \right]\leq \sqrt{\mathbb{E}\left[ \Vert g^{\bot} \Vert^{2k} \right]}\leq \left(\sqrt{2\lambda_0\psi_{3}(L)}\right)^{k} k!.
	\end{equation*}
	Thus, by Lemma~\ref{lemma:HilbertBernstein}, we have
		\begin{equation}\label{eq:J3bound}
		P\left( J_3 > \delta \right)=P\left( \Vert g^{\bot} \Vert > \delta \right) \leq 2\exp\left(-\frac{\delta^2}{8\lambda_0\psi_3(L)+2\sqrt{2\lambda_0}\sqrt{\psi_3(L)}\delta}\right).
		\end{equation}
		
	The final result follows from \eqref{eq:J1bound}, \eqref{eq:J2-bound-new}, and \eqref{eq:J3bound}.
	\end{proof}

\newpage

\section{Lemmas and Their Proofs}\label{app:prooflemma}

In this section, we introduce some useful lemmas along with their proofs.

\begin{lemma}\label{lema:S-ind-joint}
Let $\sigma_{\max} = \max\{ \vert \Sigma^{X,M} \vert_\infty , \ \vert \Sigma^{Y,M} \vert_\infty\}$. 
Suppose that 
\begin{equation}\label{eq:lema:S-ind-joint-1}
\vert S^{X,M}-\Sigma^{X,M} \vert_{\infty}  \leq \delta,
\qquad
\vert S^{Y,M}-\Sigma^{Y,M} \vert_{\infty} \leq \delta,
\end{equation}
for some $\delta \geq 0$. Then 
\begin{align}\label{eq:lema:S-ind-joint-2}
&\vert (S^{Y,M} \otimes{S^{X,M}})-(\Sigma^{Y,M}\otimes{\Sigma^{X,M}})\vert_{\infty}\leq \delta^2 + 2\delta\sigma_{\max},
\intertext{and}
\label{eq:lema:S-ind-joint-3}
&|\vect{(S^{Y,M}-S^{X,M})}-\vect{(\Sigma^{Y,M}-\Sigma^{X,M})}|_{\infty}\leq 2\delta.
\end{align}
\end{lemma}
\begin{proof}
Note that for any $(j,l), (j^{\prime},l^{\prime}) \in V^2$ and $1 \leq k,k^{\prime},m,m^{\prime} \leq M$, by \eqref{eq:lema:S-ind-joint-1}, we have
\begin{equation*}
\begin{aligned}
& \left\vert S^{X,M}_{jl,km}S^{Y,M}_{j^{\prime}l^{\prime},k^{\prime}m^{\prime}}-\Sigma^{X,M}_{jl,km}\Sigma^{Y,M}_{j^{\prime}l^{\prime},k^{\prime}m^{\prime}} \right\vert \\
& \leq \left\vert S^{X,M}_{jl,km}-\Sigma^{X,M}_{jl,km} \right\vert \cdot \left\vert S^{Y,M}_{j^{\prime}l^{\prime},k^{\prime}m^{\prime}}-\Sigma^{Y,M}_{j^{\prime}l^{\prime},k^{\prime}m^{\prime}} \right\vert  + \left\vert \Sigma^{X,M}_{jl,km} \right\vert \cdot \left\vert S^{Y,M}_{j^{\prime}l^{\prime},k^{\prime}m^{\prime}}-\Sigma^{Y,M}_{j^{\prime}l^{\prime},k^{\prime}m^{\prime}} \right\vert\\
& \quad + \left\vert \Sigma^{Y,M}_{j^{\prime}l^{\prime},k^{\prime}m^{\prime}} \right\vert \cdot \left\vert S^{X,M}_{jl,km}-\Sigma^{X,M}_{jl,km} \right\vert \\
& \leq \left\vert S^{X,M}-\Sigma^{X,M} \right\vert_{\infty} \left\vert S^{Y,M}-\Sigma^{Y,M} \right\vert_{\infty} + \sigma_{\max} \left\vert S^{Y,M}-\Sigma^{Y,M} \right\vert_{\infty} + \sigma_{\max} \left\vert S^{X,M}-\Sigma^{X,M} \right\vert_{\infty} \\
& \leq \delta^2 +2 \delta \sigma_{\max}.
\end{aligned}
\end{equation*}
For \eqref{eq:lema:S-ind-joint-3}, note that
\begin{equation*}
\begin{aligned}
\left\vert \vect{(S^{Y,M}-S^{X,M})}-\vect{(\Sigma^{Y,M}-\Sigma^{X,M})} \right\vert_{\infty} & = \left\vert (S^{X,M}-\Sigma^{X,M})-(S^{Y,M}-\Sigma^{Y,M}) \right\vert_{\infty} \\
&\leq|S^{X,M}-\Sigma^{X,M}|_{\infty}+|S^{Y,M}-\Sigma^{Y,M}|_{\infty}\\
&\leq 2\delta.
\end{aligned}
\end{equation*}
This completes the proof.
\end{proof}

\begin{lemma}\label{lemma:LemmaOptFFGL}
	Given $Z^{(1)},Z^{(2)},A^{(1)},A^{(2)}\in \mathbb{R}^{M \times M}$ and $\lambda>0$, let $\{\hat{Z}^{(1)}, \hat{Z}^{(2)}\}$ denote the solution of
	\begin{equation}\label{eq:LemmaOptFFGLobj}
	\argmin_{\{Z^{(1)}, Z^{(2)} \}}\; \frac{1}{2}\sum^{2}_{q=1}\Vert Z^{(q)}-A^{(q)} \Vert^{2}_{\text{F}}+\lambda\Vert Z^{(1)}-Z^{(2)} \Vert_{\text{F}}.
	\end{equation}
	If $\Vert A^{(1)}-A^{(2)} \Vert_{\text{F}}\leq 2\lambda$, then
	\begin{equation}\label{eq:LemmaOptFFGLsol1}
	\hat{Z}^{(1)}=\hat{Z}^{(2)}=\frac{1}{2}\left(A^{(1)}+A^{(2)}\right).
	\end{equation}
	If $\Vert A^{(1)}-A^{(2)} \Vert_{\text{F}}>2\lambda$, then
	\begin{equation}\label{eq:LemmaOptFFGLsol2}
	\begin{aligned}
	&\hat{Z}^{(1)}=A^{(1)}-\frac{\lambda}{\Vert A^{(1)}-A^{(2)} \Vert_{\text{F}}}\left(A^{(1)}-A^{(2)}\right),\\
	&\hat{Z}^{(2)}=A^{(2)}+\frac{\lambda}{\Vert A^{(1)}-A^{(2)} \Vert_{\text{F}}}\left(A^{(1)}-A^{(2)}\right).
	\end{aligned}
	\end{equation}
\end{lemma}
	\begin{proof}
		The subdifferential of the objective function in \eqref{eq:LemmaOptFFGLobj} is
		\begin{align}
		\label{eq:LemmaOptFFGLeq1}
		G^{(1)}(Z^{(1)},Z^{(2)}) 
		&\coloneqq Z^{(1)}-A^{(1)}+\lambda T(Z^{(1)},Z^{(2)}), \\
		\label{eq:LemmaOptFFGLeq2}
		G^{(2)}(Z^{(1)},Z^{(2)}) & \coloneqq Z^{(2)}-A^{(2)}-\lambda T(Z^{(1)},Z^{(2)}),
		\end{align}
		where
		\begin{equation}\label{eq:LemmaOptFFGLeq3}
		T(Z^{(1)},Z^{(2)})=\left\{
		\begin{array}{ll}
		\frac{Z^{(1)}-Z^{(2)}}{\Vert Z^{(1)}-Z^{(2)} \Vert_{\text{F}}}& \text{if}\; Z^{(1)}\neq Z^{(2)}\\
		\left\{T\in \mathbb{R}^{M \times M}: \Vert T \Vert_{\text{F}}\leq 1 \right\}&\text{if}\; Z^{(1)}=Z^{(2)}
		\end{array}
		\right..
		\end{equation}
		The optimality condition is $0\in G^{(q)}(Z^{(1)},Z^{(2)})$.
		
		\textbf{Claim} We have $\hat{Z}^{(1)}\neq \hat{Z}^{(2)}$ if and only if $\Vert A^{(1)}-A^{(2)} \Vert_{\text{F}}>2\lambda$.
		
		We first prove the necessity, that is, when $\hat{Z}^{(1)}\neq \hat{Z}^{(2)}$, then $\Vert A^{(1)}-A^{(2)} \Vert_{\text{F}}>2\lambda$. By the optimality condition, we have
		\begin{equation*}
		\hat{Z}^{(1)}-\hat{Z}^{(2)}-\left(A^{(1)}-A^{(2)} \right)-2\lambda\frac{\hat{Z}^{(1)}-\hat{Z}^{(2)}}{\Vert \hat{Z}^{(1)}-\hat{Z}^{(2)} \Vert_{\text{F}}}=0,
		\end{equation*}
		which implies that
		\begin{equation*}
		\Vert A^{(1)}-A^{(2)} \Vert_{\text{F}}=2\lambda + \Vert \hat{Z}^{(1)}-\hat{Z}^{(2)} \Vert_{\text{F}}>2\lambda.
		\end{equation*}
		We then prove the sufficiency, that is, when $\Vert A^{(1)}-A^{(2)} \Vert_{\text{F}}>2\lambda$, then $\hat{Z}^{(1)}\neq \hat{Z}^{(2)}$. By the optimality condition, we have $\hat{Z}^{(1)}+\hat{Z}^{(2)}=A^{(1)}+A^{(2)}$.
		If $\hat{Z}^{(1)}=\hat{Z}^{(2)}$, then
		$\hat{Z}^{(1)}=\hat{Z}^{(2)}=(A^{(1)}+A^{(2)})/{2}$.
		Furthermore, $\Vert \hat{Z}^{(1)}-A^{(1)} \Vert_{\text{F}}=\Vert A^{(1)}-A^{(2)}\Vert_{\text{F}}/2=\lambda \Vert T(\hat{Z}^{(1)},\hat{Z}^{(2)}) \Vert_{\text{F}}\leq \lambda$,
		which implies that $\Vert A^{(1)}-A^{(2)} \Vert_{\text{F}} \leq 2\lambda$. This contradicts the assumption that $\Vert A^{(1)}-A^{(2)} \Vert_{\text{F}}>2\lambda$. Thus, we must have $\hat{Z}^{(1)}\neq \hat{Z}^{(2)}$.
		
		By proving the claim, we have already established \eqref{eq:LemmaOptFFGLsol1}. We now prove \eqref{eq:LemmaOptFFGLsol2}. When $\Vert A^{(1)}-A^{(2)} \Vert_{\text{F}}>2\lambda$, according to the claim above, we must have $\hat{Z}^{(1)}\neq \hat{Z}^{(2)}$. Then
		\begin{align}\label{eq:LemmaOptFFGLeq5}
		&\hat{Z}^{(1)}-A^{(1)}+\frac{\lambda}{\Vert \hat{Z}^{(1)}-\hat{Z}^{(2)} \Vert_{\text{F}}}\left(\hat{Z}^{(1)}-\hat{Z}^{(2)}\right)=0,\\
		\label{eq:LemmaOptFFGLeq6}
		&\hat{Z}^{(2)}-A^{(2)}-\frac{\lambda}{\Vert \hat{Z}^{(1)}-\hat{Z}^{(2)} \Vert_{\text{F}}}\left(\hat{Z}^{(1)}-\hat{Z}^{(2)}\right)=0,
		\end{align}
		which implies that $\hat{Z}^{(1)}-\hat{Z}^{(2)}=\alpha\cdot\left(A^{(1)}-A^{(2)}\right)$,
		where $\alpha$ is a constant. The result in \eqref{eq:LemmaOptFFGLsol2} follows by substituting the relationship back into the above display.
	\end{proof}

\begin{lemma}{\label{lemma:QRleq12Norm}}
	Let $\vert \cdot\vert_{1,2}$ be defined as in \eqref{con:Rnorm}, where $\mathcal{G}=\{G_{t}\}_{t=1, \ldots, N_{\mathcal{G}}}$ is a set of indices. For any matrix $A\in{\mathbb{R}^{p^{2}M^{2}\times{p^{2}M^{2}}}}$ and $\theta\in{\mathbb{R}^{p^{2}M^{2}}}$, we have $|\theta^{\top}A\theta|\leq{M^{2}|A|_{\infty}\vert\theta\vert^{2}_{1,2}}$.
\end{lemma}	
	\begin{proof}
	By direct calculation, we have
		\begingroup
		\allowdisplaybreaks
		\begin{multline*}
		|\theta^{\top}A\theta|
		\leq{\sum_{i}\sum_{j}|A_{ij}\theta_{i}\theta_{j}|}
		\leq{|A|_{\infty}\left(\sum_{i}|\theta_{i}|\right)^{2}}
		=|A|_{\infty}\left(\sum_{t=1}^{N_{\mathcal{G}}}\sum_{k\in{G_{t}}}|\theta_{k}|\right)^{2} \\
		=|A|_{\infty}\left(\sum_{t=1}^{N_{\mathcal{G}}}\vert\theta_{G_{t}}\vert_{1}\right)^{2}
		\leq{|A|_{\infty}\left(\sum_{t=1}^{N_{\mathcal{G}}}M\vert\theta_{G_{t}}\vert_{2}\right)^{2}}
		=M^{2}|A|_{\infty}\vert\theta\vert^{2}_{1,2},
		\end{multline*}
		\endgroup
		where we use that for any vector $v\in{\mathbb{R}^{n}}$, $\vert v \vert_{1}\leq{\sqrt{n}\vert v\vert_{2}}$.	
	\end{proof}

\begin{lemma}{\label{lemma:L12NormleqL2Norm}}
	Suppose that $\mathcal{M}$ is defined as in \eqref{con:Msubspace}. For any $\theta\in{\mathcal{M}}$, we have $\vert\theta\vert_{1,2}\leq{\sqrt{s}}\vert \theta\vert_{2}$. Furthermore, for $\Psi(\mathcal{M})$ as defined in \eqref{con:SubspaceCompConst}, we have $\Psi(\mathcal{M})=\sqrt{s}$.
\end{lemma}
	\begin{proof}
		By the definitions of $\mathcal{M}$ and $\vert\cdot\vert_{1,2}$, we have
		\begin{equation*}
		\begin{aligned}
		\vert\theta\vert_{1,2}=\sum_{t\in{S_{\mathcal{G}}}}\vert\theta_{G_{t}}\vert_{2}+\sum_{t\notin{S_{\mathcal{G}}}}\vert\theta_{G_{t}}\vert_{2}
		=\sum_{t\in{S_{\mathcal{G}}}}\vert\theta_{G_{t}}\vert_{2}
		\leq{\sqrt{s}}\left(\sum_{t\in{S_{\mathcal{G}}}}\vert\theta_{G_{t}}\vert^{2}_{2}\right)^{\frac{1}{2}}
		=\sqrt{s}\vert\theta\vert_{2}.
		\end{aligned}
		\end{equation*}
To show $\Psi(\mathcal{M})=\sqrt{s}$, it suffices to show that the upper bound above can be achieved. Select $\theta\in{\mathbb{R}^{p^{2}M^{2}}}$ so that $\vert\theta_{G_{t}}\vert_{2}=c$, $\forall{t\in{S_{\mathcal{G}}}}$, where $c$ is some positive constant. This implies that $\vert\theta\vert_{1,2}=sc$ and $\vert\theta\vert_{2}=\sqrt{s}c$ so that $\vert\theta\vert_{1,2}=\sqrt{s}\vert\theta\vert_{2}$. Thus, $\Psi(\mathcal{M})=\sqrt{s}$.
	\end{proof}

\begin{lemma}{\label{lemma:dualnorm}}
	Let $\mathcal{R}(\cdot)$ be the norm defined in \eqref{con:Rnorm}. Its dual norm $\mathcal{R}^{*}(\cdot)$, defined in \eqref{con:Rdualnorm}, is
	\begin{equation}
	\mathcal{R}^{*}(v)\;=\;\max_{t=1,\ldots,N_{\mathcal{G}}}\vert v_{G_{t}}\vert_{2}.
	\end{equation}
\end{lemma}
	\begin{proof}
		For any $u$ satisfying $\vert u\vert_{1,2}\leq{1}$ and $v\in{\mathbb{R}^{p^{2}M^{2}}}$, we have
		\begin{multline*}
		\langle{v,u}\rangle
		=\sum_{t=1}^{N_{\mathcal{G}}}\langle{v_{G_{t}},u_{G_{t}}}\rangle
		\leq{\sum_{t=1}^{N_{\mathcal{G}}}\vert v_{G_{t}}\vert_{2}\vert u_{G_{t}}\vert_{2}}
		\leq \left(\max_{t=1,2,\cdots,N_{\mathcal{G}}}\vert v_{G_{t}}\vert_{2}\right)\sum_{t=1}^{N_{\mathcal{G}}}\vert u_{G_{t}}\vert_{2}\\
		=\left(\max_{t=1,2,\cdots,N_{\mathcal{G}}}\vert v_{G_{t}}\vert_{2}\right)\vert u\vert_{1,2}
		\leq{\max_{t=1,2,\cdots,N_{\mathcal{G}}}\vert v_{G_{t}}\vert_{2}}.
		\end{multline*}
		We show that this upper bound can be obtained. Let $t^{*}=\argmax_{t=1,2,\cdots,N_{\mathcal{G}}}\vert v_{G_{t}}\vert$ and set $u$ such that
		\begin{equation*}
		u_{G_{t}}=
		\left\{
		\begin{array}{ll}
		0 &  t\neq{t^{*}} \\
		\frac{v_{G_{t^{*}}}}{\vert v_{G_{t^{*}}}
			\vert_{2}} & {t={t^{*}}}
		\end{array}
		\right..
		\end{equation*}
		Then $\vert u\vert_{1,2}=1$ and $\langle{v,u}\rangle=\vert v_{G_{t^{*}}}\vert_{2}=\max_{t=1,\ldots,N_{\mathcal{G}}}\vert v_{G_{t}}\vert_{2}$.
	\end{proof}

\begin{lemma}\label{lemma:I1bound}
	Given that $A1$-$A5$ hold, we have $\lvert I_1 \rvert \leq \delta/16$  for all $1 \leq j,l \leq p$, $ \ 1 \leq k,m \leq M$.
\end{lemma}
\begin{proof}
	This follows directly from the assumption that $A_5$ is true.
\end{proof}

\begin{lemma}\label{lemma:I2bound}
	Given that $A1$-$A5$ hold, we have $\lvert I_2 \rvert \leq \delta/16$  for all $1 \leq j,l \leq p$, $ \ 1 \leq k,m \leq M$.
\end{lemma}
\begin{proof}
		We have
		\begingroup
		\allowdisplaybreaks
		\begin{align*}
		\lvert I_2 \rvert &= \left\lvert \langle \frac{1}{n}\sum^{n}_{i=1}a_{ijk}(\hat{g}_{il}-g_{il}),\phi_{lm} \rangle \right\rvert
		\leq \left\lVert \frac{1}{n}\sum^{n}_{i=1}a_{ijk}(\hat{g}_{il}-g_{il}) \right\rVert
		\\
		& \overset{(i)}{\leq}  \sqrt{\frac{1}{n} \sum^{n}_{i=1}a^{2}_{ijk}}\sqrt{\frac{1}{n} \sum^{n}_{i=1}\lVert \hat{g}_{il}-g_{il} \rVert^2}
		\overset{(ii)}{\leq} \delta_{1} \sqrt{\frac{1}{n} \sum^{n}_{i=1}a^{2}_{ijk}}
		=\delta_{1} \lambda^{1/2}_{jk} \sqrt{\frac{1}{n} \sum^{n}_{i=1}\xi^{2}_{ijk}}\\
		&\overset{(iii)}{\leq} \sqrt{\frac{3}{2}} \delta_{1} \lambda^{1/2}_{jk}
		\leq \sqrt{\frac{3}{2}} \sqrt{d_1} \delta_{1} k^{-\beta/2}
		\leq \sqrt{\frac{3}{2}} \sqrt{d_1} \delta_{1},
		\end{align*}
		\endgroup
		where $(i)$ follows Lemma~\ref{lemma:vecfuncnorm}, $(ii)$ follows $A_1$, $(iii)$ follows $A_3$. From the definition of $d_0$, we have $\lvert I_2 \rvert \leq \sqrt{{3}/{2}} d_0 \delta_{1}$. Since
		\begin{equation}\label{eq:I2boundeq2}
		\delta_1=\delta/\left(144d^{2}_{0}M^{1+\beta}\sqrt{3\lambda_{0,\max}}\right)\leq \delta/(8\sqrt{6}d_{0}),
		\end{equation}
		 we have
		\begin{equation}\label{eq:I2boundeq3}
		    \sqrt{\frac{3}{2}}d_0\delta_1\leq \sqrt{\frac{3}{2}} d_0 \cdot \frac{\delta}{8\sqrt{6}d_{0}}=\frac{\delta}{16}.
		\end{equation}
		Thus, $\lvert I_2 \rvert \leq {\delta}/{16}$.
\end{proof}

\begin{lemma}\label{lemma:I3bound}
	Given that $A1$-$A5$ hold, we have $\lvert I_3 \rvert \leq \delta/16$  for all $1 \leq j,l \leq p$, $ \ 1 \leq k,m \leq M$.
\end{lemma}
	\begin{proof}
		We have
		\begingroup
		\allowdisplaybreaks
		\begin{align*}
		\lvert I_3 \rvert 
		& = \left\lvert \langle \frac{1}{n}\sum^{n}_{i=1}a_{ijk}g_{il},\hat{\phi}_{lm}-\phi_{lm} \rangle \right\rvert
		\leq \lambda^{1/2}_{jk} \left\lVert \frac{1}{n}\sum^{n}_{i=1}\xi_{ijk}g_{il} \right\rVert \left\lVert \hat{\phi}_{lm}-\phi_{lm} \right\rVert\\
		& \overset{(i)}{\leq} \lambda^{1/2}_{jk} \left( \frac{1}{n}\sum^{n}_{i=1} \xi^{2}_{ijk} \right)^{1/2} \left( \frac{1}{n}\sum^{n}_{i=1} \lVert g_{il} \rVert^2 \right)^{1/2} \lVert \hat{\phi}_{lm}-\phi_{lm} \rVert\\
		&\overset{(ii)}{\leq} \lambda^{1/2}_{jk} \left( \frac{1}{n}\sum^{n}_{i=1} \xi^{2}_{ijk} \right)^{1/2} \left( \frac{1}{n}\sum^{n}_{i=1} \lVert g_{il} \rVert^2 \right)^{1/2} d_{lm} \lVert \hat{K}_{ll}-K_{ll} \rVert_{\text{HS}},
		\end{align*}
		\endgroup
		where $(i)$ follows Lemma~\ref{lemma:vecfuncnorm}, and $(ii)$ follows Lemma~\ref{lemma:EigenfuncEstBound}. Since $\lambda^{1/2}_{jk}\leq \sqrt{d_1}k^{-\beta/2}$, $d_{lm}\leq d_2 m^{1+\beta}$, and $A_2$-$A_4$ hold, we have
		\begin{equation}\label{eq:I3boundeq}
		\begin{aligned}
		\lvert I_3 \rvert \leq \sqrt{d_1}d_2 k^{-\beta/2} m^{1+\beta} \sqrt{\frac{3}{2}} \sqrt{2\lambda_{0,\max}} \delta_2
		\leq d^{2}_{0} M^{1+\beta} \sqrt{3\lambda_{0,\max}} \delta_2.
		\end{aligned}
		\end{equation}
		By the definition of $\delta_2$, we have
		\begin{equation}\label{eq:I3boundeq2}
		d^{2}_{0} M^{1+\beta} \sqrt{3\lambda_{0,\max}} \delta_2 \leq d^{2}_{0} M^{1+\beta} \sqrt{3\lambda_{0,\max}} \times \frac{\delta}{16d^{2}_{0}M^{1+\beta}\sqrt{3\lambda_{0,\max}}}
		=\frac{\delta}{16}.
		\end{equation}
		Thus, $\lvert I_3 \rvert \leq {\delta}/{16}$.
	\end{proof}

\begin{lemma}\label{lemma:I4bound}
	Given that $A1$-$A5$ hold, we have $\lvert I_4 \rvert \leq \delta/16$  for all $1 \leq j,l \leq p$, $ \ 1 \leq k,m \leq M$.
\end{lemma}
	\begin{proof}
		We have
		\begingroup
		\allowdisplaybreaks
		\begin{align*}
		\lvert I_4 \rvert 
		& \leq \lambda^{1/2}_{jk} \frac{1}{n} \lVert \sum^{n}_{i=1}\xi_{ijk}\left(\hat{g}_{il}-g_{il}\right) \rVert \lVert \hat{\phi}_{lm}-\phi_{lm} \rVert\\
		&\overset{(i)}{\leq} \lambda^{1/2}_{jk} \left( \frac{1}{n} \sum^{n}_{i=1}\xi^{2}_{ijk} \right)^{1/2} \left( \frac{1}{n} \sum^{n}_{i=1} \lVert \hat{g}_{il}-g_{il} \rVert^2 \right)^{1/2} \lVert \hat{\phi}_{lm}-\phi_{lm} \rVert\\
		&\overset{(ii)}{\leq} \lambda^{1/2}_{jk} d_{lm} \left( \frac{1}{n} \sum^{n}_{i=1}\xi^{2}_{ijk} \right)^{1/2} \left( \frac{1}{n} \sum^{n}_{i=1} \lVert \hat{g}_{il}-g_{il} \rVert^2 \right)^{1/2} \lVert \hat{K}_{ll}-K_{ll} \rVert_{\text{HS}},
		\end{align*}
		\endgroup
		where $(i)$ follows Lemma~\ref{lemma:vecfuncnorm}, and $(ii)$ follows Lemma~\ref{lemma:EigenfuncEstBound}. Since $\lambda^{1/2}_{jk}\leq \sqrt{d_1}k^{-\beta/2}$, $d_{lm}\leq d_2 m^{1+\beta}$, and $A_1$-$A_3$ hold, we have
		\begingroup
		\allowdisplaybreaks
		\begin{align*}
		\lvert I_4 \rvert &\leq \sqrt{\frac{3}{2}}\sqrt{d_1}d_2 k^{-\beta/2}m^{1+\beta}\delta_1 \delta_2
		\leq \sqrt{\frac{3}{2}} d^{2}_{0} M^{1+\beta} \delta_1 \delta_2\\
		&\overset{(iii)}{\leq} \frac{\delta}{16} \times \frac{\sqrt{\frac{3}{2}} d^{2}_{0} M^{1+\beta} \delta_1 \delta_2}{\sqrt{\frac{3}{2}} d_0 \delta_1}
		\leq \frac{\delta}{16} \times \frac{\delta}{16d_0 \sqrt{3\lambda_{0,\max}}}
		\leq \frac{\delta}{16},
		\end{align*}
		\endgroup
		where $(iii)$ follows \eqref{eq:I2boundeq3}.
	\end{proof}

\begin{lemma}\label{lemma:I5bound}
	Given that $A1$-$A5$ hold, we have $\lvert I_5 \rvert \leq \delta/16$  for all $1 \leq j,l \leq p$, $ \ 1 \leq k,m \leq M$.
\end{lemma}
	\begin{proof}
		This proof is similar to the proof of Lemma~\ref{lemma:I2bound}, thus is omitted.
	\end{proof}

\begin{lemma}\label{lemma:I6bound}
	Given that $A1$-$A5$ hold, we have $\lvert I_6 \rvert \leq \delta/16$  for all $1 \leq j,l \leq p$, $ \ 1 \leq k,m \leq M$.
\end{lemma}
	\begin{proof}
		We have
		\begingroup
		\allowdisplaybreaks
		\begin{align*}
		\lvert I_6 \rvert
		&\leq \sqrt{\frac{1}{n}\sum^{n}_{i=1}\lvert \langle \hat{g}_{ij}-g_{ij},\phi_{jk} \rangle \rvert^2} \cdot \sqrt{\frac{1}{n}\sum^{n}_{i=1} \lvert \langle \hat{g}_{il}-g_{il},\phi_{lm} \rangle \rvert^2}\\
		& \leq \sqrt{\frac{1}{n}\sum^{n}_{i=1 } \lVert \hat{g}_{ij}-g_{ij} \rVert^2} \cdot \sqrt{\frac{1}{n}\sum^{n}_{i=1} \lVert \hat{g}_{il}-g_{il} \rVert^2}.
		\end{align*}
		\endgroup
		By the assumption that $A_1$ holds, we have
		$\lvert I_6 \rvert \leq \delta^{2}_1$.
		By \eqref{eq:I2boundeq2},\eqref{eq:I2boundeq3} and Lemma~\ref{lemma:I2bound}, we have
		\begingroup
		\allowdisplaybreaks
		\begin{align}\label{eq:I6boundeq2}
		\delta^2_1 &\leq \frac{\delta}{16} \times \frac{\delta^{2}_{1}}{\sqrt{\frac{3}{2}}d_0\delta_1}
		\leq \frac{\delta}{16},
		\end{align}
		\endgroup
		which completes the proof.
	\end{proof}

\begin{lemma}\label{lemma:I7bound}
	Given that $A1$-$A5$ hold, we have $\lvert I_7 \rvert \leq \delta/16$  for all $1 \leq j,l \leq p$, $ \ 1 \leq k,m \leq M$.
\end{lemma}
	\begin{proof}
		We have
		\begingroup
		\allowdisplaybreaks
		\begin{align*}
		\lvert I_7 \rvert 
		&\leq \sqrt{\frac{1}{n}\sum^{n}_{i=1} \lvert \langle \hat{g}_{ij}-g_{ij},\phi_{jk} \rangle \rvert^2} \cdot \sqrt{\frac{1}{n}\sum^{n}_{i=1} \lvert \langle g_{il},\hat{\phi}_{lm}-\phi_{lm} \rangle \rvert^2}\\
		&\leq \sqrt{\frac{1}{n}\sum^{n}_{i=1} \lVert \hat{g}_{ij}-g_{ij} \rVert^2} \cdot \sqrt{\frac{1}{n}\sum^{n}_{i=1} \lVert g_{il} \rVert^2 \lVert \hat{\phi}_{lm}-\phi_{lm} \rVert^2}\\
		&\overset{(i)}{\leq} \delta_1 \lVert \hat{\phi}_{lm}-\phi_{lm} \rVert \cdot \sqrt{\frac{1}{n}\sum^{n}_{i=1} \lVert g_{il} \rVert^2}
		\overset{(ii)}{\leq} \delta_1 \sqrt{2\lambda_{0,\max}} \cdot \lVert \hat{\phi}_{lm}-\phi_{lm} \rVert \\
		& \overset{(iii)}{\leq} \delta_1 \sqrt{2\lambda_{0,\max}} d_{lm} \Vert \hat{K}_{ll}-K_{ll} \rVert_{\text{HS}}
		\overset{(iv)}{\leq} \delta_1 \delta_2 \sqrt{2\lambda_{0,\max}} d_{lm}
		\leq d_0 \sqrt{2\lambda_{0,\max}} M^{1+\beta} \delta_1 \delta_2,
		\end{align*}
		\endgroup
		where $(i)$ follows since $A_1$ holds, $(ii)$ follows since $A_4$ holds, $(iii)$ follows from Lemma~\ref{lemma:EigenfuncEstBound}, and $(iv)$ follows since $A_2$ holds. By \eqref{eq:I2boundeq2} and \eqref{eq:I3boundeq2}, we have
		\begingroup
		\allowdisplaybreaks
		\begin{align*}
		\lvert I_7 \rvert &\leq \frac{\delta}{16} \times \frac{d_0 \sqrt{2\lambda_{0,\max}} M^{1+\beta} \delta_1 \delta_2}{d^{2}_{0} M^{1+\beta} \sqrt{3\lambda_{0,\max}} \delta_2}
		\leq \frac{\delta}{16} \times \sqrt{\frac{2}{3}} \times \frac{\delta}{8\sqrt{6}d^{2}_{0}}
		\leq \frac{\delta}{16},
		\end{align*}
		\endgroup
		which completes the proof.
	\end{proof}

\begin{lemma}\label{lemma:I8bound}
	Given that $A1$-$A5$ hold, we have $\lvert I_8 \rvert \leq \delta/16$  for all $1 \leq j,l \leq p$, $ \ 1 \leq k,m \leq M$.
\end{lemma}
	\begin{proof}
	We have
		\begingroup
		\allowdisplaybreaks
		\begin{align*}
		\lvert I_8 \rvert
		&\leq \sqrt{\frac{1}{n} \sum^{n}_{i=1}\lvert \langle \hat{g}_{ij}-g_{ij},\phi_{jk} \rangle \rvert^2} \cdot \sqrt{\frac{1}{n} \sum^{n}_{i=1} \lvert \langle \hat{g}_{il}-g_{il},\hat{\phi}_{lm}-\phi_{lm} \rangle \rvert^2}\\
		&\leq \sqrt{\frac{1}{n} \sum^{n}_{i=1}\lVert \hat{g}_{ij}-g_{ij} \rVert^2} \cdot \sqrt{\frac{1}{n} \sum^{n}_{i=1} \lVert \hat{g}_{il}-g_{il} \rVert^2 \lVert \hat{\phi}_{lm}-\phi_{lm} \rVert^2}\\
		&\overset{(i)}{\leq} \delta^{2}_1 \lVert \hat{\phi}_{lm}-\phi_{lm} \rVert
		\overset{(ii)}{\leq} \delta^{2}_1 d_{lm} \lVert \hat{K}_{ll}-K_{ll} \rVert_{\text{HS}}
		\leq \delta^{2}_1 d_2 m^{1+\beta} \lVert \hat{K}_{ll}-K_{ll} \rVert_{\text{HS}}\\
		& \leq \delta^{2}_1 d_0 M^{1+\beta} \lVert \hat{K}_{ll}-K_{ll} \rVert_{\text{HS}}
		\overset{(iii)}{\leq} d_0 M^{1+\beta} \delta^{2}_1 \delta_2,
		\end{align*}
		\endgroup
		where $(i)$ follows since $A_1$ holds, $(ii)$ follows from Lemma~\ref{lemma:EigenfuncEstBound}, and $(iii)$ follows since $A_2$ holds. By \eqref{eq:I6boundeq2}, we have
		\begingroup
		\allowdisplaybreaks
		\begin{align*}
		\lvert I_8 \rvert \leq \frac{\delta}{16} \times \frac{d_0 M^{1+\beta} \delta^{2}_1 \delta_2}{\delta^{2}_{1}}
		\leq \frac{\delta}{16},
		\end{align*}
		\endgroup
		which completes the proof.
	\end{proof}

\begin{lemma}\label{lemma:I9bound}
	Given that $A1$-$A5$ hold, we have $\lvert I_9 \rvert \leq \delta/16$  for all $1 \leq j,l \leq p$, $ \ 1 \leq k,m \leq M$.
\end{lemma}
	\begin{proof}
		This proof is similar to the proof of Lemma~\ref{lemma:I3bound}, and is therefore omitted.
	\end{proof}

\begin{lemma}\label{lemma:I10bound}
	Given that $A1$-$A5$ hold, we have $\lvert I_{10} \rvert \leq \delta/16$  for all $1 \leq j,l \leq p, 1 \leq k,m \leq M$.
\end{lemma}
	\begin{proof}
		This proof is similar to the proof of Lemma~\ref{lemma:I7bound}, and is therefore omitted.
	\end{proof}

\begin{lemma}\label{lemma:I11bound}
	Given that $A1$-$A5$ hold, we have $\lvert I_{11} \rvert \leq \delta/16$  for all $1 \leq j,l \leq p, 1 \leq k,m \leq M$.
\end{lemma}
	\begin{proof}
		We have
		\begingroup
		\allowdisplaybreaks
		\begin{align*}
		\lvert I_{11} \rvert 
		&\leq \sqrt{\frac{1}{n}\sum^{n}_{i=1}\lvert \langle g_{ij},\hat{\phi}_{jk}-\phi_{jk} \rangle \rvert^2} \cdot \sqrt{\frac{1}{n}\sum^{n}_{i=1}\lvert \langle g_{il},\hat{\phi}_{lm}-\phi_{lm} \rangle \rvert^2}\\
		&\leq \sqrt{\frac{1}{n}\sum^{n}_{i=1}\lVert g_{ij} \rVert^2} \cdot\sqrt{\frac{1}{n}\sum^{n}_{i=1}\lVert g_{il} \rVert^2} \cdot\lVert \hat{\phi}_{jk}-\phi_{jk} \rVert\cdot \lVert \hat{\phi}_{lm}-\phi_{lm} \rVert\\
		&\overset{(i)}{\leq}2\lambda_{0,\max}  \lVert \hat{\phi}_{jk}-\phi_{jk} \rVert \lVert \hat{\phi}_{lm}-\phi_{lm} \rVert
		\overset{(ii)}{\leq}2\lambda_{0,\max} \delta^{2}_2 d_{jk} d_{lm}
		\leq 2\lambda_{0,\max} \delta^{2}_2 d^{2}_2 k^{1+\beta} m^{1+\beta},
		\end{align*}
		\endgroup
		where $(i)$ follows since $A_4$ holds and $(ii)$ follows from Lemma~\ref{lemma:EigenfuncEstBound}. Then, we have
		\begin{equation}\label{eq:I11boundeq}
		\lvert I_{11} \rvert \leq 2 d^{2}_0 \lambda_{0,\max} M^{2+2\beta} \delta^{2}_2.
		\end{equation} 
		By \eqref{eq:I3boundeq2}, we have
		\begingroup
		\allowdisplaybreaks
		\begin{align}\label{eq:I11boundeq2}
		2 d^{2}_0 \lambda_{0,\max} M^{2+2\beta} \delta^{2}_2 
		\leq \frac{\delta}{16} \times \frac{2 d^{2}_0 \lambda_{0,\max} M^{2+2\beta} \delta^{2}_2}{d^{2}_{0} M^{1+\beta} \sqrt{3\lambda_{0,\max}} \delta_2}
		\leq \frac{\delta}{16},
		\end{align}
		\endgroup
		which completes the proof.
	\end{proof}

\begin{lemma}\label{lemma:I12bound}
	Given that $A1$-$A5$ hold, we have $\lvert I_{12} \rvert \leq \delta/16$  for all $1 \leq j,l \leq p, 1 \leq k,m \leq M$.
\end{lemma}
	\begin{proof}
		We have
		\begingroup
		\allowdisplaybreaks
		\begin{align*}
		\lvert I_{12} \rvert 
		&\leq \sqrt{\frac{1}{n}\sum^{n}_{i=1} \lvert \langle g_{ij},\hat{\phi}_{jk}-\phi_{jk} \rangle \rvert^2} \cdot\sqrt{\frac{1}{n}\sum^{n}_{i=1} \lvert \langle \hat{g}_{il}-g_{il},\hat{\phi}_{lm}-\phi_{lm} \rangle \rvert^2}\\
		&\leq \sqrt{\frac{1}{n}\sum^{n}_{i=1} \lVert g_{ij} \rVert^2} \cdot\sqrt{\frac{1}{n}\sum^{n}_{i=1} \lVert \hat{g}_{il}-g_{il} \rVert^2} \cdot\lVert \hat{\phi}_{jk}-\phi_{jk} \rVert\cdot \lVert \hat{\phi}_{lm}-\phi_{lm} \rVert\\
		&\overset{(i)}{\leq} \sqrt{2\lambda_{0,\max}} \delta_1 \delta^2_2 d_{jk} d_{lm}
		\leq d^2_2 \sqrt{2\lambda_{0,\max}} k^{1+\beta} m^{1+\beta} \delta_1 \delta^2_2,
		\end{align*}
		\endgroup
		where $(i)$ follows since $A_1$-$A_3$ hold and Lemma~\ref{lemma:EigenfuncEstBound}. 
		Then, we have
		\begin{equation}\label{eq:I12boundeq}
		\lvert I_{12} \rvert \leq d^2_0 \sqrt{2\lambda_{0,\max}} M^{2+2\beta} \delta_1 \delta^2_2.
		\end{equation}
		By \eqref{eq:I2boundeq2} and \eqref{eq:I11boundeq2},  we have
		\begingroup
		\allowdisplaybreaks
		\begin{align}\label{eq:I12boundeq2}
		d^2_0 \sqrt{2\lambda_{0,\max}} M^{2+2\beta} \delta_1 \delta^2_2 &\leq \frac{\delta}{16} \times \frac{d^2_0 \sqrt{2\lambda_{0,\max}} M^{2+2\beta} \delta_1 \delta^2_2}{2 d^{2}_0 \lambda_{0,\max} M^{2+2\beta} \delta^{2}_2}\leq \frac{\delta}{16},
		\end{align}
		\endgroup
		which completes the proof.
	\end{proof}

\begin{lemma}\label{lemma:I13bound}
	Given that $A1$-$A5$ hold, we have $\lvert I_{13} \rvert \leq \delta/16$  for all $1 \leq j,l \leq p, 1 \leq k,m \leq M$.
\end{lemma}
	\begin{proof}
		This proof is similar to the proof of Lemma~\ref{lemma:I4bound}, and is therefore omitted.
	\end{proof}

\begin{lemma}\label{lemma:I14bound}
	Given that $A1$-$A5$ hold, we have $\lvert I_{14} \rvert \leq \delta/16$  for all $1 \leq j,l \leq p, 1 \leq k,m \leq M$.
\end{lemma}
	\begin{proof}
		This proof is similar to the proof of Lemma~\ref{lemma:I8bound}, and is therefore omitted.
	\end{proof}

\begin{lemma}\label{lemma:I15bound}
	Given that $A1$-$A5$ hold, we have $\lvert I_{15} \rvert \leq \delta/16$  for all $1 \leq j,l \leq p$, $ \ 1 \leq k,m \leq M$.
\end{lemma}
	\begin{proof}
		This proof is similar to the proof of Lemma~\ref{lemma:I2bound}, thus is omitted.
	\end{proof}

\begin{lemma}\label{lemma:I16bound}
	Given that $A1$-$A5$ hold, we have $\lvert I_{16} \rvert \leq \delta/16$  for all $1 \leq j,l \leq p, 1 \leq k,m \leq M$.
\end{lemma}
	\begin{proof}
		We have
		\begingroup
		\allowdisplaybreaks
		\begin{align*}
		\lvert I_{16} \rvert
		&\leq \sqrt{\frac{1}{n}\sum^{n}_{i=1} \lVert \hat{g}_{ij}-g_{ij} \rVert^2} \cdot \sqrt{\frac{1}{n}\sum^{n}_{i=1} \lVert \hat{g}_{il}-g_{il} \rVert^2} \cdot\lVert \hat{\phi}_{jk}-\phi_{jk} \rVert \cdot\lVert \hat{\phi}_{lm}-\phi_{lm} \rVert\\
		&\overset{(i)}{\leq} \delta^2_1 d_{jk} d_{lm} \delta^2_2
		\leq d^2_2 k^{1+\beta} m^{1+\beta} \delta^2_1 \delta^2_2
		\leq d^2_0 M^{2+2\beta} \delta^2_1 \delta^2_2,
		\end{align*}
		\endgroup
		where $(i)$ follows since $A_1$ and $A_2$ hold, and Lemma~\ref{lemma:EigenfuncEstBound}. Thus, by \eqref{eq:I2boundeq3} and \eqref{eq:I12boundeq2}, we have
		\begingroup
		\allowdisplaybreaks
		\begin{align*}
		\lvert I_{16} \rvert &\leq \frac{\delta}{16} \times \frac{d^2_0 M^{2+2\beta} \delta^2_1 \delta^2_2}{d^2_0 \sqrt{2\lambda_{0,\max}} M^{2+2\beta} \delta_1 \delta^2_2}
		\leq \frac{\delta}{16},
		\end{align*}
		\endgroup
		which completes the proof.
	\end{proof}

\begin{lemma}\label{lemma:vecfuncnorm}
	Suppose $f_1,f_2,\dots,f_n \in \mathbb{H}$ and $v_1,v_2,\dots,v_n \in \mathbb{R}$. Then
	\begin{equation*}
	\left\lVert \sum^{n}_{i=1}v_{i}f_{i} \right\rVert \leq \sqrt{\sum^{n}_{i=1}v^{2}_{i}} \cdot \sqrt{\sum^{n}_{i=1}\lVert f_i \rVert^2}.
	\end{equation*}
\end{lemma}
	\begin{proof}
		Note that
	    \begin{multline*}
	    \left\lVert \sum^{n}_{i=1}v_{i}f_{i} \right\rVert^2=\int \left( \sum^{n}_{i=1}v_{i}f_{i}(t) \right)^2 dt \\
	    \overset{(i)}{\leq} \int \left( \sum^{n}_{i=1}v^{2}_{i} \right) \left( \sum^{n}_{i=1}f^{2}_{i}(t) \right) dt
	    =\left( \sum^{n}_{i=1}v^{2}_{i} \right)\left( \sum^{n}_{i=1}\lVert f_{i} \rVert^2 \right),
	    \end{multline*}
	    where $(i)$ follows the Cauchy-Schwarz inequality. This directly implies the result.
	\end{proof}

\begin{lemma}[Lemma 4.3 of \cite{Bosq2000Linear}]\label{lemma:EigenfuncEstBound}
	Suppose that Assumption~\ref{assump:EigenAssump} holds. Denote $\tilde{\phi}_{jk}=\text{sgn}\left(\langle \hat{\phi}_{jk},\phi_{jk} \rangle\right) \phi_{jk}$, where $\text{sgn}(t)=1$ if $t \geq 0$ and $\text{sgn}(t)=-1$ if $t < 0$. Then 
	\begin{equation*}
	\lVert \hat{\phi}_{jk}-\tilde{\phi}_{jk} \rVert \leq d_{jk} \lVert \hat{K}_{jj}-K_{jj} \rVert_{\text{HS}},
	\end{equation*}
	where $d_{j1}=2\sqrt{2}(\lambda_{j1}-\lambda_{j2})^{-1}$
	and $d_{jk}=2\sqrt{2}\max\{(\lambda_{j(k-1)}-\lambda_{jk})^{-1},(\lambda_{jk}-\lambda_{j(k+1)})^{-1}\}$, $k\geq 2$.
\end{lemma}

\begin{lemma}\label{lemma:boundnormalnorm}
	Suppose $z \sim N_{L}\left(0,I_L\right)$. Then
	\begin{equation*}
	P\left( \lVert z \rVert_2 > \delta \right) \leq 2\exp\left( -\frac{\delta^2}{8L+2\sqrt{2L}\delta} \right), \qquad \delta > 0.
	\end{equation*}
\end{lemma}
	\begin{proof}
		Since
		\begin{equation*}
		\mathbb{E}\left[ \lVert z \rVert^{2k}_2 \right]=\frac{\Gamma(\frac{L}{2}+k)}{\Gamma(\frac{L}{2})} \times 2^k \leq k! (2L)^k,
		\end{equation*}
		we have
		\begin{equation*}
		\mathbb{E}\left[ \lVert z \rVert^{k}_2 \right] \leq \sqrt{\mathbb{E}\left[ \lVert z \rVert^{2k}_2 \right]}\leq \sqrt{k!}\left( \sqrt{2L} \right)^{k}\leq \frac{k!}{2}\cdot 4L \cdot (\sqrt{2L})^{k-2}
		\end{equation*}
		for $k\geq2$. The result follows from Lemma~\ref{lemma:HilbertBernstein}.
	\end{proof}

\begin{lemma}[Theorem 2.5 (2) of \cite{Bosq2000Linear}]\label{lemma:HilbertBernstein}
	Let $Z_1,Z_2,\dots,Z_n$ be independent random variables in a separable Hilbert space with norm $\lVert \cdot \rVert$. If $\mathbb{E}[Z_i]=0$, $i=1,\ldots,n$, and
	\begin{equation*}
	\sum^{n}_{i=1}\mathbb{E}\left[ \lVert Z_i \rVert^k \right] \leq \frac{k!}{2}nL_1L^{k-2}_2, \qquad k\geq 2,
	\end{equation*}
	for two positive constants $L_1$ and $L_2$, then 
	\begin{equation*}
	P\left( \lVert \sum^{n}_{i=1} Z_i \rVert \geq n\delta \right)\leq 2\exp\left( -\frac{n\delta^2}{2L_1+2L_2\delta} \right),\qquad \delta > 0.
	\end{equation*}
\end{lemma}

\begin{lemma}\label{lemma:integralapprox}
	Let $f(t)$ be a function defined on $\mathcal{T}$ and suppose that $f$ has a continuous derivative. Let $D_{0,f} \coloneqq \sup_{t\in{\mathcal{T}}}\lvert f(t) \rvert$ and $D_{1,f} \coloneqq \sup_{t\in{\mathcal{T}}}\lvert f^{\prime}(t) \rvert$. Assume that $D_{0,f}, D_{1,f} < \infty$. Let $\lvert \mathcal{T} \rvert$ denote the length of the interval $\mathcal{T}$, and let $u_1<u_2<\dots<u_T \in \mathcal{T}$. We denote the endpoints of $\mathcal{T}$ as $u_0$ and $u_{T+1}$. Assume that there is a positive constant $\zeta_0$ such that
	\begin{equation}\label{eq:unifcondapp}
	\max_{1\leq k \leq T+1}\left\vert \frac{u_{k}-u_{k-1}}{\lvert \mathcal{T} \rvert} - \frac{1}{T} \right\vert \leq \frac{\zeta_0}{T^2}.
	\end{equation}
	Let $\zeta_1=\zeta_0+1$. Then 
	\begin{equation*}\label{eq:integralapproxeq}
	\left\vert \frac{1}{T}\sum^{T}_{k=1}f(u_k)-\frac{1}{\lvert \mathcal{T} \rvert}\int_{\mathcal{T}}f(t)dt \right\vert \leq \frac{D_{1,f}\zeta^{2}_1 \lvert \mathcal{T} \rvert/2+D_{0,f}(\zeta_{1}+\zeta_{0})}{T}.
	\end{equation*}
	
\end{lemma}
\begin{proof}
		Since
		\begin{multline*}
		\left\vert \frac{1}{T}\sum^{T}_{k=1}f(u_k)-\frac{1}{\lvert \mathcal{T} \rvert}\int_{\mathcal{T}}f(t)dt \right\vert 
		\leq \left\vert \frac{1}{T}\sum^{T}_{k=1}f(u_k)- \frac{1}{\lvert \mathcal{T} \rvert}\sum^{T}_{k=1}f(u_k)(u_k-u_{k-1})\right\vert \\
		+\left\vert \frac{1}{\lvert \mathcal{T} \rvert}\sum^{T}_{k=1}f(u_k)(u_k-u_{k-1})-\frac{1}{\lvert \mathcal{T} \rvert}\int_{\mathcal{T}}f(t)dt \right\vert,
		\end{multline*}
		we proceed to show that the first part is smaller than $D_{0,f}\zeta_{0}/T$ and that the second part is smaller than $(D_{1,f}\zeta^{2}_1 \lvert \mathcal{T} \rvert/2+D_{0,f}\zeta_{1})/T$.
		For the first part, we have
		\begingroup
		\allowdisplaybreaks
		\begin{multline*}
		\left\vert \frac{1}{T}\sum^{T}_{k=1}f(u_k)- \frac{1}{\lvert \mathcal{T} \rvert}\sum^{T}_{k=1}f(u_k)(u_k-u_{k-1})\right\vert 
		\leq \sum^{T}_{k=1} \left\vert f(u_k) \right\vert \left\vert  \frac{1}{T}-\frac{u_k-u_{k-1}}{\lvert \mathcal{T} \rvert}  \right\vert  \\
		\leq \max_{1\leq k \leq T}\left\vert \frac{u_{k}-u_{k-1}}{\lvert \mathcal{T} \rvert} - \frac{1}{T} \right\vert \sum^{T}_{k=1} \left\vert f(u_k) \right\vert
		\leq \frac{\zeta_0}{T^2} \times T \times D_{0,f}
		=\frac{\zeta_0D_{0,f}}{T}.
		\end{multline*}
		\endgroup
		To prove the second part, we first note that based on \eqref{eq:unifcondapp}, we have
		\begin{equation*}
		\max_{1\leq k \leq T+1}\lvert u_{k} - u_{k-1} \rvert \leq \frac{\zeta_1 \lvert \mathcal{T} \rvert}{T}.
		\end{equation*}
		Then, for any $t \in (u_{k}, u_{k+1})$, by Taylor's expansion, we have $f(t)=f(u_k)+f^{\prime}(\bar{t})(t-u_k)$, where $\bar{t} \in (u_{k}, t)$, and
		$\lvert f(t)-f(u_k) \rvert = \lvert f^{\prime}(\bar{t}) \rvert(t-u_k) \leq D_{1,f}(t-u_k)$.
		Therefore,
		\begingroup
		\allowdisplaybreaks
		\begin{align*}
		&\left\vert \frac{1}{\lvert \mathcal{T} \rvert}\sum^{T}_{k=1}f(u_k)(u_k-u_{k-1})-\frac{1}{\lvert \mathcal{T} \rvert}\int_{\mathcal{T}}f(t)dt \right\vert \\
		&\leq \frac{1}{\lvert \mathcal{T} \rvert}\sum^{T}_{k=1} \int^{u_k}_{u_{k-1}} \lvert f(u_k) - f(t) \rvert dt + \frac{1}{\lvert \mathcal{T} \rvert} \int^{u_{T+1}}_{u_{T}}\lvert f(t) \rvert dt\\
		&\leq \frac{1}{\lvert \mathcal{T} \rvert} \times T \times D_{1,f} \times  \int^{u_k}_{u_{k-1}}(t-u_k)dt + \frac{1}{\lvert \mathcal{T} \rvert} \times D_{0,f} \times \frac{\zeta_1 \lvert \mathcal{T} \rvert}{T}\\
		&= \frac{1}{\lvert \mathcal{T} \rvert} \times T \times D_{1,f} \times \frac{(u_{k+1}-u_k)^2}{2} + \frac{1}{\lvert \mathcal{T} \rvert} \times D_{0,f} \times \frac{\zeta_1 \lvert \mathcal{T} \rvert}{T}\\
		&\leq \frac{1}{\lvert \mathcal{T} \rvert} \times T \times \frac{D_{1,f}}{2} \times \left(\max_{1\leq k \leq T+1} \lvert u_{k+1}-u_k \rvert\right)^2 + \frac{1}{\lvert \mathcal{T} \rvert} \times D_{0,f} \times \frac{\zeta_1 \lvert \mathcal{T} \rvert}{T}\\
		&\leq \frac{1}{\lvert \mathcal{T} \rvert} \times T \times \frac{D_{1,f}}{2} \times \left(\frac{\zeta_1 \lvert \mathcal{T} \rvert}{T}\right)^2 + \frac{1}{\lvert \mathcal{T} \rvert} \times D_{0,f} \times \frac{\zeta_1 \lvert \mathcal{T} \rvert}{T}\\
		&=\frac{D_{1,f}\zeta^{2}_1 \lvert \mathcal{T} \rvert/2+D_{0,f}\zeta_{1}}{T}.
		\end{align*}
		\endgroup
		The result follows by combining the two bounds.
	\end{proof}

\begin{lemma}\label{lemma:2kmomentGuassFunc}
Let $g$ be a mean zero Gaussian random function in a Hilbert space $\mathbb{H}$. We have $\mathbb{E}\left[ \Vert g \Vert^{2k} \right]\leq (2\lambda_0)^{k}\cdot k!$ where $\lambda_0=\mathbb{E}\left[ \Vert g \Vert^{2} \right]$.

\end{lemma}
	\begin{proof}
		Let $\{\phi_{m} \}_{m\geq 1}$ be the orthonormal eigenfunctions of $g$ and $a_m=\langle g,\phi_{m} \rangle$. Then $a_m\sim N(0,\lambda_m)$ and $\lambda_0=\sum_{m\geq 1}\lambda_{m}$. Let $\xi_{m}=\lambda^{-1/2}_{m}a_m$. By the Karhunen-Loève theorem, we have
		$g=\sum^{\infty}_{m=1}\lambda_{m}^{1/2} \xi_{m} \phi_{m}$.
		Thus, $\Vert g \Vert=\left(\sum_{m\geq 1}\lambda_{m}\xi^{2}_{m}\right)^{1/2}$ and $\Vert g \Vert^{2k}=\left(\sum_{m\geq 1}\lambda_{m}\xi^{2}_{m}\right)^{k}$.
		By Jensen's inequality, we have
		\begin{multline*}
		\Vert g \Vert^{2k}
		=\left(\sum_{m\geq 1}\lambda_{m}\right)^{k} \cdot  \left(\frac{\sum_{m\geq 1}\lambda_{m}\xi^{2}_{m}}{\sum_{m\geq 1}\lambda_{m}}\right)^{k}\\
		\leq \left(\sum_{m\geq 1}\lambda_{m}\right)^{k} \cdot \frac{\sum_{m\geq 1}\lambda_{m}\xi^{2k}_{m}}{\sum_{m\geq 1}\lambda_{m}}
		=\left(\sum_{m\geq 1}\lambda_{m}\right)^{k-1} \cdot \left( \sum_{m\geq 1}\lambda_{m}\xi^{2k}_{m} \right).
		\end{multline*}
		Thus,
		\begin{multline*}
		\mathbb{E}\left[ \Vert g \Vert^{2k} \right]
		\leq \left(\sum_{m\geq 1}\lambda_{m}\right)^{k-1} \cdot \left( \sum_{m\geq 1}\lambda_{m}\mathbb{E}\left[\xi^{2k}_{m}\right] \right)
		=\left(\sum_{m\geq 1}\lambda_{m}\right)^{k} \mathbb{E}\left[\xi^{2k}_{1}\right] \\
		=\left(\sum_{m\geq 1}\lambda_{m}\right)^{k}\cdot \pi^{-1/2} \cdot 2^k \cdot \Gamma (k+1/2)
		\leq \left(\sum_{m\geq 1}\lambda_{m}\right)^{k}\cdot 2^k \cdot k!
		=(2\lambda_{0})^{k}k!,
		\end{multline*}
		which completes the proof.
    \end{proof}

\begin{lemma}\label{lemma:crosscovfuncbound}
	For any $\delta>0$ and any $j=1,\ldots,p$, we have
	\begin{equation*}
	P\left( \left\Vert \frac{1}{n}\sum^{n}_{i=1}\left[ g_{ij}(t)g_{ij}(s)-K_{jj}(s,t) \right] \right\Vert_{\text{HS}}>\delta \right)\leq 2\exp\left( -\frac{n\delta^2}{64\lambda^{2}_{0,\max}+8\lambda_{0,\max}\delta} \right).
	\end{equation*}
\end{lemma}

\begin{proof}
	Since $g_{ij}(t)=\sum_{m\geq 1}\lambda^{1/2}_{jm}\xi_{ijm}\phi_{jm}(t)$ and $\xi_{ijm}\sim N(0,1)$, we have 
	$$g_{ij}(s)g_{ij}(t)=\sum_{m,m^{\prime}\geq 1}\lambda^{1/2}_{jm}\lambda^{1/2}_{jm^{\prime}}\xi_{ijm}\xi_{ijm^{\prime}}\phi_{jm}(s)\phi_{jm^{\prime}}(t),$$ 
	and 
	$$K_{jj}(s,t)=\mathbb{E}[g_{ij}(s)g_{ij}(t)]=\sum_{m,m^{\prime}\geq 1}\lambda^{1/2}_{jm}\lambda^{1/2}_{jm^{\prime}}\phi_{jm}(s)\phi_{jm^{\prime}}(t)\mathbbm{1}_{mm^{\prime}},$$ 
	where $\mathbbm{1}_{mm^{\prime}}=\mathbbm{1}(m=m^{\prime})=1$ if $m=m^{\prime}$ and $0$ if $m \neq m^{\prime}$. Thus,
	\begin{equation*}
	\left\Vert g_{ij}(s)g_{ij}(t)-K_{jj}(s,t) \right\Vert^{2}_{\text{HS}}=\sum_{m,m^{\prime}\geq 1}\lambda_{jm}\lambda_{jm^{\prime}}(\xi_{ijm}\xi_{ijm^{\prime}}-\mathbbm{1}_{mm^{\prime}})^{2},
	\end{equation*}
	and, for any $k \geq 2$, we have
	\begin{equation*}
	\begin{aligned}
	\mathbb{E}&\left[ \left\Vert g_{ij}(s)g_{ij}(t)-K_{jj}(s,t) \right\Vert^{k}_{\text{HS}} \right]\\
	&=\mathbb{E}\left[ \left\{ \sum_{m,m^{\prime}\geq 1}\lambda_{jm}\lambda_{jm^{\prime}}(\xi_{ijm}\xi_{ijm^{\prime}}-\mathbbm{1}_{mm^{\prime}})^{2} \right\}^{k/2} \right]\\
	&\overset{(i)}{\leq} \left( \sum_{m,m^{\prime}\geq 1}\lambda_{jm}\lambda_{jm^{\prime}} \right)^{k/2-1} \sum_{m,m^{\prime}\geq 1}\lambda_{jm}\lambda_{jm^{\prime}}\mathbb{E}\left[\left( \xi_{ijm}\xi_{ijm^{\prime}}-\mathbbm{1}_{mm^{\prime}} \right)^{k} \right],
	\end{aligned}
	\end{equation*}
	where $(i)$ follows from Jensen's inequality. Since
	\begin{equation*}
	\begin{aligned}
	\mathbb{E}\left[\left( \xi_{ijm}\xi_{ijm^{\prime}}-\mathbbm{1}_{mm^{\prime}} \right)^{k} \right]
	&\leq 2^{k-1}\left(\mathbb{E}\left[ (\xi_{ijm}\xi_{ijm^{\prime}})^{k}\right]+1\right)\\
	&\leq 2^{k-1} \left( \mathbb{E}[\xi^{2k}_{ij1}] +1\right)
	\leq 2^{k-1}(2^k k! +1 )
	\leq 4^k k!,
	\end{aligned}
	\end{equation*}
	we have $\mathbb{E}\left[ \left\Vert g_{ij}(s)g_{ij}(t)-K_{jj}(s,t) \right\Vert^{k}_{\text{HS}} \right] \leq (4\lambda_{j0})^{k}k!\leq (4\lambda_{0,\max})^{k}k!$.
	The result follows from Lemma~\ref{lemma:HilbertBernstein}.
\end{proof}

\end{appendices}

\newpage

\bibliography{reference}
	
\end{document}